\documentclass{article}

\usepackage{amsmath}
\usepackage{amssymb}
\usepackage{mathtools}
\usepackage{amsthm}
\newtheorem{definition}{Definition}
\usepackage{booktabs} 
\usepackage{multirow}
\newtheorem{theorem}{Theorem}
\usepackage{algorithm}
\usepackage{algorithmic}
\usepackage{bm}
\newcommand{\separator}{
  \begin{center}
    \rule{\columnwidth}{0.3mm}
  \end{center}
}

 \def\11{{\textbf{1}}}
\newcommand{\beq}{\begin{eqnarray*}}
\newcommand{\eeq}{\end{eqnarray*}}
\newcommand{\beqn}{\begin{eqnarray}}
\newcommand{\eeqn}{\end{eqnarray}}
\newcommand{\bemn}{\begin{multiline}}
\newcommand{\eemn}{\end{multiline}}

\def\N{\mathbb{N}}

\def\R{\mathbb{R}}

\usepackage{amsmath,amssymb,amsfonts,bm}

\def\R{\mathrm{I\kern-0.4ex R}}
\def\N{\mathrm{I\kern-0.4ex N}}
\def\E{\mathrm{I\kern-0.4ex E}}
\newcommand{\Set}[1]{\mathcal{#1}}

\newcommand{\ve}[1]{\bm #1}

    \PassOptionsToPackage{numbers, compress}{natbib}
\usepackage[preprint]{neurips_2026}

\usepackage[utf8]{inputenc} % allow utf-8 input
\usepackage[T1]{fontenc}    % use 8-bit T1 fonts
\usepackage{hyperref}       % hyperlinks
\usepackage{url}            % simple URL typesetting
\usepackage{booktabs}       % professional-quality tables
\usepackage{amsfonts}       % blackboard math symbols
\usepackage{nicefrac}       % compact symbols for 1/2, etc.
\usepackage{microtype}      % microtypography
\usepackage{xcolor}         % colors
\usepackage[capitalize,noabbrev]{cleveref}
\title{Stagnant Neuron: Towards Understanding the Plasticity Loss in Multi-Agent Reinforcement Learning Value Factorization Methods}

\author{
\shortstack[c]{
Zhengzhu Liu\textsuperscript{a,b},
Zeming Gao\textsuperscript{a,b},
Haoyuan Qin\textsuperscript{a,b},
Jiawei Hu\textsuperscript{a,b},
Junhao Wu\textsuperscript{a,b} \\
Miao Zhu\textsuperscript{a,b},
Haipeng Zhang\textsuperscript{a,b},
Chennan Ma\textsuperscript{a,b},
Siqi Shen\textsuperscript{a,b},
Cheng Wang\textsuperscript{a,b} \\
\textsuperscript{a}Fujian Key Laboratory of Urban Intelligent Sensing and Computing, \\
School of Informatics, Xiamen University (XMU), China \\
\textsuperscript{b}Key Laboratory of Multimedia Trusted Perception and Efficient Computing, XMU, China
}
}

\begin{document}

\maketitle

\begin{abstract}

Multi-Agent Reinforcement Learning (MARL) value factorization methods can suffer from a loss of plasticity, gradually failing to adapt when transferring to new task instances. We trace this issue to stagnant neurons, units whose gradient updates become negligibly small relative to their weights, thereby hindering learning. While existing plasticity injection methods exist, they prove ineffective for such neurons. To address this, we propose \textbf{K}nowledge-retentive \textbf{N}euron-level Plast\textbf{I}city \textbf{F}ocusing Inj\textbf{E}ction (KNIFE), a novel method that directly targets stagnant neurons. KNIFE replaces each stagnant neuron with a composite unit comprising three specialized components: a frozen knowledge neuron to preserve acquired knowledge, a re-initialized active neuron to restore learning capacity, and a compensation neuron to ensure the combined output matches the original, thus maintaining previous learned cooperation knowledge. Extensive experiments on SMACv2, SMAC, predator-prey, and matrix games demonstrate that KNIFE significantly outperforms state-of-the-art plasticity injection methods. %in both final performance and learning efficiency.

\end{abstract}

\section{Introduction}

%In a typical MARL value factorization method~\cite{QMIX}, once the neural networks are trained, the neurons of the neural network are frozen and used for evaluation. This contrasts with human learning which requires continual learning on new tasks~\cite{nature24plasticity}.

%In value factorization methods, the joint state-action value function is represented through neural networks consisting of agent networks and mixing networks.

Multi-Agent Reinforcement Learning (MARL) is widely used in domains such as traffic control~\cite{MARLSurvey} and games~\cite{QMIX}. In practical deployments, the environment dynamics or cooperation objectives may change over time, requiring continual adaptation~\cite{tomilin2025meal}. While plasticity loss has been studied in reinforcement learning~\cite{plasticity,plasticity2,nature24plasticity} and in MARL for stationary tasks~\cite{reborn}, it remains unclear whether value factorization methods~\cite{QMIX,qplex} suffer from plasticity loss under sequential task changes.

% We study this question in MARL settings built from the one-step matrix game, predator-prey, and a modified SMACv2 benchmark~\cite{smacv2}. Tasks are drawn from a small repertoire and switched periodically at episode boundaries. We observe consistent degradation as the number of task switch accumulate: learning bias in the matrix game grows and win rates in SMACv2 drop. Moreover, under the same total training budget, MARL methods (e.g., QMIX~\cite{QMIX} and QPLEX~\cite{qplex}) underperform their non-switching counterparts on each task, indicating a loss of plasticity for periodic task switches. We focus on adaptation degradation under periodic switching, rather than catastrophic forgetting on a non-revisiting task stream.

We study this question in MARL settings built from the one-step matrix game, predator-prey, and a modified SMACv2 benchmark~\cite{smacv2}. Agents are required to sequentially transfer across newly configured task instances. We observe consistent adaptation degradation as the number of sequential transfers accumulates: learning bias in the matrix game grows and win rates in SMACv2 drop. To rigorously evaluate this degradation, we demonstrate this by returning the agent to a previously mastered task; despite an equal and ample training budget, fine-tuned MARL methods (e.g., QMIX~\cite{QMIX} and QPLEX~\cite{qplex}) completely fail to re-acquire their previous performance, significantly underperforming a randomly initialized (trained from scratch) baseline.

To diagnose the source of adaptation degradation, we analyze neuron-level update dynamics and identify two extreme groups of neurons: stagnant neurons and volatile neurons, based on relative update activity. \emph{Stagnant neurons} have small relative updates (a low ratio between the gradient norm and the incoming weight norm), while \emph{volatile neurons} exhibit the opposite. We find that stagnant neurons consistently appear in value factorization methods and are concentrated in the mixing network. These neurons correlate with slower learning and are distinct from dormant neurons~\cite{redo} and GraMa neurons~\cite{regrama}. A stagnant neuron tends to remain in the same group for long periods, which motivates targeted neuron-level intervention.
%Once a neuron becomes stagnant/volatile, it tends to remain in the same group for long periods, which motivates targeted neuron-level intervention.

To mitigate plasticity loss during sequential task transfers, we propose \textbf{K}nowledge-retentive \textbf{N}euron-level Plast\textbf{I}city \textbf{F}ocusing Inj\textbf{E}ction (KNIFE). KNIFE applies a knowledge-invariant neuron surgery. In the injection stage, KNIFE replaces each stagnant/volatile neuron with three neurons (knowledge, active, compensation) so that the combined output matches the original at the intervention boundary. The knowledge and compensation are frozen while the active neuron is re-initialized to restore plasticity. 

We theoretically show that the injection stage of KNIFE satisfies the knowledge invariant++ (KI++) principle proposed in this work. Experiments across the SMACv2, predator-prey, and one-step payoff matrix benchmarks show that KNIFE improves adaptation under task switches compared to existing plasticity injection baselines, thanks to its ability to reduce the stagnant neurons.  We show that stagnant neurons provide a more accurate or sensitive tracking of plasticity loss in MARL than existing indicators such as dormant neurons.%Moreover, the ablation study justifies the design choice that satisfies the KI++ principle.

\section{Background}

\subsection{Dec-POMDPs}

For cooperative MARL, the Decentralized Partially Observable Markov Decision Processes (Dec-POMDPs)~\cite{DEC-POMDP}, can be formally described by the tuple $\langle \Set S, \{\Set A_i\}_{i=1}^N, P, r, \{\Set O_i\}_{i=1}^N, \{\Omega_i\}_{i=1}^N, N \rangle$, where $\mathcal{N}$ is the number of agents, $\mathcal{S}$ is a finite set of states, and $\Set{A}_i$ is the set of actions available to agent $i$. At each step $t$, agent $i$ chooses an action $a_i^t \in \Set A_i$, forming a joint action $\mathbf{a}^t \in \Set{A}^N = \Set A_1 \times \ldots \times \Set A_N$. This joint action lead to a new state $s^{t+1} \sim P(\cdot|s^t, \mathbf{a}^t)$ and a reward $r^t$. Each agent access an individual observation $o_i^t \in O_i$, which is drawn from $o_i^t \sim \Omega_i^i(\cdot|s^t)$. The individual policy of agent $i$ is $\pi_i(a_i|\tau_i)$, $\tau_i = (O_i \times A_i)^*$ is the local action-observation history for agent $i$. The state-action value function for agent $i$ is denoted as $Q_i(\tau_i, a_i)$.
The joint action-observation history is denoted as $ \mathbf{\tau} \in \mathcal{T}^N:=\tau_1 \times \ldots \times \tau_N$, on which it conditions the joint policy $\pi=<\pi_1, \ldots, \pi_N>$. The joint policy $\pi$ has a joint state-action value function $Q_{tot}(\mathbf{\tau},\mathbf{a})$.

\subsection{Value Function Factorization}

Value factorization is a widely adopted class of MARL methods. In these methods, the neural network typically comprises an agent network and a mixing network. The mixing network represents the joint state‑action value function $Q_{tot}(\mathbf{\tau}, \mathbf{a})$ by taking as input each agent’s individual utility function $Q_i$. \citet{reborn} proposes the following knowledge invariant principle for a neuron-level method, which assesses whether a method preserves the learned cooperative knowledge.

\begin{definition}[Knowledge Invariant Principle (KI)~\cite{reborn}]
\label{def:ki}
A joint state-action value function is represented as 
$Q_{tot}^{\theta,\phi}(\tau,\mathbf{a}) = f_{\theta}(Q_1^{\phi},...,Q_N^{\phi})$, where $f_{\theta}$ is the mixing function that mixes $Q_i$ into $Q_{tot}$, $\tau$ is joint observation-action history, $\mathbf{a} = [a_1,...a_N]$ is the joint action of multi-agent, $g: \mathbb{R} \mapsto \mathbb{R}$ is a function that maps weights in $\theta$ to $\hat{\theta}$, $g(\theta) = \hat{\theta}$. $h: \mathbb{R} \mapsto \mathbb{R}$, $h(\phi) = \tilde{\phi}$, $h$ map the weights in $\phi$ to $\tilde{\phi}$. If the following condition holds:
\begin{equation*}
\label{eq:ki}
\begin{aligned}
Q_{\mathrm{tot}}^{\theta,\phi}(\tau,\mathbf{a})
&\ge
Q_{\mathrm{tot}}^{\theta,\phi}(\tau,\mathbf{a}')
\Rightarrow
Q_{\mathrm{tot}}^{\hat{\theta},\tilde{\phi}}(\tau,\mathbf{a})
\ge
Q_{\mathrm{tot}}^{\hat{\theta},\tilde{\phi}}(\tau,\mathbf{a}'), 
\; \exists!~k:\ a_k\neq a_k'.
\end{aligned}
\end{equation*}
then, the two functions \( g \) and \( h \) satisfy the Knowledge Invariant Principle for \( Q_{\mathrm{tot}}^{\theta,\phi} \), where \(\left[Q_i\right]_{i=1}^N\) is individual agent utility function, \( N \) is the number of agents, \( \tau_i \) and \( a_i \) are the observation-action history and action of agent \( i \), respectively. \( \exists! \) represents the concept of unique existence.
\end{definition}

In this work, we relax the $\exists!~k:\ a_k\neq a_k'$ condition in KI, and propose the KI++ principle (see~\cref{def:ki++}).

\begin{figure}[!th]
	\centering
%	\begin{minipage}[b]{0.49\linewidth} % 第一行画图结果 一个现象，两个指标
%		\centering
		\includegraphics[width=0.32\columnwidth]{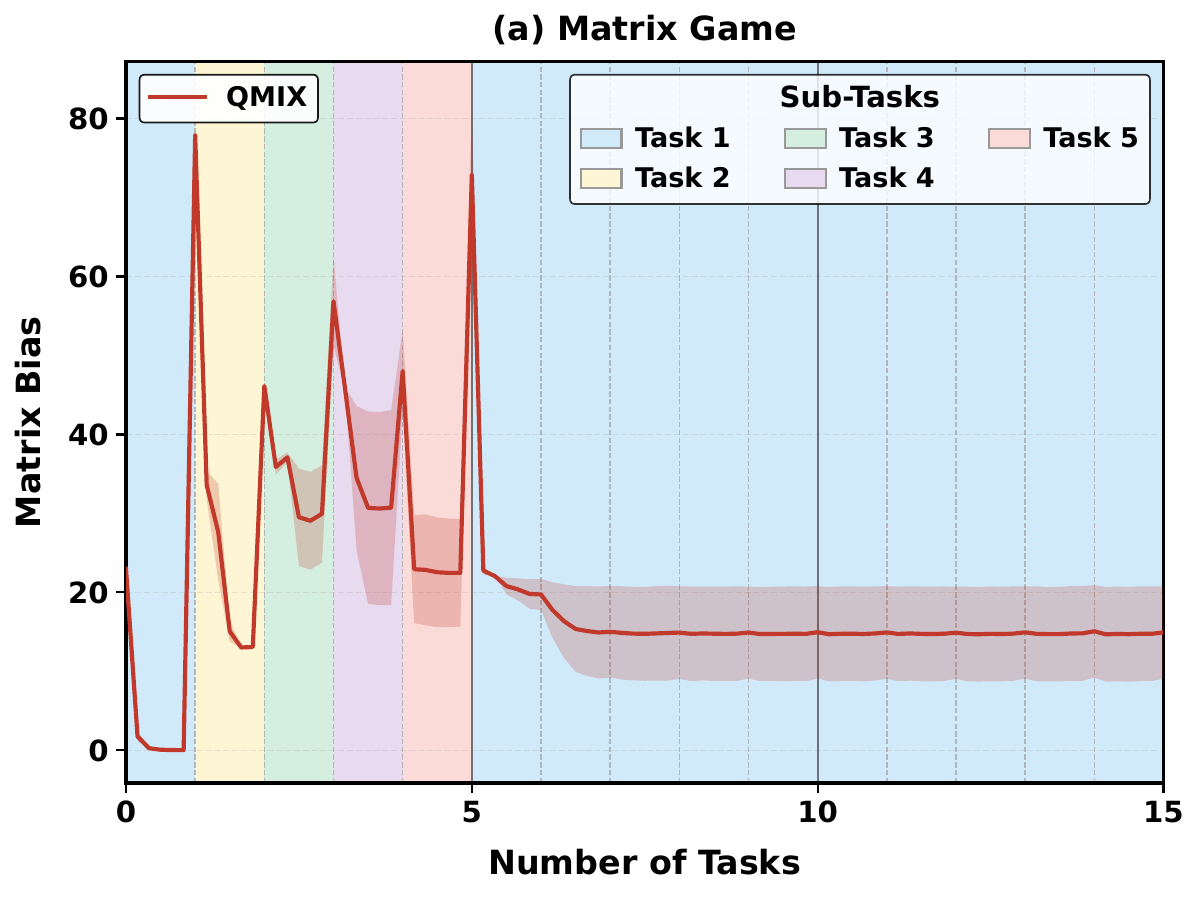} 
        \includegraphics[width=0.32\columnwidth]{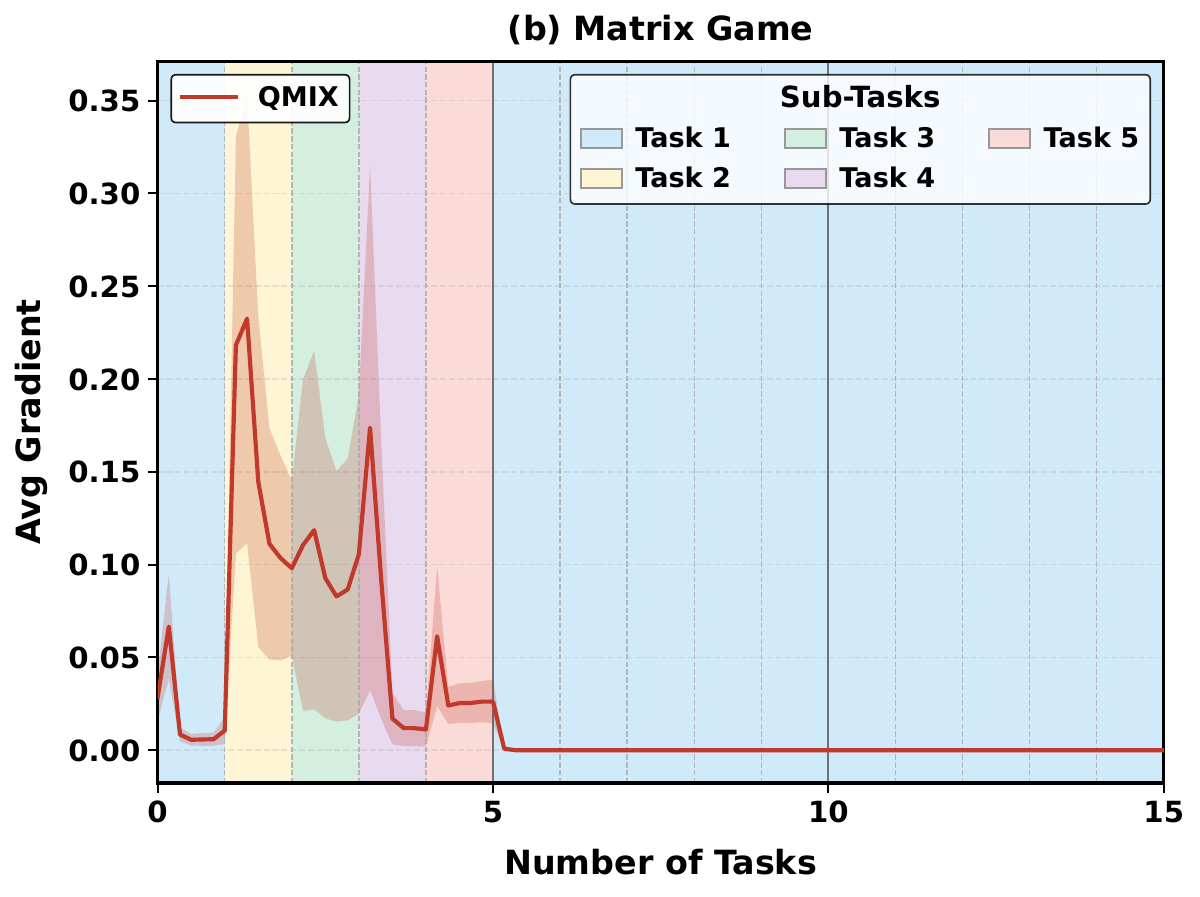} 		
        \includegraphics[width=0.32\columnwidth]{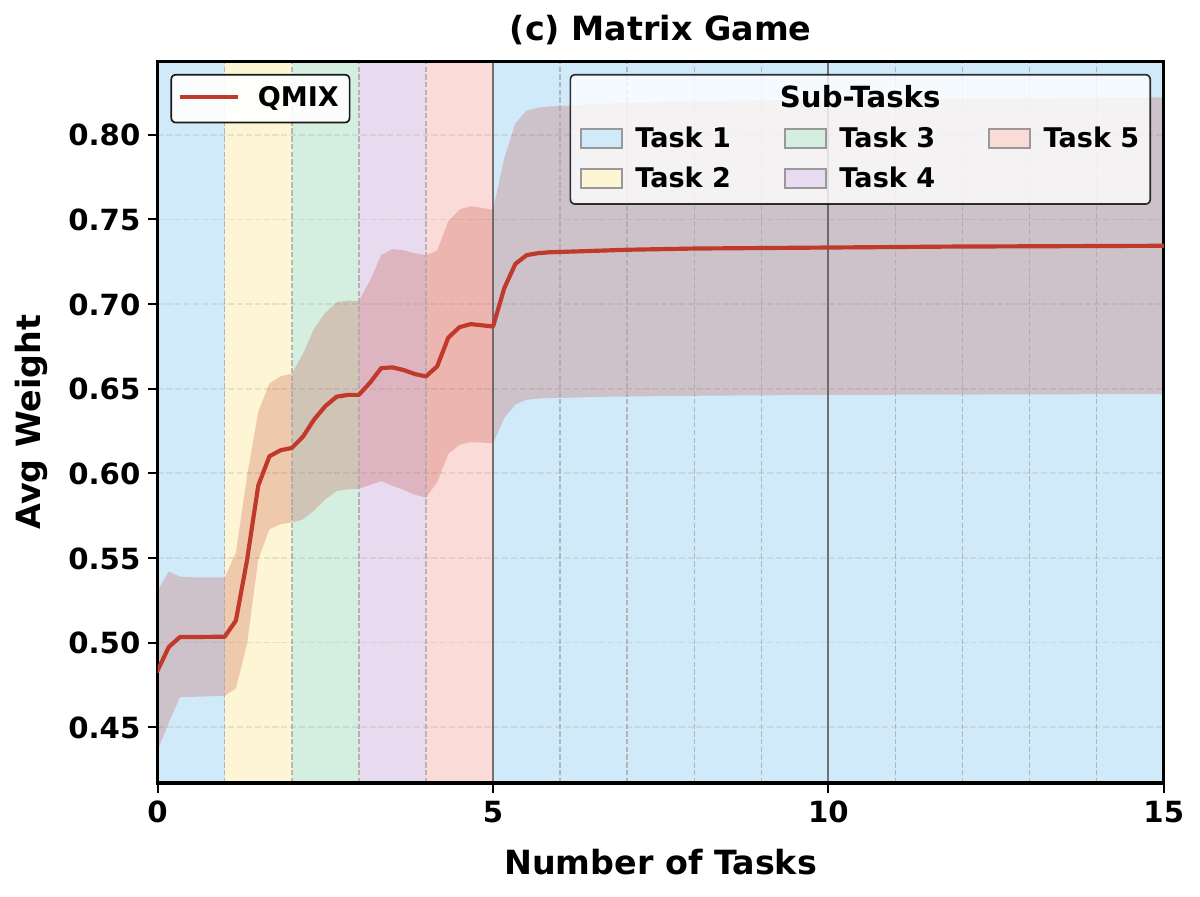} 		
  %       \includegraphics[width=0.49\columnwidth]{ICML26/experiment/stagnant/[1.4]qmix_multi_matrix_game_matrix_bias_task0.pdf}
		% \includegraphics[width=0.49\columnwidth]{ICML26/experiment/stagnant/[1.5]multi_matrix_game_matrix_bias_task1.pdf} 	
        % \includegraphics[width=0.49\columnwidth]{ICML26/experiment/stagnant/[1.3]multi_matrix_game_matrix_bias_task1.pdf}
%	\end{minipage}
        \caption{Two agents sequentially learn 5 different tasks, and then they go back to Task 1:
        (a) the matrix approximation bias (b) the average gradient norm (c) the average weight norm.}
  %        \caption{Two agents learns a sequence of 4 new tasks periodically. (a) Four one-step matrix games, the matrix approximation bias under all tasks across cycles for (b) QMIX and (c) QPLEX (d) Bias for QMIX under Task 1 across cycles.}
	\label{stagnant:matrix}
	\vspace{-0.3cm}
\end{figure}

\section{Related Work}\label{sec:related}

\subsection{Continual Reinforcement Learning (CRL)}

In Continual reinforcement learning (CRL)~\cite{abel2023definition}, the agent should continually learn diverse policies in changing environments and tasks. The prior work generally falls into three categories: regularization-based, architecture-based, and replay-based approaches. Regularization-based methods constrain parameter updates to preserve prior knowledge~\cite{kirkpatrick2017overcoming, zenke2017continual, aljundi2018memory, lewandowski2024learning}. Architecture-based methods mitigate interference through structural modularization~\cite{mallya2018packnet, liu2024neuroplastic}, while replay-based methods rehearse past experience~\cite{atkinson2021pseudo, wang2025experience}. Research in continual MARL remains limited, yet it requires both individual policies and cooperative strategies to adapt to other tasks~\cite{S2Q}. The continual MARL benchmark from MEAL~\cite{tomilin2025meal} shows that directly applying continual methods to MARL could fail.
\subsection{Plasticity Loss}

%CBP~\cite{dohare2021continual} maintains adaptability by periodically replacing selected parts of the network, introducing architectural plasticity during training.

We focus on plasticity injection methods which change/pertube the parameters of neurons. CBP~\cite{dohare2021continual} maintains adaptability by periodically replacing parts of neural networks. Redo~\cite{redo} identifies the existence of dormant neurons, whose activation score is low during the training procedure. ReGraMa~\cite{regrama} goes further by detecting whether gradients actually update a neuron, and then resets inactive ones. Reset~\cite{ReSet2} improves the plasticity through periodic resets of the weights of neurons in the last layer of a neural network. \cite{nature24plasticity} proposes continue backpropagation algorithm which periodically re-initiates some less-used neurons. MARR~\cite{yang2024sample} periodically resets weights for plasticity loss caused by a high replay ratio. Reborn~\cite{reborn} perturbs the parameters of dormant neurons in the MARL mixing network. Plasticity Injection (PI)~\cite{plasticity2} injects plasticity by injecting multiple last neural network layers. In contrast, our proposed KNIFE targets plasticity loss explicitly at the neuron level to avoid blind layer duplication. Please refer to Appendix \ref{sec:appendix_pi} and \ref{sec:appendix_init} for detailed comparison.

Our work, KNIFE, focuses on MARL value factorization methods~\cite{MARLSurvey,QMIX,wqmix,ResQ,riskq,Qatten,updet}. Moreover, it works on other MARL methods such as MADDPG~\cite{MADDPG}, MAPPO~\cite{MAPPO}, and DGN~\cite{DGN}.

\section{The Stagnant Neuron in MARL}

% 1.定义指标
% 2.摆烂现象存在marl，用smac举例
% 2.1摆烂神经元主要存在混合网络 已经跑了MMM2 ,3s_vs_5z, 2c_vs_64zg（如果这两个没有只能在v2上面说了）
% 2.2 摆烂神经元影响对新任务学习（人为构造不同个数的摆烂神经元？mask数量0,1,2,3，让他权重不更新）单步矩阵
% 2.3 摆烂神经元不怎么变化（热力图，标注谁摆烂）Overlap Coefficient
% 2.4 捣蛋鬼与摆烂的联系(摆烂越多捣蛋鬼越多)
% 3.加剧摆烂指标的原因
% 3.1 target non-stationarity/explicit goal switches exacerbate stagnation，越频繁越加剧 用500k的图切换任务只画摆烂率(这个要用不同的频率，至少三个)
% 3.2 rnn加深（rnn加深/ 性能变差，摆烂变多）
% 3.3 

% 本章旨在系统刻画一种在多智能体强化学习（MARL）值分解网络中普遍存在的可塑性退化现象：网络中部分神经元虽然并未“死亡”（并非持续零激活），但其参数更新在训练中逐渐失去有效性，表现为“活着但学不动”。我们称这类神经元为摆烂神经元（stagnant neurons），并将其根本病理归因于优化惰性（Optimization Inertia）：权重范数逐渐膨胀、而相对更新率趋近于零，导致该神经元在后续训练与分布变化下缺乏适应能力。

% 本章首先给出神经元可塑性指标与摆烂判据（Sec. 4.1），随后在 SMAC/SMACv2 环境中展示摆烂现象的存在性与普遍性（Sec. 4.2），并进一步分析其空间分布（更集中于 mixing network）、其对新任务适应的危害、其跨时间的身份稳定性（“同一批神经元长期摆烂”）以及其与“捣蛋鬼”（高尾部过激更新神经元）之间的联系（Sec. 4.2.1–4.2.4）。最后，我们用受控实验揭示两类会显著加剧摆烂的因素：显式目标非平稳性与 RNN 深度（Sec. 4.3）。

In this section, we study the loss of plasticity in MARL. We find plasticity loss in MARL across multiple benchmarks and value factorization methods during task transfers. We identify a critical source of this plasticity loss during continual adaptation: stagnant neurons, whose weights become excessively large while their relative updates are negligible. Stagnant neurons occur across multiple MARL methods and mainly reside in the mixing network. Ultimately, these stagnant neurons severely hinder MARL learning and adaptation when transferring to new environments.

\subsection{The Plasticity Loss during Sequential Task Transfers}

We study the plasticity loss of MARL during sequential task transfers by modifying a one-step two-agent matrix game and the StarCraft Multi-Agent Challenge v2 (SMACv2)~\cite{smacv2} for continual learning scenarios.

In the one-step matrix game, each unique payoff matrix constitutes a task. The agents are required to sequentially transfer across five distinct tasks (Task 1 to 5) before returning to Task 1 to evaluate their re-adaptation capacity. Figure~\ref{stagnant:matrix} (a) depicts the overall matrix approximation bias of QMIX across this continual adaptation process. Furthermore, Figure~\ref{stagnant:matrix} (b) and (c) delve into the micro-level dynamics of the network during these transfers, illustrating the average gradient and the average weight magnitude, respectively. The horizontal axis represents the sequential transfer stage, while the vertical axes indicate the corresponding metrics. The sequence of tasks is visually separated by distinct background rectangles, with the color indicating the currently active task environment.

% In the one-step matrix game, each unique payoff matrix constitutes a task. The agents must learn a sequence of four tasks (Task 1–4), with the active payoff matrix switching periodically at fixed intervals, as illustrated in Figure~\ref{stagnant:matrix} (a). Figure~\ref{stagnant:matrix} (b) depicts the overall approximation bias of QMIX across all tasks, respectively. The horizontal axis represents the cycle number, while the vertical axis indicates the bias. Each cycle consists of four consecutive tasks, represented by blue, yellow, green, and pink rectangles, with the color indicating the currently active task.

%Figure~\ref{stagnant:matrix} (b) and (c) depict the overall approximation bias of QMIX and QPLEX across all tasks, respectively. In these figures, the horizontal axis represents the cycle number, while the vertical axis indicates the bias. Each cycle consists of four consecutive tasks, represented by blue, yellow, green, and pink rectangles, with the color indicating the currently active task.

\begin{figure}[!th] %这里放112311和  112311 + 3s vs 5z的摆烂率
	\centering
	
   % \begin{minipage}[b]{\linewidth} % 第一行画图结果
		%\centering
        \includegraphics[width=0.32\columnwidth]{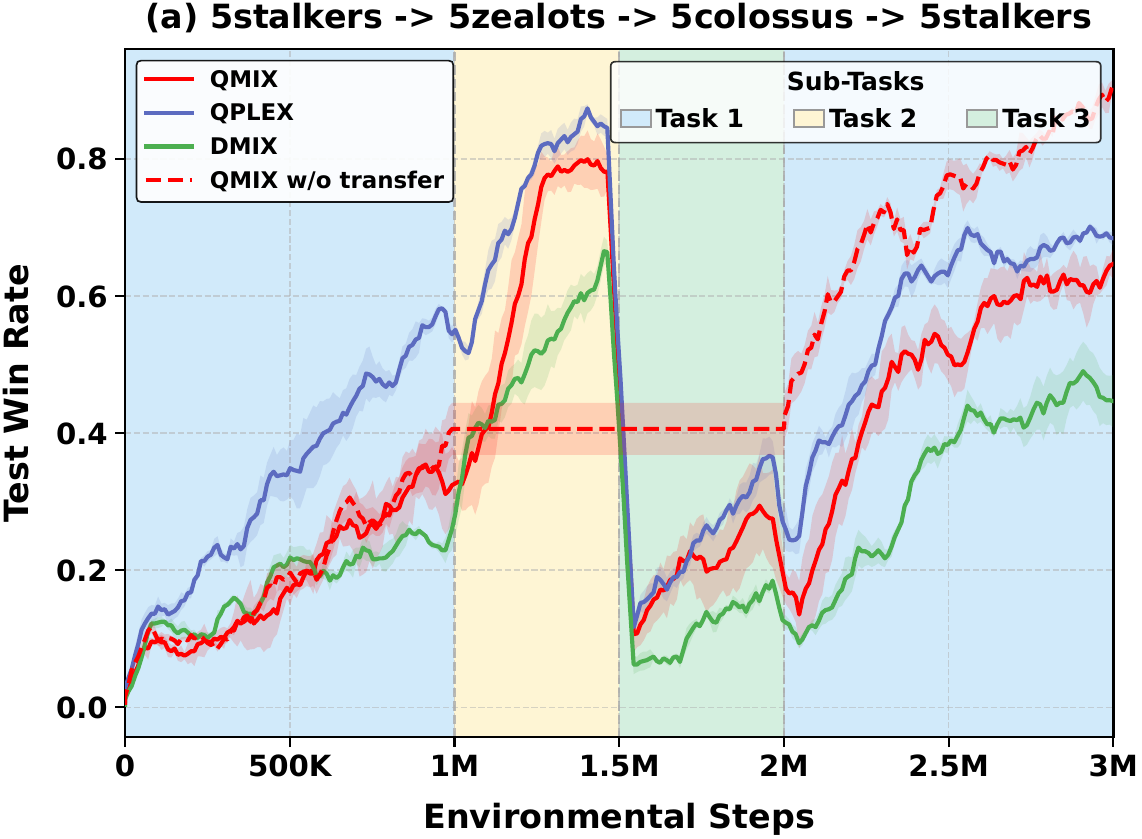} 
        \includegraphics[width=0.32\columnwidth]{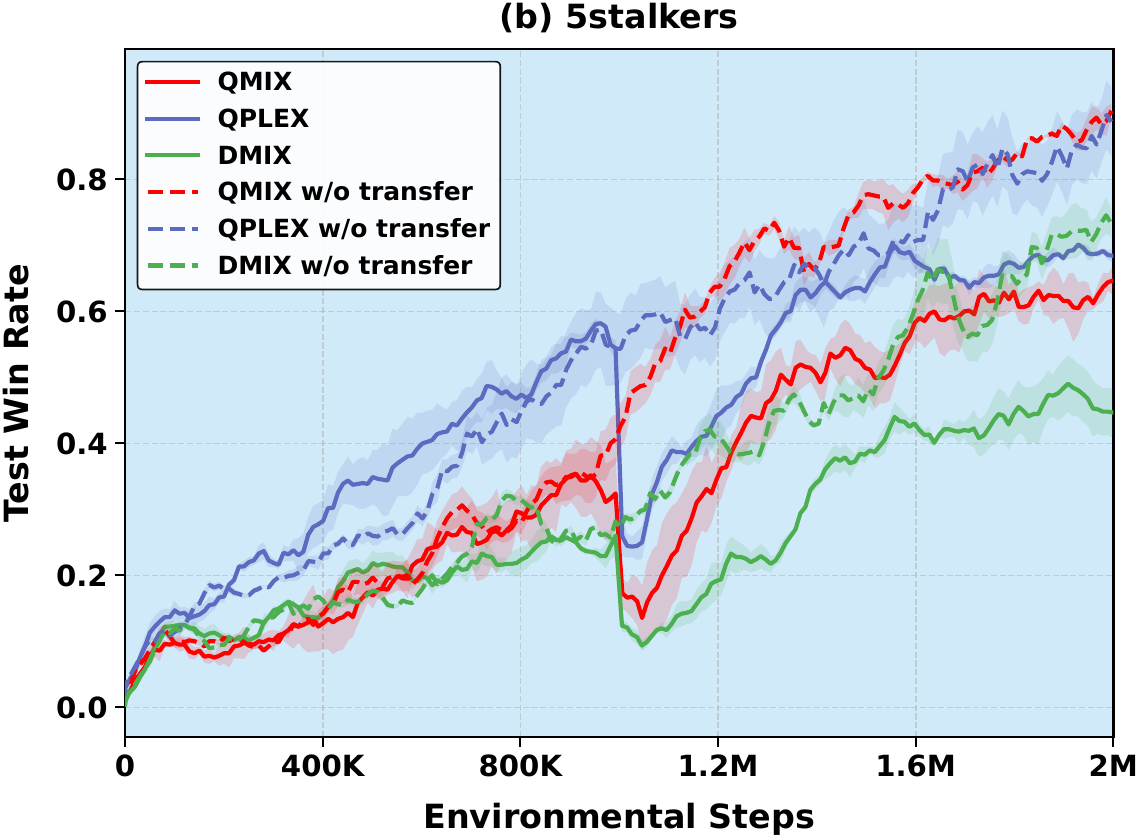} 	
		\includegraphics[width=0.32\columnwidth]{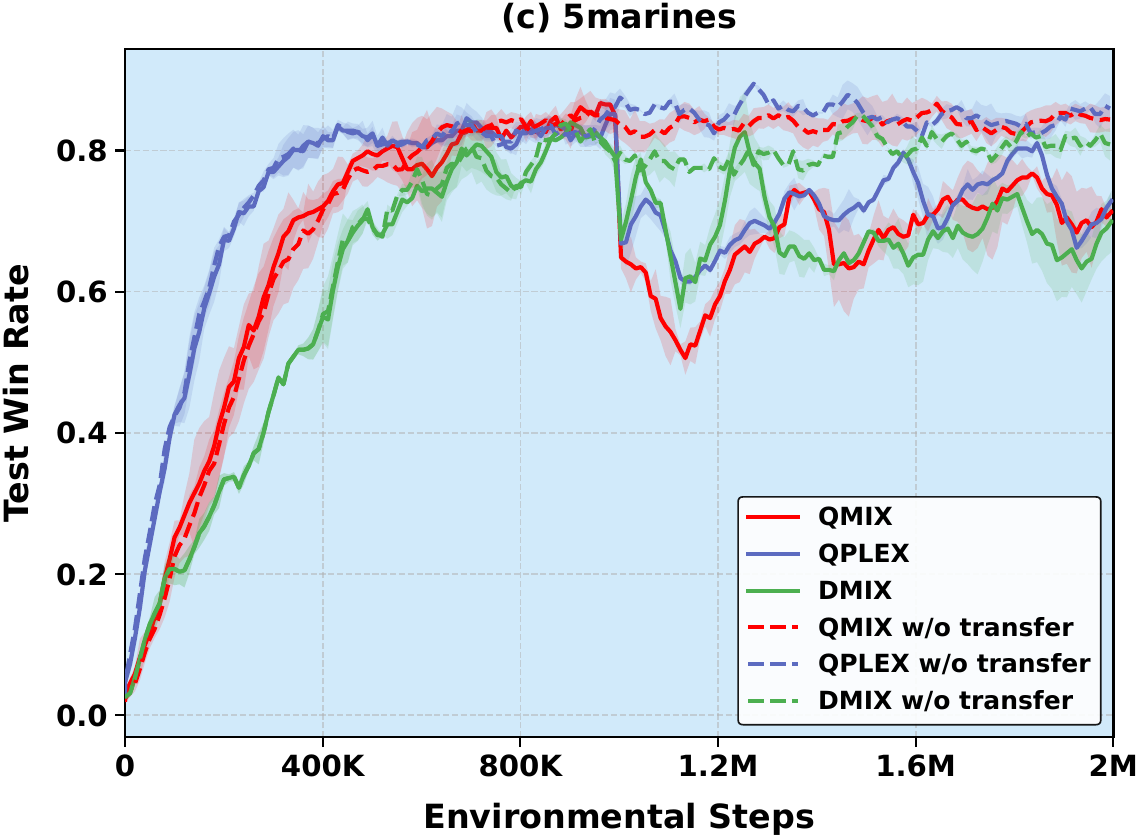} 
        % \includegraphics[width=0.24\columnwidth]{ICML26/experiment/stagnant/[3.3]3s_vs_5z_Ratio_Combined.pdf} 
%		\includegraphics[width=0.32\columnwidth]{ICML26/experiment/stagnant/[2.3]5gen_protoss_tasks_v6_500k_task1.pdf}
%	\end{minipage}
        \caption{Plasticity loss in modified SMACv2 with sequential tasks. (a) Overall win rate for all Protoss tasks, (b) Win rate for Protoss task 1, (c) Win rate for Terran task 1. } %and (d) 3s\_vs\_5z.  The tasks switch between 5 Stalkers vs. 5 Stalkers and 5 Zealots vs. 5 Zealots.
%        \caption{Periodic task switching reveals plasticity loss in modified SMACv2. The tasks switch between 5 Stalkers vs. 5 Stalkers and 5 Zealots vs. 5 Zealots. (a) Overall win rate; (b) Task win rate compared to single-task training (“w/o new task”).}
	\label{plasticity:smacv2}
	\vspace{-0.3cm}
\end{figure}

\begin{figure}[!th] %这里放112311和  112311 + 3s vs 5z的摆烂率
	\centering
	
        \includegraphics[width=0.32\columnwidth]{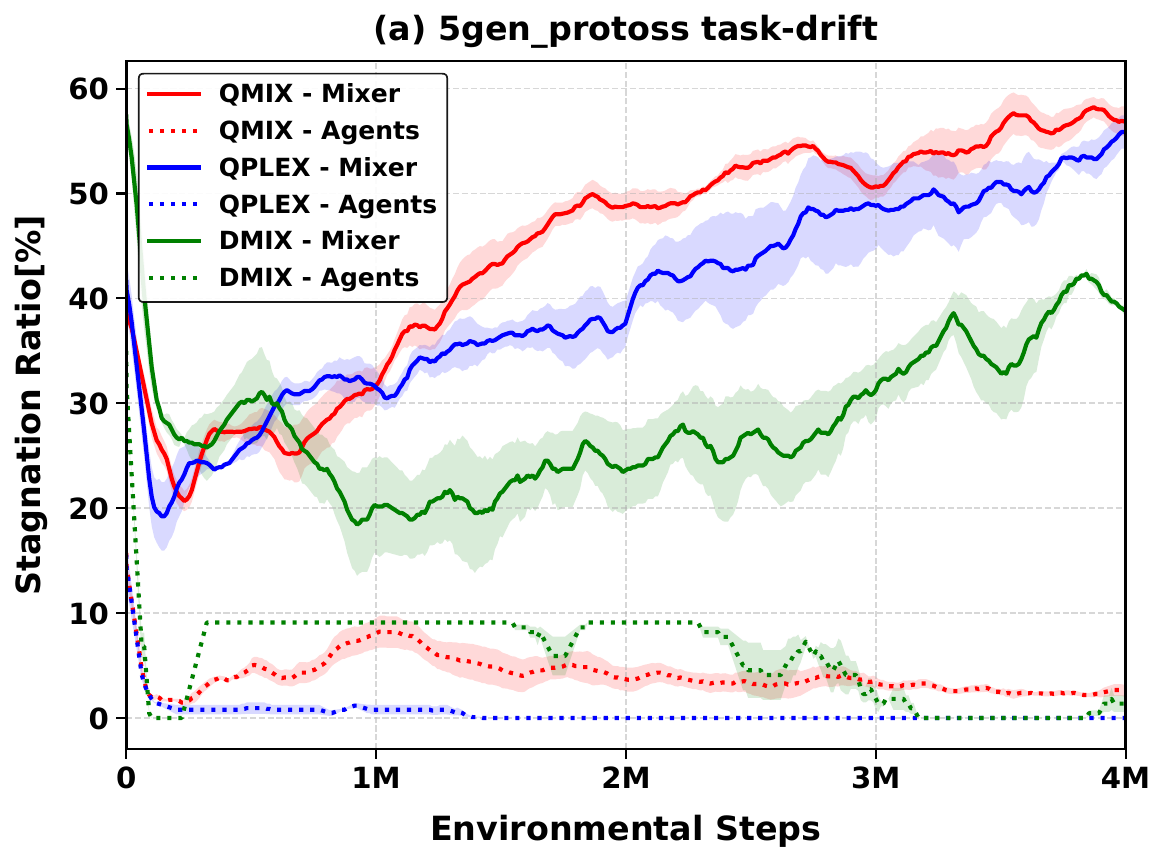} 
        \includegraphics[width=0.32\columnwidth]{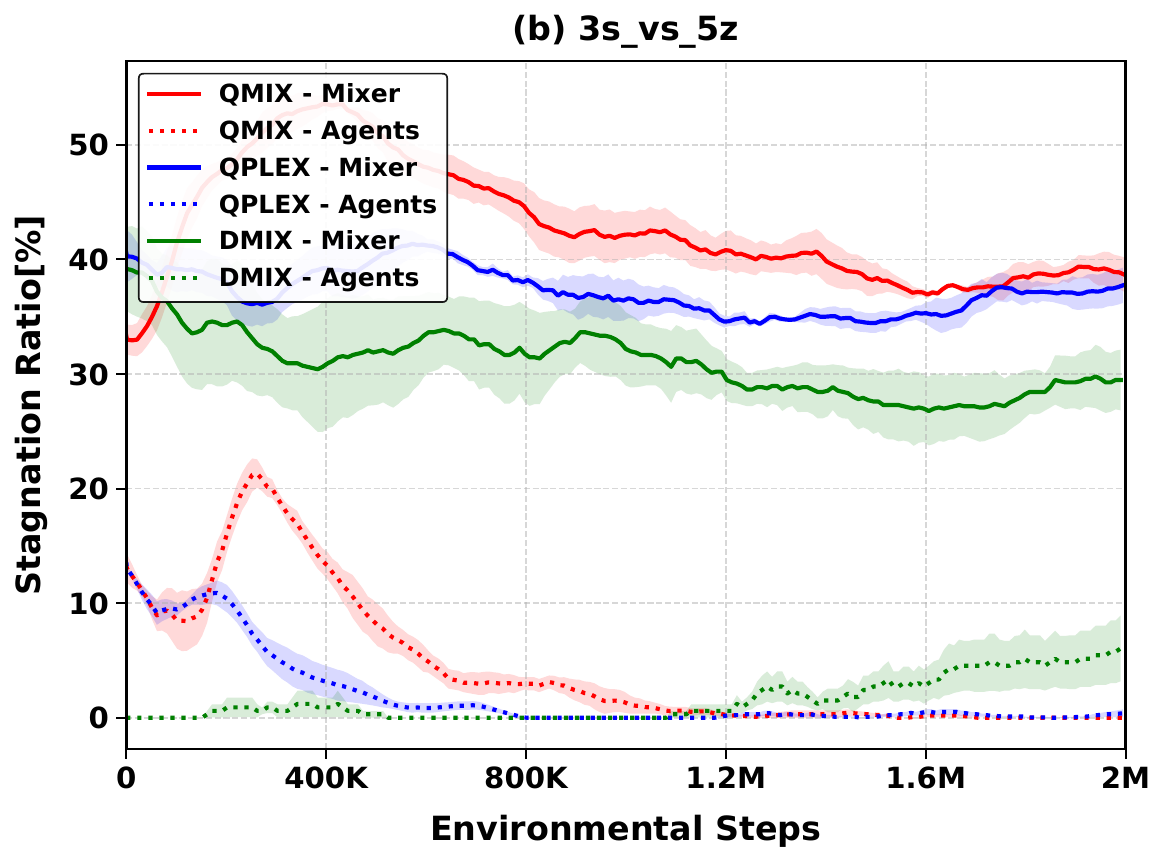} 
        \includegraphics[width=0.32\columnwidth]{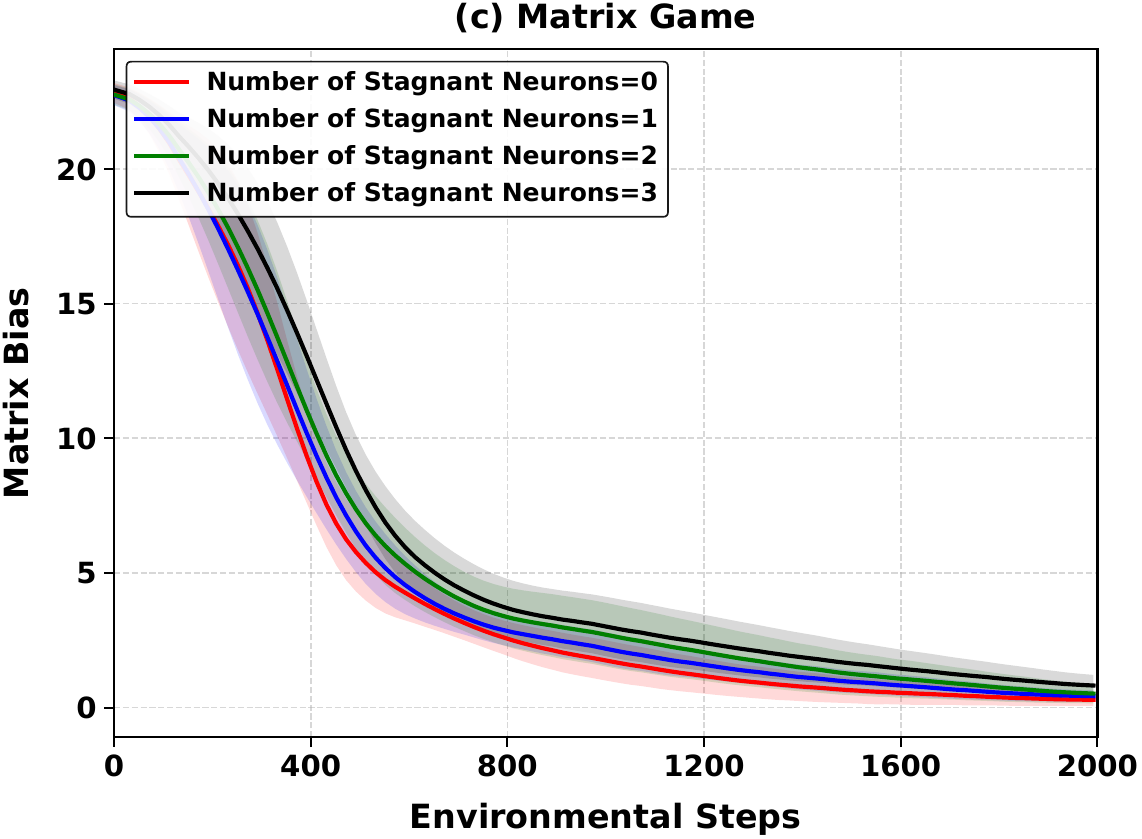} 		
		%\includegraphics[width=0.32\columnwidth]{ICML26/experiment/stagnant/[4.2]Stagnant_Ratio_Volatile_neurons_Matrix_Game.pdf} 	
       % \includegraphics[width=0.32\columnwidth]{ICML26/experiment/stagnant/[3.3]3s_vs_5z_Ratio_Combined.pdf} 
%		\includegraphics[width=0.32\columnwidth]{ICML26/experiment/stagnant/[2.3]5gen_protoss_tasks_v6_500k_task1.pdf}
%	\end{minipage}
        \caption{Stagnant neurons are prevalent across environments, and are consistently more concentrated in the mixer network than in the agent. Stagnant neurons correlate with increase matrix bias.  (a) SMACv2 5 protoss task-drift, (b) SMAC 3s\_vs\_5z. (c) Matrix bias with different stagnant neurons.}
	\label{stagnant:smacv2andmatrix-change}
	\vspace{-0.3cm}
\end{figure}

% \begin{figure}[!th]
% 	\centering
	
% %    \begin{minipage}[b]{\linewidth} % 第一行画图结果
% 		\centering
%         \includegraphics[width=0.32\columnwidth]{ICML26/experiment/stagnant/[3.1]5gen_protoss_tasks_v6_500k_Ratio_Combined.pdf}
% %        \includegraphics[width=0.49\columnwidth]{ICML26/experiment/stagnant/[3.2]5gen_protoss_Ratio_Combined.pdf} 
% 		 \includegraphics[width=0.32\columnwidth]{ICML26/experiment/stagnant/[3.3]3s_vs_5z_Ratio_Combined.pdf} 	
%     \includegraphics[width=0.32\columnwidth]{ICML26/experiment/stagnant/[3.3]3s_vs_5z_Ratio_Combined.pdf} 	
% 	% \end{minipage}
%         \caption{Stagnant neurons are prevalent across environments, and are consistently more concentrated in the mixer network than in the agent networks: (a) 5gen\_protoss task switching and (b) 3s\_vs\_5z.}
% 	\label{stagnant:analysis1}
% 	\vspace{-0.3cm}
% \end{figure}

The model exhibits a progressive loss of plasticity. As illustrated in Figure~\ref{stagnant:matrix} (a), when the agents return to the initially mastered Task 1 (the final extended phase), the matrix approximation bias plateaus at a significantly higher level compared to its performance during the initial exposure. This macroscopic degradation is directly explained by the micro-level dynamics detailed in Figure~\ref{stagnant:matrix} (b) and (c). Upon returning to Task 1, the average gradient vanishes to near zero, while the average weight magnitude norm is much larger than the weight norm. This stark contrast reinforces that the network has not merely forgotten the task, but rather, the plasticity loss is systemic, fundamentally hindering the agents' capacity to re-adapt.

In the modified SMACv2, two teams of five agents combat each other. Each task switches the agents on \emph{both} sides on a fixed schedule. Figure~\ref{plasticity:smacv2}(a) compares the win rates of QMIX, QPLEX, and DMIX under this continual task setting. Figure~\ref{plasticity:smacv2}(b) focuses on one specific scenario (task 1), which stitch the curves of two task 1 together from (a). Figure~\ref{plasticity:smacv2}(c) shows the performance of the Terran in Task 1. Given the same training time, these is a clear performance drop across methods that are trained consistently without task-transfer (e.g., QMIX w/o transfer) compare to their transfer counter-part (e.g., QMIX). This demonstrates a loss of plasticity for them under continue learning setting. %For detailed settings, please refer to the appendix~\ref{app:fig:112311}.

%Together, these results demonstrate that MARL plasticity diminishes as the task sequence progresses.

%In SMACv2, two group of agents combat with each others. In each cycle, the combating agents of the two sides are changed according to a schedule at fixed time intervals. For example, in one time interval, the task is stalkers vs stalkers, and the next task at next time interval is zealots vs zealots. Figure~\ref{stagnant:smacv2} (a) depict the win ratio of different scenarios with periodical new tasks for QMIX, QPLEX, and DMIX. The performance of these methods without new tasks are plotted as well. Figure~\ref{stagnant:smacv2} (b) depict the win ratio for one of the tasks (5 stalkers vs 5 stalkers). These methods suffer perform poor compare to their non-new-task scenarios.

%QMIX suffers from plasticity loss when facing new tasks of SMACv2. 

%紧接着一个smacv2的图，三张，使用标准qmix训练当作标准w/o new task
%We study the impact of new tasks for QPLEX and DMIX. As it is depicted in Figure~\ref{stagnant:smacv2} , QPLEX and DMIX suffer from plasticity loss as well. 

\subsection{Stagnant Neurons}

%\begin{figure*}[h]

To understand potential sources of plasticity loss, we examine neurons in neural networks. We define a few metrics and neurons as follows. 
%qplex，dmix也有，smacv2图来跑，用标准的自己方法当作 w/o new task 六根线。

\begin{definition}[Relative Update Activity (RUA)]
For a neural network layer ${\ell}$, let $w_i^{\ell}$ denote the incoming weight vector of neuron $i$ and $g_i^{\ell}=\nabla_{w_i^{\ell}}\mathcal{L}$ denote its gradient. We define the \textbf{RUA}:
%We quantify neuron-wise plasticity at the parameter level. For a linear layer ${\ell}$, let $w_i^{\ell}$ denote the incoming weight vector of neuron $i$ and $g_i^{\ell}=\nabla_{w_i^{\ell}}\mathcal{L}$ denote its gradient. We define the \textbf{Update Activity (UA)}:

\begin{equation*}
\label{eq:rua}
\mathrm{UA}_i^{\ell} = \frac{\|g_i^{\ell}\|_2}{\|w_i^{\ell}\|_2 + \epsilon}, \quad
\mathrm{\overline{UA}}_i^{\ell} = \frac{1}{T}\sum_{t=1}^T\mathrm{UA}_i^{\ell}, \quad
\mathrm{RUA}_i^{\ell} = \frac{\mathrm{\overline{UA}}_i^{\ell}}{\mathbb{E}_{j\in{\ell}}[\mathrm{\overline{UA}}_j^{\ell}]}
\end{equation*}

where $\epsilon$ is a small constant for numerical stability. $\mathrm{\overline{UA}}_i^{\ell}$ is the average update activity over time interval $T$. %Since absolute scales vary across layers and training stages, we normalize $\overline{UA}$ within each layer ${\ell}$:
%\begin{equation}

%\end{equation}

% For simplicity, we exclude the bias term, as the inertia we study is dominated by the growth of $\|w_i\|$. Including bias in the definition leads to the same qualitative conclusions (Appendix~X).
\end{definition}

\begin{figure}[!th]
	\centering
	 \begin{minipage}[b]{\linewidth} % 第一行画图结果
	 	\centering
 
         \includegraphics[width=0.32\columnwidth]{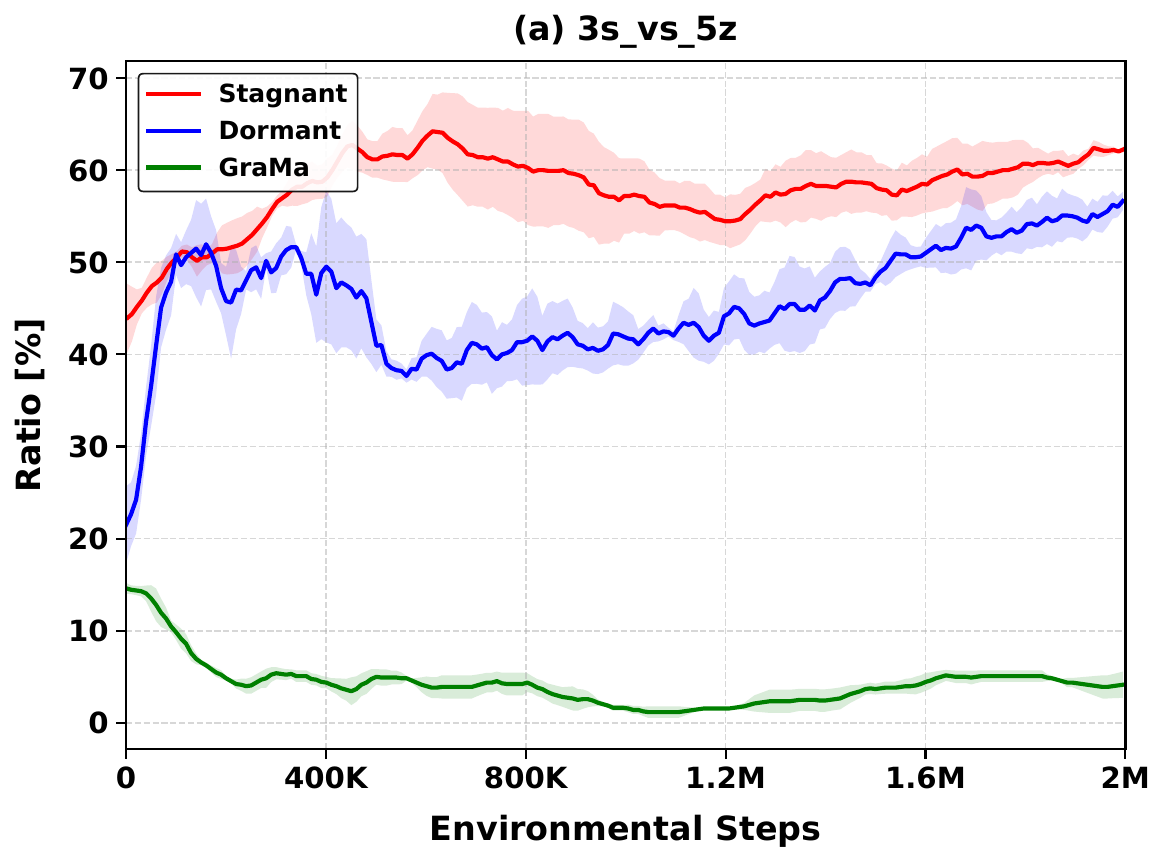}
         \includegraphics[width=0.32\columnwidth]{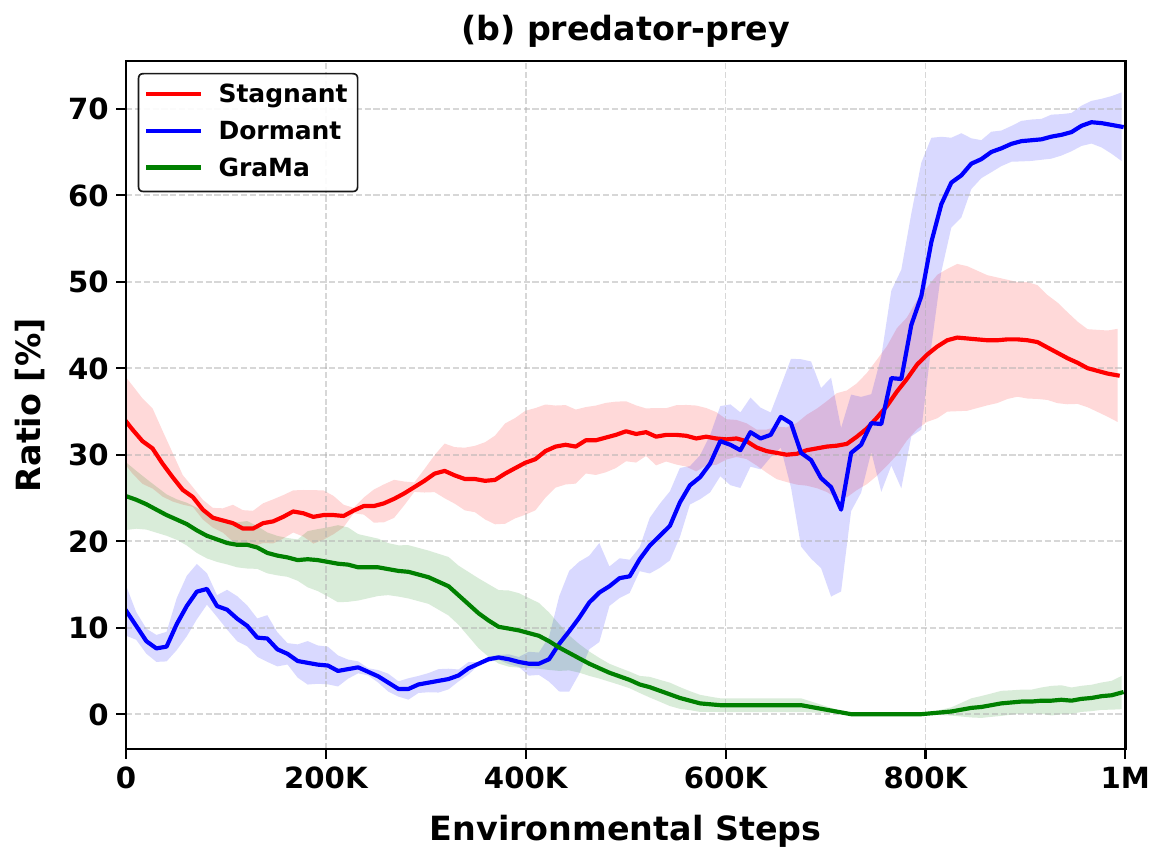} 
                \includegraphics[width=0.32\columnwidth]{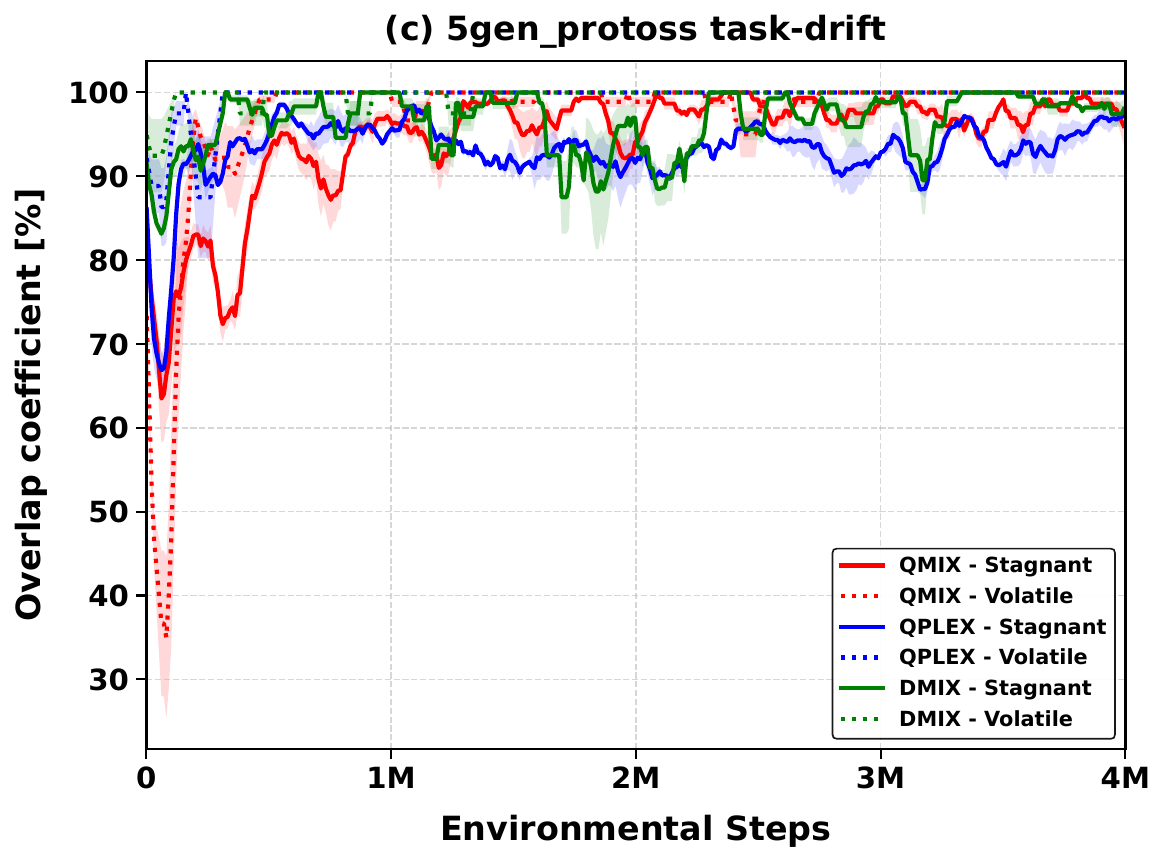} 
	 \end{minipage}
      \caption{The ratio of the Stagnant, Dormant, and GraMa neurons in (a) 3s\_vs\_5z of SMAC and (b) Predator-prey. (c) Overlap coefficient stagnant/volatile Neurons between the current iteration and the previous iteration for SMACv2 5 protoss task-drift}\label{fig:overlapAndOtherNeurons}
\end{figure}

\begin{definition}[Stagnant Neuron]
A neuron i in layer ${\ell}$ is a stagnant neuron if its score $\mathrm{RUA}_i^{\ell} < \alpha$ (i.e., 0.25). % 阈值是多少得填一下
\end{definition}

Stagnant neurons exhibit excessive optimization inertia (e.g., $\lVert \mathbf{w_i^\ell}\rVert_2 \gg 0$ while $|g_i^\ell\|_2$ is small), representing \emph{wasted capacity} that cannot adapt effectively.

\begin{definition}[Volatile Neuron]
A neuron i in layer ${\ell}$ is a volatile neuron if its score $\mathrm{RUA}_i^{\ell} > \beta$. %(i.e., 3). % 阈值是多少得填一下
\end{definition}

Figure~\ref{stagnant:smacv2andmatrix-change} (a) depicts the ratio of stagnant neurons for QMIX, QPLEX, and DMIX in the agent network and the mixing network. We find that stagnant neurons concentrate heavily in the mixing network. Interestingly, these stagnant neurons emerge not only in continual tasks but also in single-task learning scenarios. As shown in Figure~\ref{stagnant:smacv2andmatrix-change} (b), stagnant neurons still accumulate in the 3s\_vs\_5z scenario of SMAC~\cite{SMAC}, where agents learn one task only. We find that this neuron pathology is not an architectural artifact unique to hypernetworks~\cite{Hypernet} used in MARL, but persists across MLP and CNN. This accumulation is consistently accompanied by a steady decrease in Weight Spectra Entropy, confirming a systemic loss of representational expressivity and excessive optimization inertia. Detailed structural ablations and entropy analyses are provided in Appendix~\ref{sec:appendix_dehypernetwork}.

%The neural networks of MARL value factorization methods consists of mixing network and agent network. 

%引入摆烂影响效果的图，td更新间隔，任务切换频率，还有活跃神经元。
% Stagnant neurons can be one source of plasticity loss in MARL. We find that stagnant neurons can impact the learning efficiency of MARL. As shown in Figure~\ref{stagnant:analysis2}(a), the matrix bias in the one-step matrix game increases as the number of stagnant neurons grows.
% We find that with the increase in the number of volatile neurons, the stagnant ratio increases, as depicted in Figure~\ref{stagnant:analysis2} (b).
\emph{Stagnant neurons can be one source of plasticity loss in MARL}. We find that stagnant neuron correlates with the slowdown in learning efficiency. We artificially inject them by manipulating backpropagated gradients without altering forward computations (details in Appendix \ref{app:staginmarl}). As shown in Figure~\ref{stagnant:smacv2andmatrix-change}(c), the matrix bias in the one-step matrix game increases as the number of stagnant neurons grows. We find that with the increase in the number of volatile neurons, the stagnant ratio increases, as depicted in Appendix~\ref{app:sec:influOfStag} Figure~\ref{app:fig:influOfStag}.%We find that with the increase in the number of volatile neurons, the stagnant ratio increases, as depicted in Figure~\ref{stagnant:analysis2} (b).

\paragraph{Difference among other Neurons} The stagnant neuron is different from the dormant neuron~\cite{redo} and the GraMa neuron~\cite{regrama} that cause plasticity loss in RL. Figure~\ref{fig:overlapAndOtherNeurons} (a) and (b) depict the ratios of different neurons in the mixing network of QMIX for the 3s\_vs\_5z and the predator-prey environments. Their curves exhibit different trends, indicating differences among these neurons, even in this non-continual setting. Moreover, we show in Figure~\ref{exp:compare:main} that the treatments (e.g, redo) from the Dormant and GraMa neurons are ineffective for stagnant neurons.

We study whether the detected stagnant/volatile neurons persist across time. To this end, at each iteration (every time interval), we calculate the overlap coefficient between the current and previous iterations for stagnant/volatile neurons. As shown in Figure~\ref{fig:overlapAndOtherNeurons} (c), there are large overlaps between stagnant neurons across iterations, as well as the volatile neurons. This indicates that once a neuron becomes stagnant/volatile, it will remain the same type of neuron over time with high probability. As we show that these neurons can affect plasticity, \emph{a surgical intervention on these neurons is needed to alleviate the plasticity loss problem.}

\section{KNIFE: A Plasticity Injection Method}
We propose a MARL plasticity injection method dubbed KNIFE, which stands for \textbf{K}nowledge-retentive \textbf{N}euron-level Plast\textbf{I}city \textbf{F}ocusing Inj\textbf{E}ction. KNIFE is executed every $T_{k}$ steps to inject plasticity into MARL neural networks. It consists of three major operations: identification, injection, and pruning. The identification operation detects the set of neurons $\mathcal{S}^\ell_{\text{repair}}$ that require neuron surgery. The injection operation injects plasticity into the neurons in $\mathcal{S}^\ell_{\text{repair}}$. The pruning operation gradually shrinks the weights of frozen neurons and prunes frozen neurons whose weights are close to zero. A schematic plot of KNIFE is depicted in Figure~\ref{fig:method:KNIFE}.

\begin{figure*}[!tbp]% 图片放置最佳位置
	\centering
    %\begin{minipage}[t]{\textwidth}
	\includegraphics[width=0.99\textwidth]{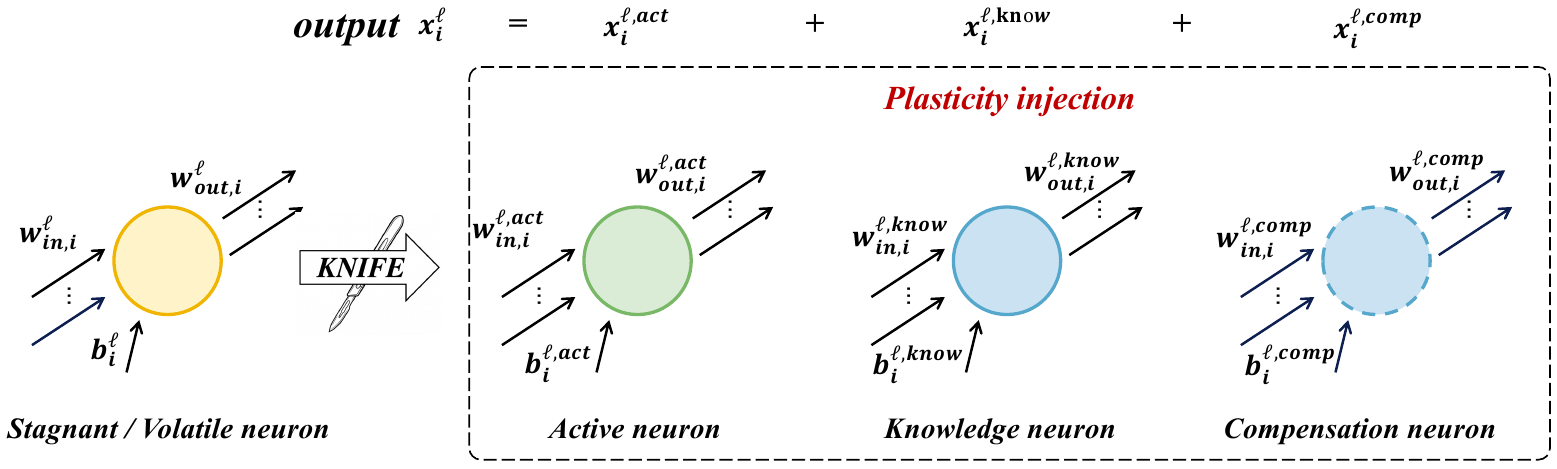}
    %	\end{minipage}
	\caption{KNIFE creates a knowledge, an active, and a compensation neuron for a stagnant neuron.}
 %\vspace{-0.3cm}
	\label{fig:method:KNIFE}
\end{figure*}

%Our method identifies neurons suffering from \textbf{Optimization Inertia} (stagnation) or \textbf{Hyper-Activity} (volatility) and revitalizes them through a mild intervention. AZTI operates in a periodic \emph{detect-and-repair} cycle with three stages: \textbf{Dual-Target Identification}, \textbf{Zero-Sum Topology Surgery}, and \textbf{Dual-Scaffold Annealing}.

% ------------------------------------------------------------
% 5.1 Periodic Dual-Target Identification
% ------------------------------------------------------------
\subsection{Identification Operation}

The identification operation identifies the set of neurons $\mathcal{S}^\ell_{\text{repair}}$ that receive neuron surgery in the plasticity injection operation, where $\mathcal{S}^\ell_{\text{repair}} = \mathcal{S}_{\text{stagnant}}^{\ell} \cup \mathcal{S}_{\text{volatile}}^{\ell} , \quad\mathcal{S}_{\text{stagnant}}^{\ell} = \{ i \mid \mathrm{RUA}_i^{\ell} < \alpha \},\quad\mathcal{S}_{\text{volatile}}^{\ell}  = \{ i \mid \mathrm{RUA}_i^{\ell} > \beta \}.$
\subsection{Plasticity Injection Operation}

For each neuron $i$ in $\mathcal{S}^\ell_{\text{repair}}$, we apply a plasticity injection surgery. This surgery aims to inject plasticity while retaining the learned cooperation knowledge. For neuron $i$, we replace it with a composite neuron consisting of an active neuron, a knowledge neuron, and a compensation neuron. They all take the same input. Their outputs are summed and used as input to the other neurons. The weights of the knowledge and the compensator neurons are frozen, whereas the weights of the active neurons are not. The knowledge neuron is used to preserve previously learned knowledge, and the compensation neuron is used to compensate for the output of the active neuron.

%Unlike aggressive reset methods (e.g., setting outgoing weights to zero), which may cause abrupt loss spikes and policy collapse in MARL, AZTI guarantees \textbf{zero perturbation at the injection step} by construction.

%For each neuron in $\mathcal{S}_{\text{repair}} \in \mathcal{S}_{\text{stagnation}} \cup \mathcal{S}_{\text{volatile}}$, we apply a topology-preserving surgery. Unlike aggressive reset methods (e.g., setting outgoing weights to zero), which may cause abrupt loss spikes and policy collapse in MARL, AZTI guarantees \textbf{zero perturbation at the injection step} by construction.

%For each target neuron in $\mathcal{S}_{\text{target}} \in \mathcal{S}_{stag} \cup \mathcal{S}_{vol}$, we apply a topology-preserving surgery. Unlike aggressive reset methods (e.g., setting outgoing weights to zero), which may cause abrupt loss spikes and policy collapse in MARL, AZTI guarantees \textbf{zero perturbation at the injection step} by construction.

Let $a_i^\ell$ denote the activation (a scalar) of neuron $i$ in layer $\ell$ as 
%\begin{equation*}
$a_i^\ell = \sigma\!\left( (\mathbf{w}_{in,i}^\ell)^\top \mathbf{x}^{\ell-1} + b_i^\ell \right)
$, where $\mathbf{x}^{\ell-1}$ is the input vector to layer $\ell$,
$\mathbf{w}_{in,i}^\ell$ and $b_i^\ell$ are the incoming weights and bias, $\sigma$ is the activation function. The output $\mathbf{x}_{i}^\ell$ of neuron $i$ to the next layer is defined 
$
\mathbf{x}_{i}^\ell = a_i^\ell \mathbf{w}_{out,i}^\ell$, where $ \mathbf{w}_{out,i}^\ell$ is the output weight of neuron $i$.
%The input of neuron $j$ from neuron $i$ is calculated as follows.

%The activation $a_i^\ell$ is sent to neuron $j$ in the next layer through outgoing weights $\mathbf{w}_{out,i,j}^\ell$.
%We define the resulting \emph{message} (i.e., the contribution to the next-layer pre-activation )
%We define the resulting \emph{message} (i.e., the input to the next-layer neuron $j$) as
%The input of neuron $j$ from neuron $i$ is calculated as follows.

As the learned knowledge is encoded in neural networks, to preserve the knowledge, it is important to ensure that $\mathbf{x}_{i}^\ell$ remains the same after the plasticity injection surgery. The surgery split the original neuron's output $\mathbf{x}_{i}^\ell$ into three parallel branches:
\begin{equation*}
\mathbf{x}_i^{\ell} =
\mathbf{x}_i^{\ell,\mathrm{know}}
+
\mathbf{x}_i^{\ell,\mathrm{active}}
+
\mathbf{x}_i^{\ell,\mathrm{comp}}
\end{equation*}

%\begin{equation*}
%\mathbf{x}_i^{\ell} =
%\underbrace{\mathbf{x}_i^{\ell,\mathrm{know}}}_{\text{frozen}}
%+
%\underbrace{\mathbf{x}_i^{\ell,\mathrm{active}}}_{\text{trainable}}
%+
%\underbrace{\mathbf{x}_i^{\ell,\mathrm{comp}}}_{\text{frozen}}
%\end{equation*}

%\begin{equation}
%\mathbf{m}_i^{\ell} =
%\underbrace{\mathbf{m}_i^{\ell,\mathrm{old}}}_{\text{frozen}}
%+
%\underbrace{\mathbf{m}_i^{\ell,\mathrm{new}}}_{\text{trainable}}
%+
%\underbrace{\mathbf{m}_i^{\ell,\mathrm{compensation}}}_{\text{frozen}}
%\end{equation}

$\mathbf{x}_i^{\ell,\mathrm{know}}$, $\mathbf{x}_i^{\ell,\mathrm{active}}$, and $\mathbf{x}_i^{\ell,\mathrm{comp}}$ are the output of the knowledge, the active, and the compensation neuron. We describe the details of these outputs and neurons in the following paragraphs.

%\paragraph{Tri-branch decomposition.}

%Here the superscripts $\mathrm{old}$, $\mathrm{new}$, and $\mathrm{cancel}$ are \emph{branch labels} (not exponents):
%$\mathrm{old}$ is the frozen copy of the original neuron, $\mathrm{new}$ is the only trainable branch, and
%$\mathrm{cancel}$ is a frozen branch constructed to cancel $\mathrm{new}$ at the injection step.

% ------------------------------------------------------------
% 5.2.1 Component Definition
% ------------------------------------------------------------

%\subsubsection{Branch specification}
%The original parameters of neuron $i$ in layer $\ell$ are denoted as$(\mathbf{w}_{in,i}^{\ell},\, b_i^{\ell},\, \mathbf{w}_{out,i}^{\ell})$. Each new neuron $r\in\{\mathrm{know},\mathrm{active},\mathrm{comp}\}$ has its own parameters $(\mathbf{w}_{in,i}^{\ell,r},\, b_i^{\ell,r},\, \mathbf{w}_{out,i}^{\ell,r})$ and output $\mathbf{x}_i^{\ell,r} = a_i^{\ell,r}\mathbf{x}_{out,i}^{\ell,r}$ 

\paragraph{Knowledge neuron (frozen).}
KNIFE copies the original stagnant neuron and freezes its parameters to rigorously preserve previously acquired cooperative capabilities:
%\begin{equation*}
$\mathbf{w}_{in,i}^{\ell,\mathrm{know}} \leftarrow \mathbf{w}_{in,i}^{\ell},\quad
b_i^{\ell,\mathrm{know}} \leftarrow b_i^{\ell},\quad
\mathbf{w}_{out,i}^{\ell,\mathrm{know}} \leftarrow \mathbf{w}_{out,i}^{\ell}.$
%\end{equation*}

\paragraph{Active neuron (state-aware initialization).}
Unlike methods (e.g., PI~\cite{plasticity2} and ReDo~\cite{redo}) that apply Kaiming initialization for neuron weights, KNIFE creates a fresh neuron branch equipped with a \textit{state-aware initialization} strategy. This design is explicitly tailored to bridge the gradient flow and prevent newly injected neurons from falling into activation dead zones (e.g., dying ReLUs). 

First, the incoming weights are sampled using Kaiming initialization:
\begin{equation*}
\mathbf{w}_{in,i}^{\ell,\mathrm{active}} \sim \mathcal{N}(0, \delta_{in}^2), \quad
\delta_{in} \leftarrow \sqrt{2/d_{\ell-1}},
\end{equation*}
where $\mathcal{N}(0,\delta_{in}^2)$ is a zero-centered Gaussian distribution with variance $\delta_{in}^2$, and $d_{\ell-1}$ is the output dimension of layer $\ell-1$. 

To increase the likelihood of positive-valued activations and ensure the neuron remains active to receive backpropagated gradients immediately, the input bias is explicitly initialized to a positive value proportional to the weight standard deviation $
b_i^{\ell,\mathrm{active}} =  0.5 \delta_{in}$.

To ensure the magnitude of the forward signal is fully compatible with the current optimization state of the existing network, the outgoing weights are initialized with a contextual, scale-matched Gaussian:
\begin{equation*}
\mathbf{w}_{out,i}^{\ell,\mathrm{active}} \sim \mathcal{N}(0,\delta_{out}^2),\quad
\delta_{out} \leftarrow \max(\mathrm{std}(w^{\ell+1}),10^{-2}),
\end{equation*}
where $w^{\ell+1}$ is the weight matrix of the subsequent layer $\ell{+}1$, and $\mathrm{std}(\cdot)$ computes the standard deviation. The $\max(\cdot)$ operation serves as a safeguard with a lower bound of $10^{-2}$, guaranteeing a viable output channel to the next layer.

% %122
% \paragraph{Knowledge neuron (frozen).}
% KNIFE copies the original neuron and freezes it:
% \begin{equation*}
% \mathbf{w}_{in,i}^{\ell,\mathrm{know}} \leftarrow \mathbf{w}_{in,i}^{\ell},\quad
% b_i^{\ell,\mathrm{know}} \leftarrow b_i^{\ell},\quad
% \mathbf{w}_{out,i}^{\ell,\mathrm{know}} \leftarrow \mathbf{w}_{out,i}^{\ell},
% \end{equation*}
% %and stop all gradients on this branch.

% \paragraph{Active neuron}
% KNIFE creates a fresh branch to restore plasticity and sample incoming weights with Kaiming initialization:

% \begin{equation*}
% \mathbf{w}_{in,i}^{\ell,\mathrm{active}} \sim \mathcal{N}(0, \delta_{in}^2), \quad
% \delta_{in} \leftarrow \sqrt{2/d_{\ell-1}},
% \end{equation*}
% where $\mathcal{N}(0,\delta_{in}^2)$ is a Gaussian distribution with mean $0$ and variance $\delta_{in}^2$, $d_{\ell-1}$ is the output dimension of layer $\ell-1$. The bias is initialized as $b_i^{\ell,\mathrm{active}} =  0.5 \delta_{in}$.
% %\begin{equation}
% %\end{equation}

% The outgoing weights are initialized with a scale-matched Gaussian
% \begin{equation*}
% \mathbf{w}_{out,i}^{\ell,\mathrm{active}} \sim \mathcal{N}(0,\delta_{out}^2),\quad
% \delta_{out} \leftarrow \max(\mathrm{std}(w^{\ell+1}),10^{-2}),
% \end{equation*}

% where $w^{\ell+1}$ is the input weight matrix of layer $\ell{+}1$,  $\mathrm{std}(\cdot)$ is the standard deviation, $\max(a,b)$ choose the larger one between $a$ and $b$.
% %denotes the standard deviation over all entries of the matrix.

\paragraph{Compensation neuron (frozen).}
We construct a compensation branch to compensate the active branch at the injection step by sharing the same incoming parameters and
negating the outgoing weights $
\mathbf{w}_{in,i}^{\ell,\mathrm{comp}} \leftarrow \mathbf{w}_{in,i}^{\ell,\mathrm{active}},\;b_i^{\ell,\mathrm{comp}} \leftarrow b_i^{\ell,\mathrm{active}},\;
\mathbf{w}_{out,i}^{\ell,\mathrm{comp}} \leftarrow -\,\mathbf{w}_{out,i}^{\ell,\mathrm{active}}$

%and we stop all gradients on this branch.

%\paragraph{Zero perturbation at injection.}
%At the injection moment, we initialize the new and cancel branches to have identical pre-activations, thus $a_i^{\ell,\mathrm{cancel}} \equiv a_i^{\ell,\mathrm{new}}$. Meanwhile, we set $\mathbf{w}_{out,i}^{\ell,\mathrm{cancel}} \equiv -\mathbf{w}_{out,i}^{\ell,\mathrm{new}}$.

\paragraph{Knowledge retentive of the Neuron Surgery}

After the neuron surgery, we update only the \emph{active} neurons by gradient descent, and freeze the two other neurons by masking their gradients to zero. The output of the new big neuron is calculated as 
%\begin{equation*}
%\begin{aligned}
%\mathbf{x}_i^{\ell}&=
$\mathbf{x}_i^{\ell,\mathrm{know}}
+
\mathbf{x}_i^{\ell,\mathrm{active}}
+
\mathbf{x}_i^{\ell,\mathrm{comp}} 
=\mathbf{x}_i^{\ell,\mathrm{know}}
=\mathbf{x}_i^{\ell}.$ We show that the injection phase of KNIFE satisfies the KI++ principle (see Appendix~\ref{app:ki++}), as formally stated in Theorem~\ref{thm:main:pis_ki}. The KI++ principle generalizes the original KI principle by removing one of its restrictive conditions. Complete proofs are provided in the appendix.
\begin{theorem}
\label{thm:main:pis_ki}
For any MARL value factorization method, after the plasticity injection operation of KNIFE, the state-action value function \( Q_{tot}^{\theta,\phi}(\ve \tau, \ve a) \) satisfies the KI++ principle.
\end{theorem}

\subsection{Pruning Operation}

%The injection is function-preserving at the injection moment (Sec.~5.2): the \emph{new} and \emph{cancel} branches cancel each other, so the network output is unchanged. To enable a smooth handover, AZTI applies a \textbf{dual-scaffold annealing} schedule.

 To recycle the neuron resources, thus reducing memory consumption. The weights of frozen neurons are shrunk gradually by multiplying the weights with $\gamma\;\; 0<\gamma<1$. Once the weights approach zero, the frozen neuron is pruned.

\section{Empirical Evaluations}

\begin{figure*}[!t]
	
	\centering
 
	\begin{minipage}[t]{\linewidth} % 第一行画图结果
		\centering
        \includegraphics[width=0.32\columnwidth]{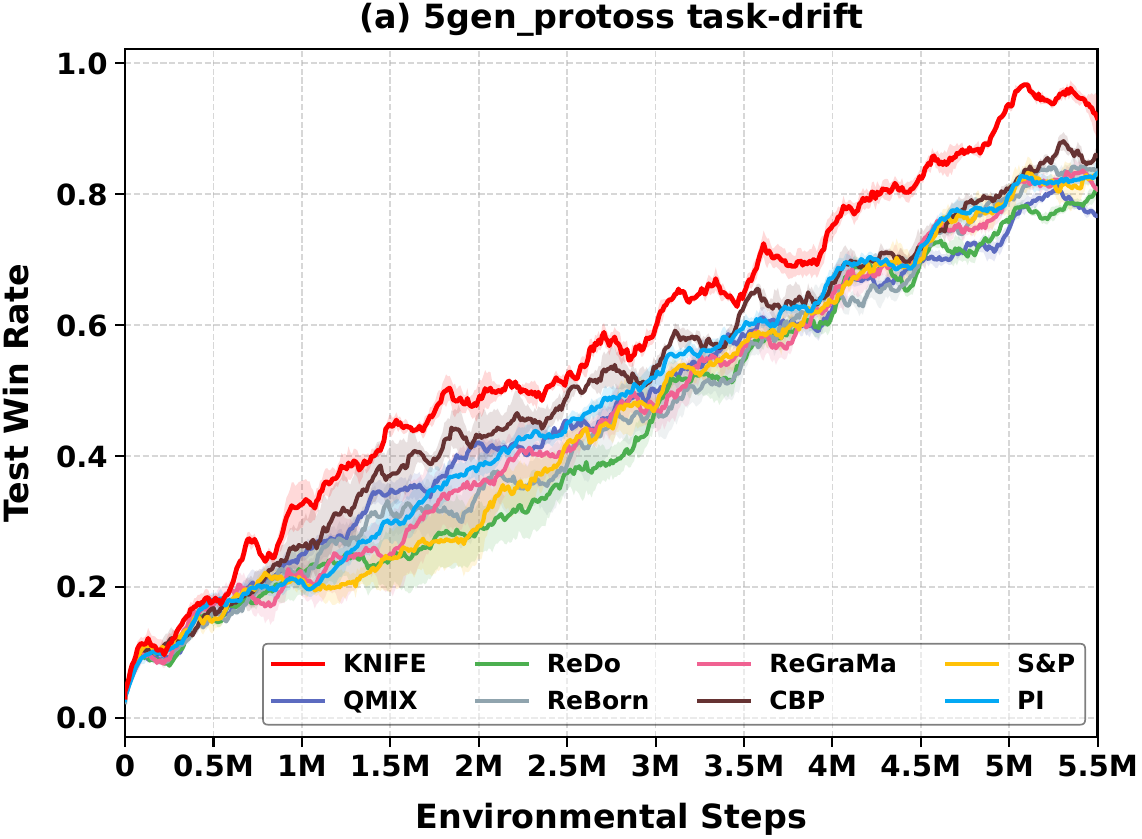} 	
            \includegraphics[width=0.32\columnwidth]{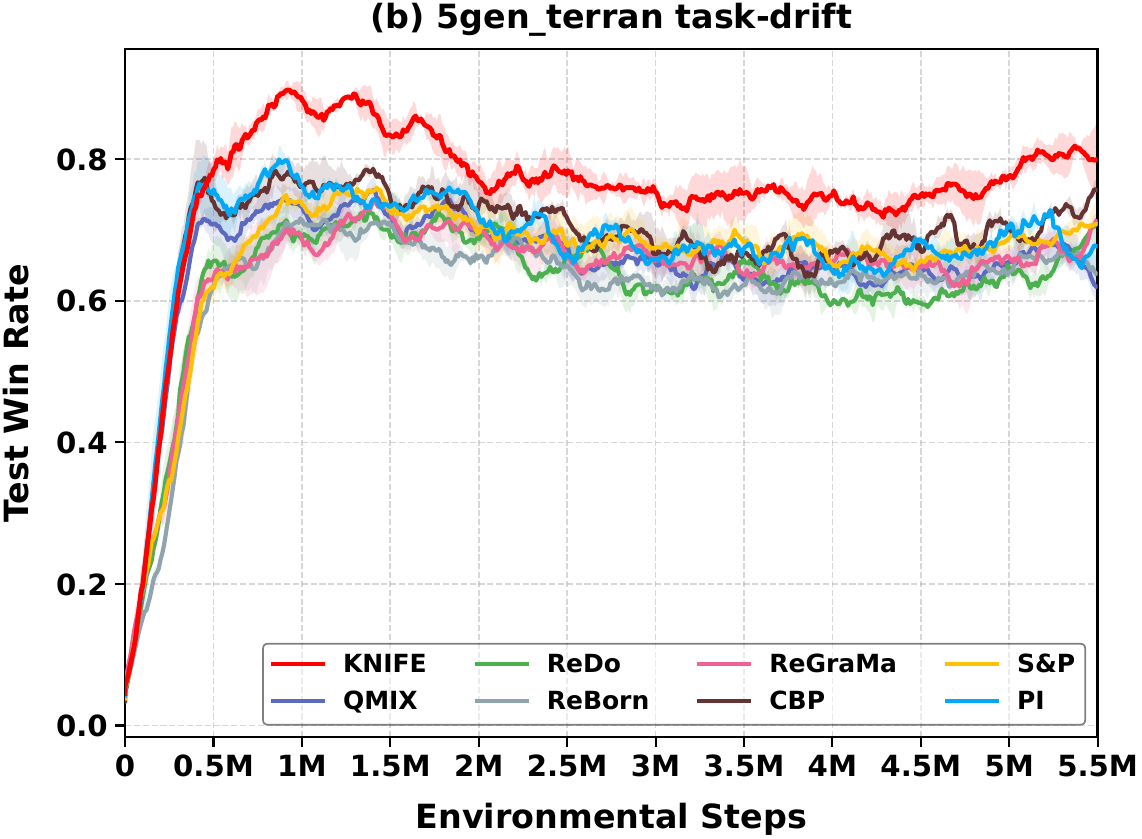} 		
            \includegraphics[width=0.32\columnwidth]{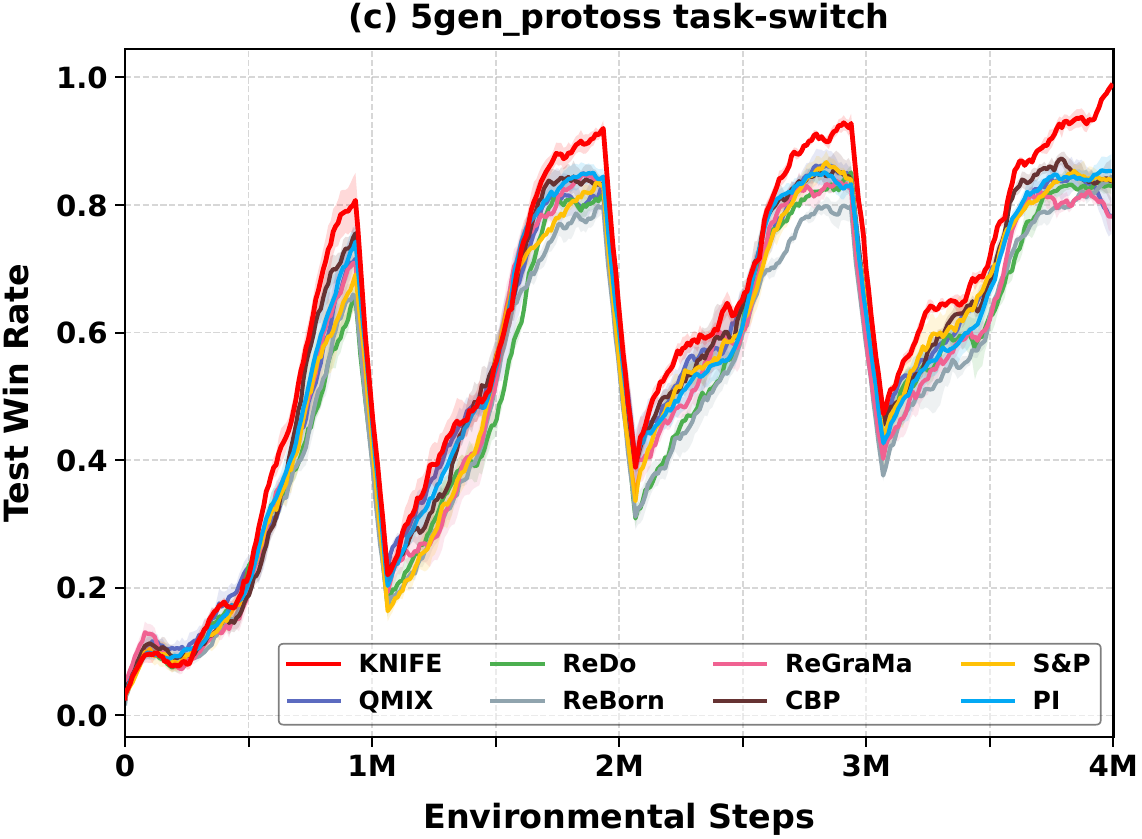} 	
            
	\end{minipage}
    \begin{minipage}[t]{\linewidth} % 第一行画图结果
		\centering
           
            \includegraphics[width=0.32\columnwidth]{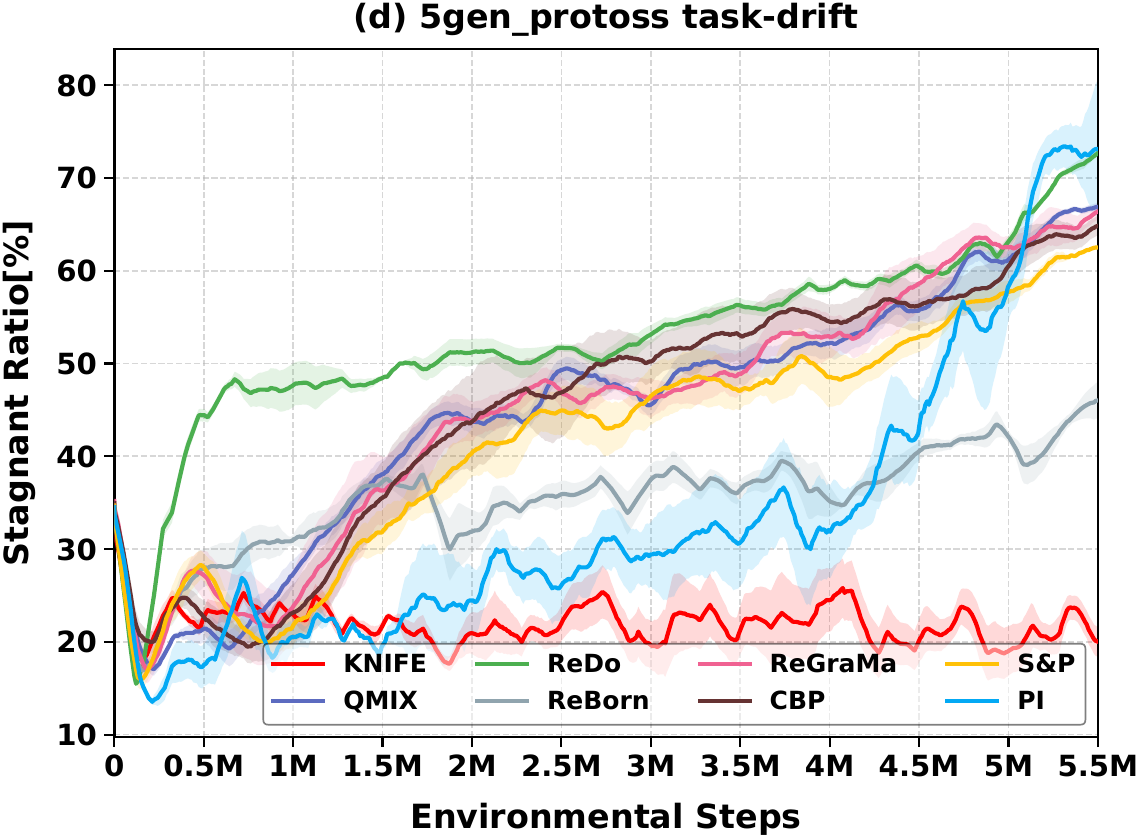} 		
		\includegraphics[width=0.32\columnwidth]{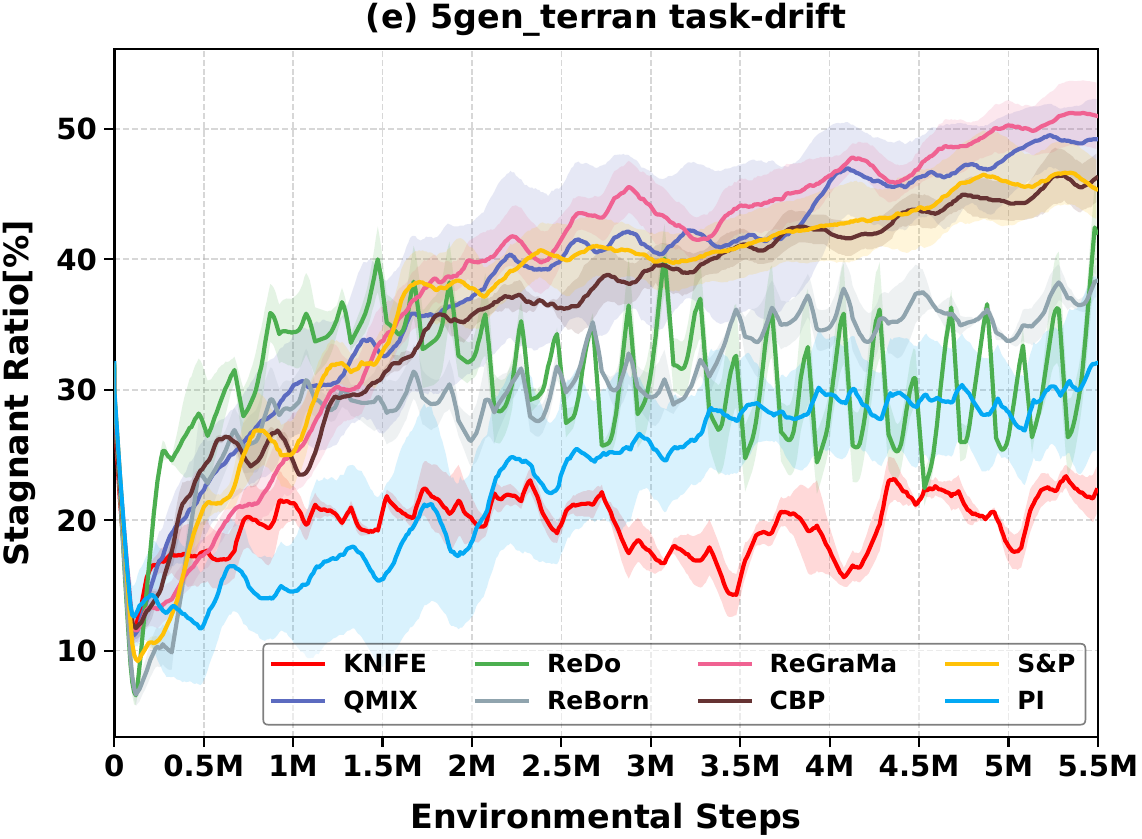} 
         \includegraphics[width=0.32\columnwidth]{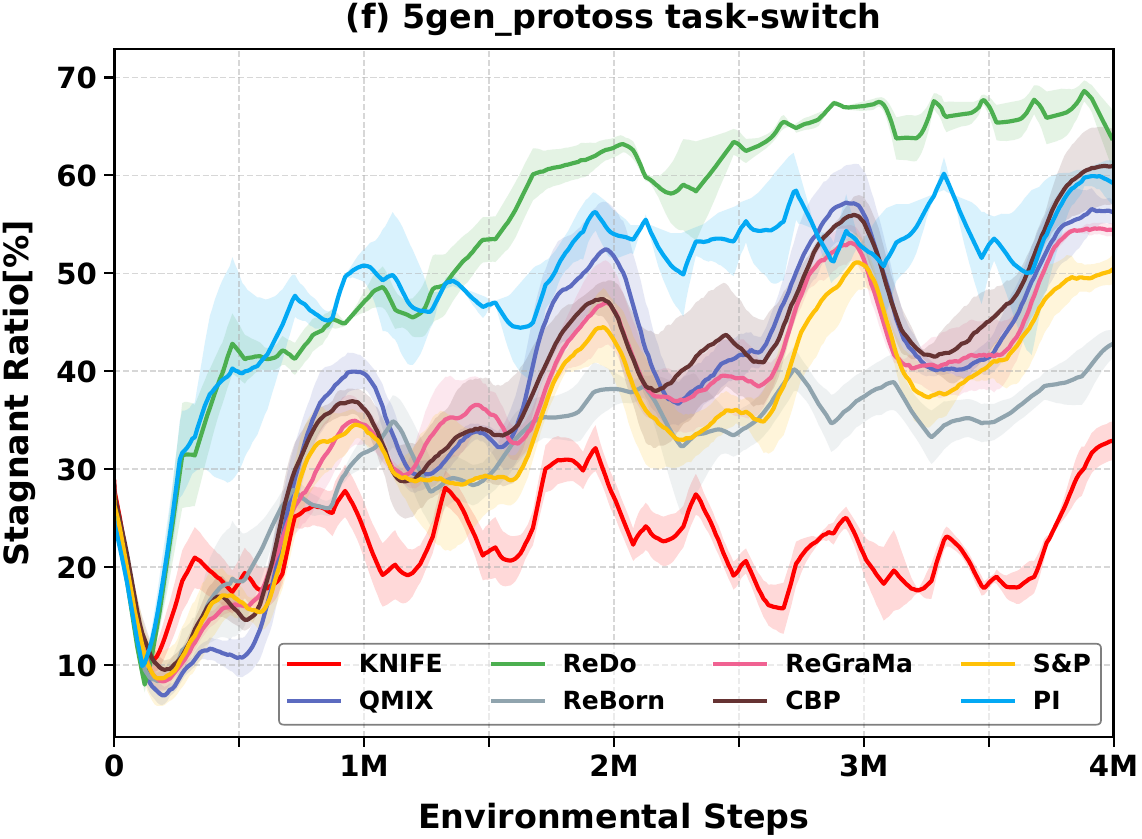} 	
	\end{minipage}
        \begin{minipage}[b]{\linewidth} % 第一行画图结果
		\centering
            \includegraphics[width=0.32\columnwidth]{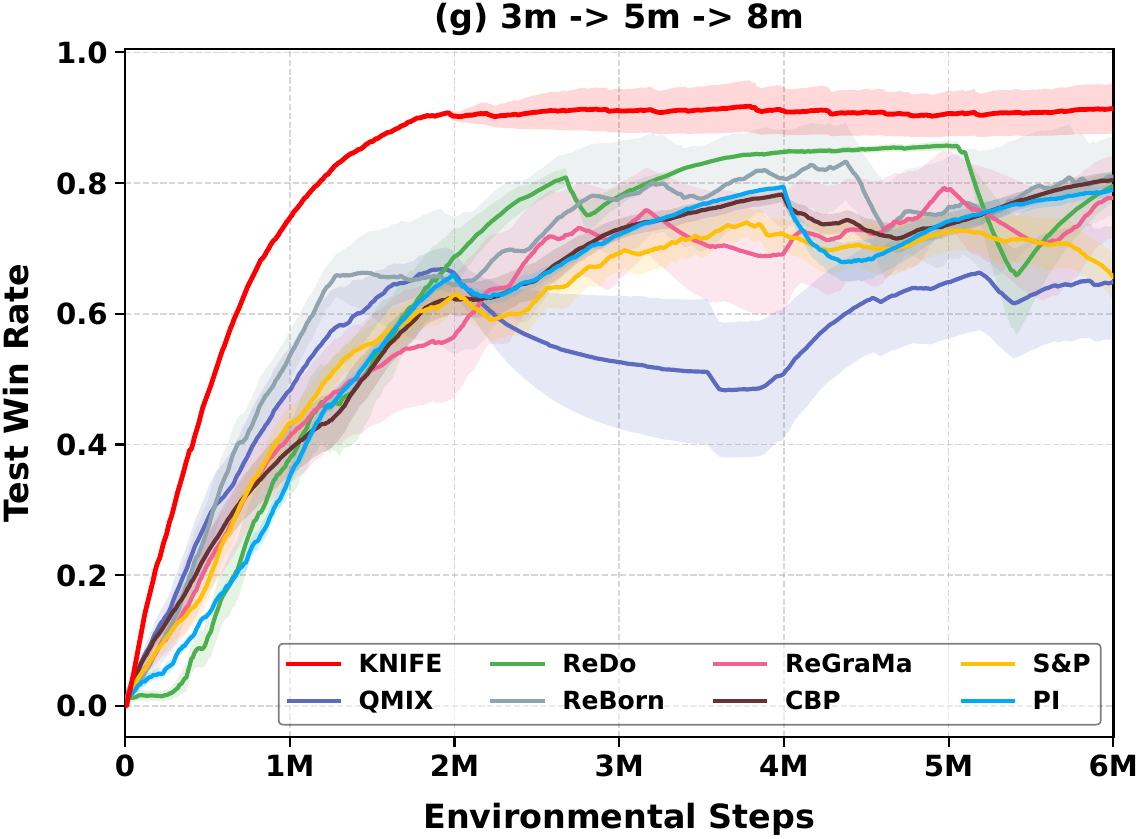} 		
		\includegraphics[width=0.32\columnwidth]{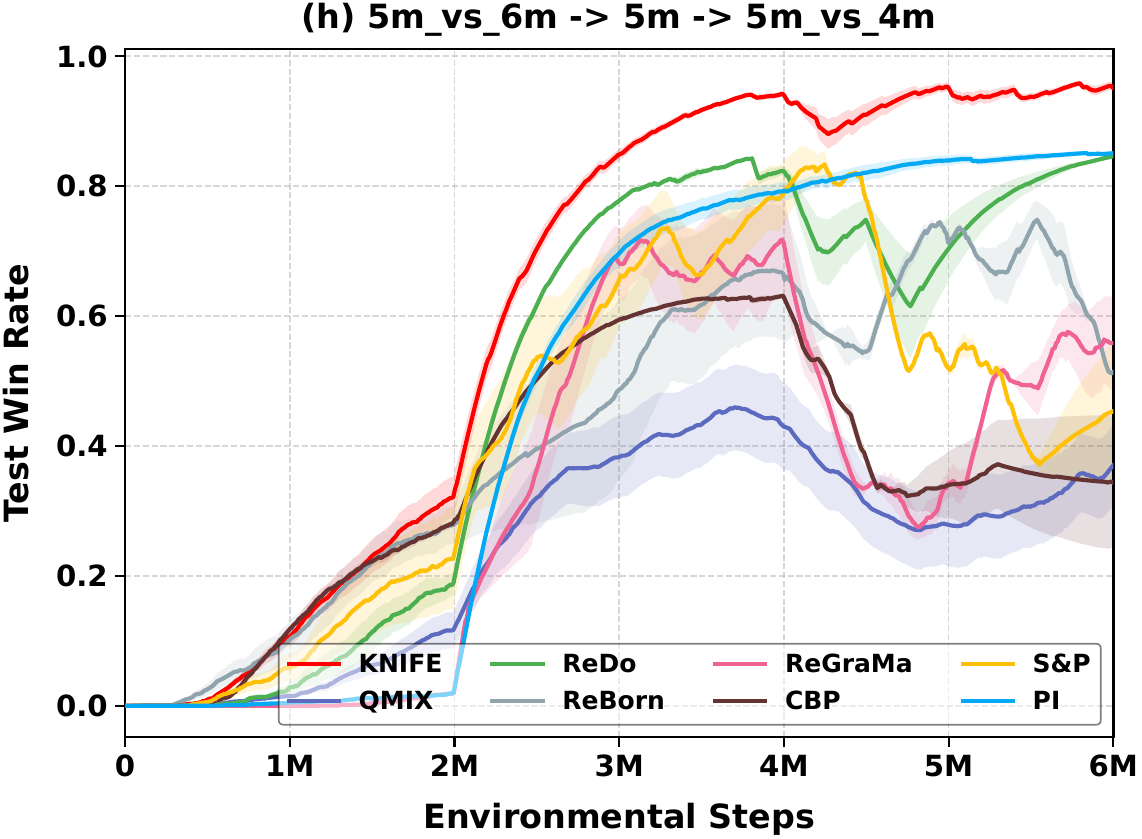} 	
            \includegraphics[width=0.32\columnwidth]{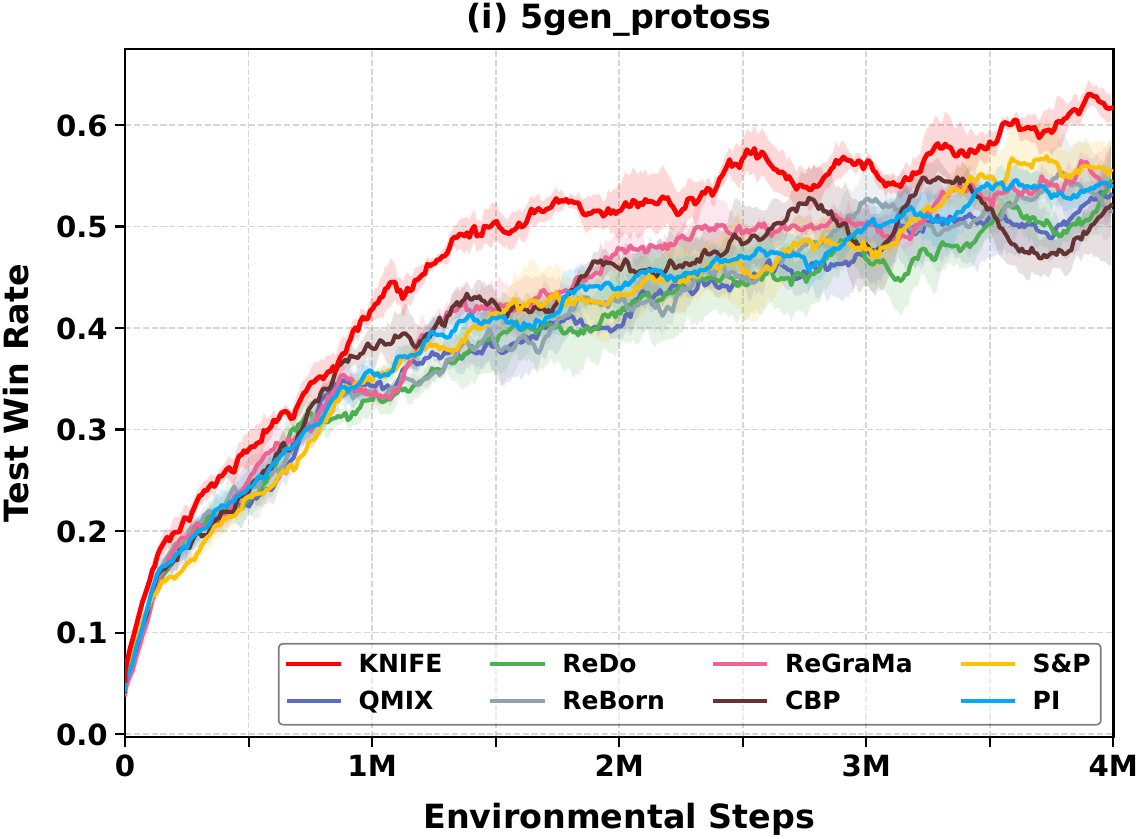} 		
	\end{minipage}
%        \caption{KNIFE performs better than other plasticity injection methods on SMACv2 (task-switch)}
        \caption{Top Row: The win rate for task-drift setting (SMACv2 5gen\_protoss, a), task-drift setting (SMACv2 5gen\_terran, b), and Task-switch (SMACv2 5gen\_protoss, c). Middle Row: The stagnant ratio for the environments in the top row, respectively. Bottom Row: The win rate for scaling-up setting (SMAC 3m $\rightarrow$ 5m $\rightarrow$ 8m, g), scaling-down setting (SMAC 5m\_vs\_6m $\rightarrow$ 5m $\rightarrow$ 5m\_vs\_4m, h) and non-continual setting (SMACv2, 5gen\_protoss, i)}
	\label{exp:compare:main}

%    \begin{minipage}[b]{\linewidth} % 第一行画图结果
%		\centering
%\includegraphics[width=0.32\columnwidth]{ICML26/experiment/smacv2/v5/[exp]stagnant_ratio_5gen_protoss_tasks_v5_500k.pdf} 	           %\includegraphics[width=0.32\columnwidth]{ICML26/experiment/smacv2/v5/[exp]stagnant_ratio_5gen_terran_tasks_v5_500k.pdf} 	%\includegraphics[width=0.32\columnwidth]{ICML26/experiment/smacv2/v5/5gen_zerg_tasks_v5_500k_Ratio.pdf} 	
%	\end{minipage}

%###紫色变淡一点，或者换成PI，迁移放到下面3张
  %  \begin{minipage}[b]{\linewidth} % 第一行画图结果
%		\centering
           
            % \includegraphics[width=0.32\columnwidth]{ICML26/experiment/smacv2/v5/[exp]win_rate_5gen_terran_tasks_v5_500k.pdf} 		
            % \includegraphics[width=0.32\columnwidth]{ICML26/experiment/smacv2/5gen_protoss_test_battle_won_mean.pdf} 		
		%\includegraphics[width=0.32\columnwidth]{ICML26/experiment/smacv2/v5/5gen_zerg_tasks_v5_500k_test_battle_won_mean.pdf} 	
%	\end{minipage}
      %  \caption{KNIFE performs better than other plasticity injection methods on (a) 5gen\_protoss of SMACv2 (task-drift), (b) 5gen\_terran of SMACv2 (task-drift), (c) 5gen\_protoss of original SMACv2.}
\end{figure*}

%\begin{figure*}[htbp]	
%	\centering 
%	\begin{minipage}[b]{\linewidth} % 第一行画图结果
%		\centering
%            \includegraphics[width=0.32\columnwidth]{ICML26/experiment/smacv2/v5/[exp]win_rate_5gen_protoss_tasks_v5_500k.pdf} 	
%            \includegraphics[width=0.32\columnwidth]{ICML26/experiment/smacv2/v5/[exp]win_rate_5gen_terran_tasks_v5_500k.pdf} 		
		%\includegraphics[width=0.32\columnwidth]{ICML26/experiment/smacv2/v5/5gen_zerg_tasks_v5_500k_test_battle_won_mean.pdf} 	
%	\end{minipage}
%    \begin{minipage}[b]{\linewidth} % 第一行画图结果
%		\centering
%\includegraphics[width=0.32\columnwidth]{ICML26/experiment/smacv2/v5/[exp]stagnant_ratio_5gen_protoss_tasks_v5_500k.pdf} 	           %\includegraphics[width=0.32\columnwidth]{ICML26/experiment/smacv2/v5/[exp]stagnant_ratio_5gen_terran_tasks_v5_500k.pdf} 	%\includegraphics[width=0.32\columnwidth]{ICML26/experiment/smacv2/v5/5gen_zerg_tasks_v5_500k_Ratio.pdf} 	
%	\end{minipage}
%        \caption{KNIFE performs better than other plasticity injection methods on SMACv2 ()}
%	\label{exp:compare:smacv2v5}
%\end{figure*}

% We show that KNIFE performs better than the other neuron plasticity methods for the one-step matrix game, MPE, SMAC, and SMACv2 benchmarks. KNIFE can improve the performance of MARL algorithms by reducing the percentage of stagnant neurons in networks. Please refer to Appendix~\ref{appendix:exp} for a detailed setup and more experimental results.

In this section, we evaluate KNIFE on representative MARL algorithms. We show that KNIFE outperforms other neuron plasticity methods on the one-step matrix game, predator-prey, SMAC, and SMACv2 benchmarks under the continue-learning setting, thanks to its ability to reduce the stagnant neuron ratio. KNIFE works in multiple MARL methods. We demonstrate that the stagnant neuron can be a better indicator for plasticity loss for MARL. The experimental results justify the design choice that satisfies the KI++ principle. Please refer to the Appendix for details.

\subsection{Environment Setup } %(SMACv2 with Controlled Non-stationarity)
\paragraph{Benchmarks} We modify existing MARL benchmarks (SMAC, SMACv2, Predator Prey) for four continual learning settings: task-drift, task-switch, scaling up, and scaling down. Moreover, we also evaluate non-continual learning settings. We describe the settings as follows.

(1) \textbf{The task-drift setting} comprises 11 SMACv2 tasks. Each team has 5 agents consisting of two different types whose probability changes every task. (2) \textbf{The task-switch setting} switches 2 SMACv2 tasks periodically.  For example, in the 5gen\_protoss scenario, the tasks switch between 5 Stalker vs 5 Stalker and 5 Zealot vs 5 Zealot. Predator Prey is evaluated under similar settings. (3) \textbf{The scaling up setting} comprises 3 SMAC tasks where the number of agents and the opponents increase from 3 to 5 and 8. (4) \textbf{The scaling down setting} evaluates 3 SMAC tasks where 5 agents fight against 6, 5, and 4 enemies, respectively. (5) \textbf{Non-continue setting} evaluate the original SMAC and SMACv2.

\paragraph{Baselines.} We compare KNIFE with 6 plasticity injection methods including ReDo~\cite{redo}, ReBorn~\cite{reborn}, ReGraMa~\cite{regrama}, S\&P~\cite{s&p}, CBP~\cite{nature24plasticity}, ReSet~\cite{ReSet2}, and Plasticity Injection (PI)~\cite{plasticity2} for different MARL. %ReDo re-initializes the parameters of dormant neurons, whereas ReBorn re-initializes the parameters of dormant and overactive neurons in MARL. ReGraMA re-initializes the parameters of neurons whose gradients are low. S\&P performs shrinkage and perturbation on neurons. CBP selectively re-initializes low-utility neurons. 
%ReSet re-initialize the parameters of the last layer of a neural network.
%To ensure a fair comparison, all methods use identical network architectures, optimizer settings, detailed implementations, and parameter configurations are available in Appendix ~\ref{appendix:exp:setup}.

%To ensure a fair comparison, all methods use identical network architectures, optimizer settings, and total training steps; they differ only in the neuron intervention mechanism. ~\ref{appendix:exp:setup}.

%We compare GradPS with various parameter sharing methods including (1) Full Parameter Sharing (FuPS),  (2) Full Parameter Sharing with index (FuPS+id), (3) Selective Parameter Sharing (SePS)~\cite{SePS}, (4) Structured Network Pruning with parameter Sharing (S\&P)~\cite{S\&P}, (5) Kaleidoscope~\cite{kalei}, (6) No Parameter Sharing (NoPS). We evaluate the performance of GradPS and all these methods with the QMIX~\cite{QMIX} value factorization function. Detailed implementations and parameter configurations are available in Appendix ~\ref{appendix:exp:setup}.

\subsection{KNIFE Performs Better than other Plasticity Injection Methods for MARL Tasks}\label{exp:compare}

The performance of QMIX with different plasticity injection methods for the task-drift setting and the task-switch setting is depicted in the top row of Figure~\ref{exp:compare:main}. The stagnant neuron ratios are plotted in the middle row of Figure~\ref{exp:compare:main}. KNIFE achieves the best performance in terms of win rate by reducing the number of stagnant neurons, which hinders learning ability. ReDo cannot deal with the stagnant neurons well, nor can ReGraMa. Although CBP is developed for the continual RL setting, it does not work for the MARL setting. ReBorn is developed for the dormant neurons in MARL. Although it satisfies the KI principle, it underperforms compared to KNIFE. 

%Regarding the stagnant neuron ratio, only ReSet outperforms KNIFE on this metric for a single environment, albeit at the cost of compromised overall performance.
%We explore the performance of KNIFE by applying different plasticity injection methods to QMIX across various experimental scenarios.

%In Figure~\ref{exp:compare:smacv2v6}

The win rates for the scaling-up and the scaling down settings are plotted in bottom row of Figure~\ref{exp:compare:main} (g and h). The results confirm the superiority of KNIFE as a plasticity injection method. Moreover, we also test the performance of KNIFE in the standard SMACv2 (5gen\_protoss). As it is depicted in bottom row of Figure~\ref{exp:compare:main} (i), KNIFE works better than others in original SMACv2, where stagnant neurons are present. For the continual predator-prey and the continual matrix game, their results are plotted in Appendix Figure ~\ref{app:bench1} and Figure~\ref{app:fig:matrix_periodic}. As shown in the figures, KNIFE performs better than others in the predator-prey and the matrix game.

%Although ReSet is very effective at reducing stagnant neurons, it comes at the cost of poor win rate due to loss of previous knowledge. 

%Although ReSet can reduce the stagnant ratio in this experiment, such large-scale resetting improves plasticity at the cost of forgetting previously learned knowledge, leading to slow convergence and poor overall performance.

% The experimental results for various PS methods and GradPS are shown in \cref{exp:compare:smac}. For the 2c\_vs\_64zg, the 27m\_vs\_30m, the 3s5z\_vs\_3s6z environments, as is shown in \cref{exp:compare:smac} (a-c), GradPS achieves the best performance in terms of win rate thanks to its ability to reduce the futile neurons better than other PS methods, as is shown in \cref{exp:compare:smac} (d-f). For the 5m\_vs\_6m, MMM2, 1c3s8z\_vs\_1c3s9z environments, GradPS performs the best among all the methods as well. \cref{exp:compare:stag_hunt} depicts the results for the Predator-Prey environments. As shown in the graph, GradPS performs better than other PS methods for the Predator-Prey environments. 

% In addition, we implemented GradPS based on PyMARL3 and extended the environmental steps, as shown in the Figure~\ref{fig:compare:extend}. More results are presented in Appendix~\ref{exp:compare:appendix:pymarl} and~\ref{exp:compare:appendix:pymarl3}.

\subsection{Neuron Analysis and Ablation Study}\label{exp:ablation}

%\paragraph{Injection scope.}
%We next vary the intervention scope: \emph{agent-only}, \emph{mixer-only}, and \emph{all} (agent + mixer). The results show that intervening on the mixing network is the dominant contributor to the gains: \emph{mixer-only} achieves most of the performance improvement and substantially reduces the stagnant ratio. \emph{agent-only} yields weaker and less consistent gains, which aligns with our earlier observation that stagnation is markedly concentrated in the mixer (Figure~\ref{stagnant:analysis1}). Applying KNIFE to \emph{all} networks provides at most marginal additional benefit over \emph{mixer-only}, suggesting that preserving plasticity in the mixer is the key bottleneck in these value-decomposition learners.

%To study whether the stagnant neuron is a useful metric for plasticity loss. 
We replace the stagnant neuron used in KNIFE with the dormant neurons, and replace the dormant neuron used in ReDo and ReBorn with the stagnant neurons. And then compare them with their original version. The results depicted in Figure~\ref{exp:compare:ablation} (a) show the performance advantage of using the stagnant neurons over the dormant neurons. The non-zero advantage demonstrates that stagnant neurons are a better indicator of plasticity loss than the dormant neurons. 

We study the impact of injecting plasticity into different sets of neurons: \emph{stagnant-only}, \emph{volatile-only}, and \emph{both} neurons. The win rate for the task-drift setting (SMACv2 5gen\_protoss) is plotted in Figure~\ref{exp:compare:ablation} (b). Repairing \emph{both} yields the best overall performance and the lowest stagnant ratio, while \emph{stagnant-only} performs the second. Activating the stagnant neurons is the main source of improvement, and repairing volatile neurons helps plasticity injection.
%helps re-balance updates under target changes, further mitigating stagnation.

We study the impact of the knowledge-retentive injection operation of KNIFE through (a) w/o knowledge neuron, (b) w/o compensation neuron, and (c) w/o both. Removing either the knowledge or the compensation neuron can lead to violation of the knowledge-invariant++ (KI++) principle (see ~\cref{def:ki}). Figure~\ref{exp:compare:ablation} (c) shows that violating KI++ principle could lead to a performance drop.

\begin{figure*}[!t]
	\begin{minipage}[!t]{\linewidth} % 第一行画图结果
		\centering
              \includegraphics[width=0.32\columnwidth]{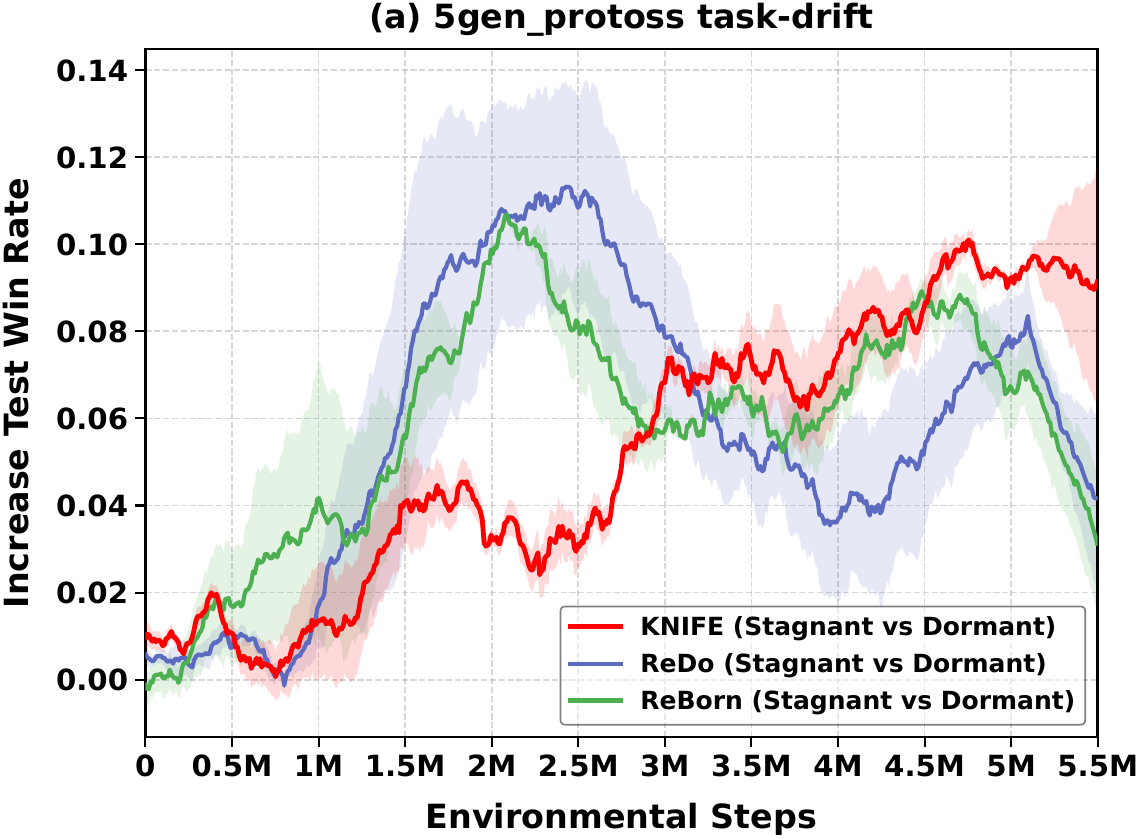} 		
		\includegraphics[width=0.32\columnwidth]{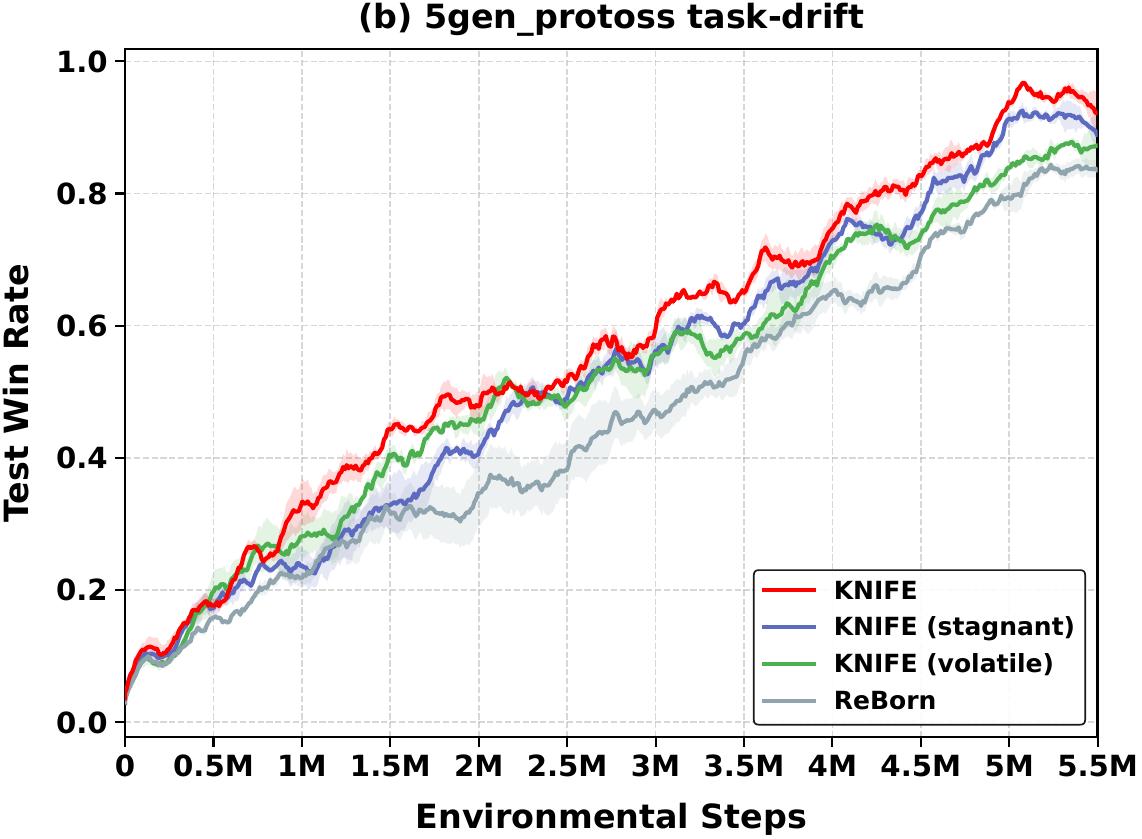} 
         \includegraphics[width=0.32\columnwidth]{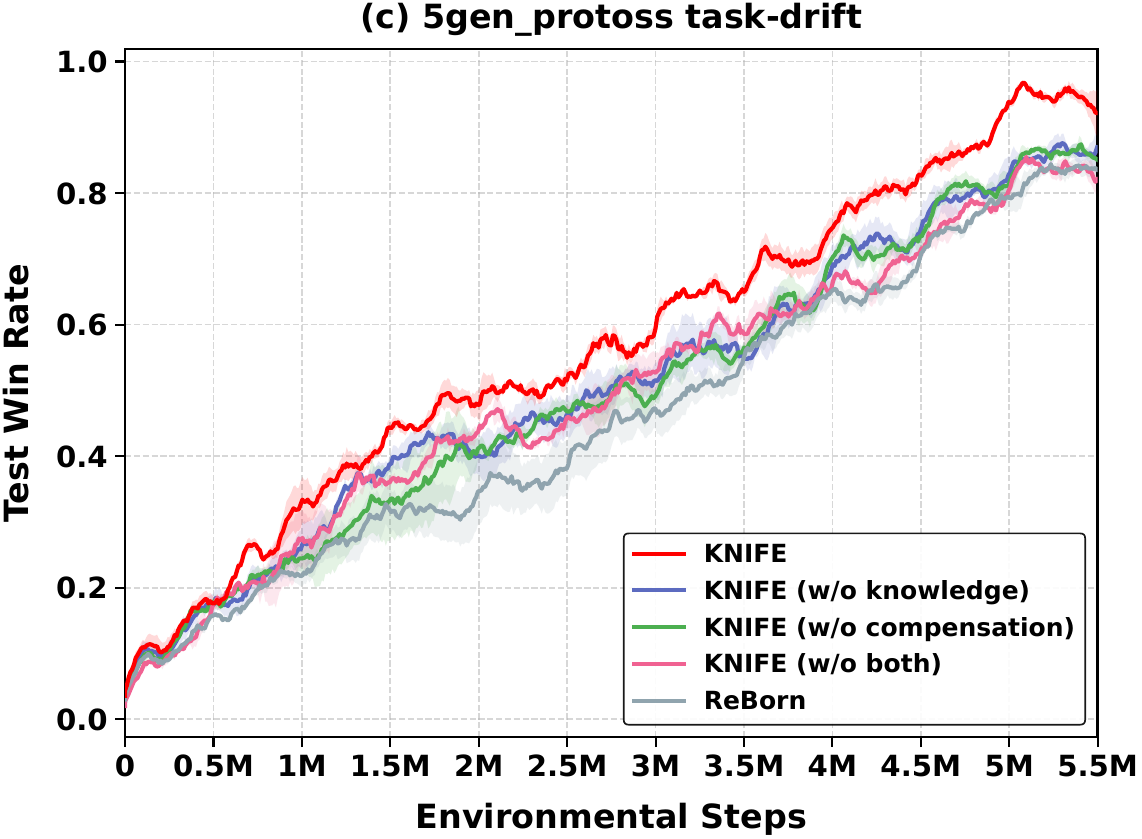} 		
	\end{minipage}
            \caption{Neuron Analysis and Ablation Study. (a) Neuron-level methods work better based on stagnant neurons, (b) KNIFE on stagnant/volatile neurons (c) Impact of the knowledge-retentive injection operations}   \label{exp:compare:ablation}
\end{figure*}   

%###做差，看性能提升，同种方法用相同颜色

\begin{figure*}[!t]
	\begin{minipage}[!t]{\linewidth} % 第一行画图结果
		\centering
    \includegraphics[width=0.32\columnwidth]{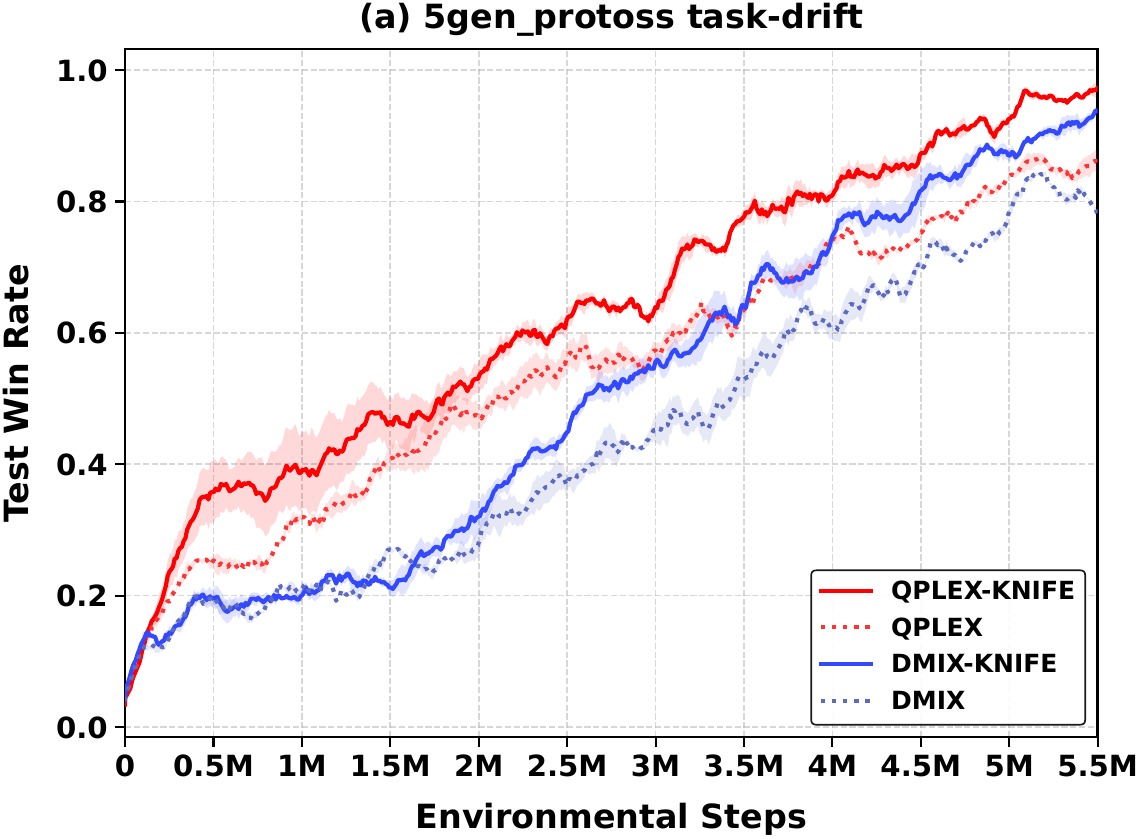}  
        \includegraphics[width=0.32\columnwidth]{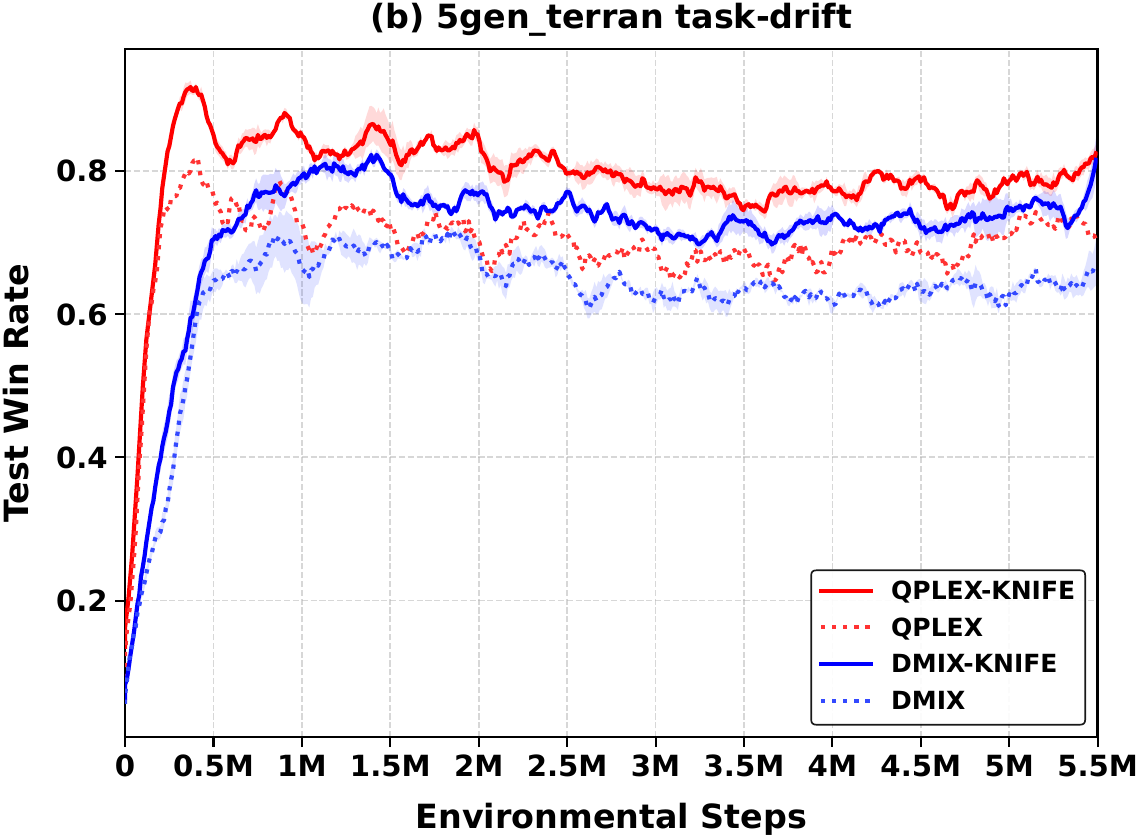}
         \includegraphics[width=0.32\columnwidth]{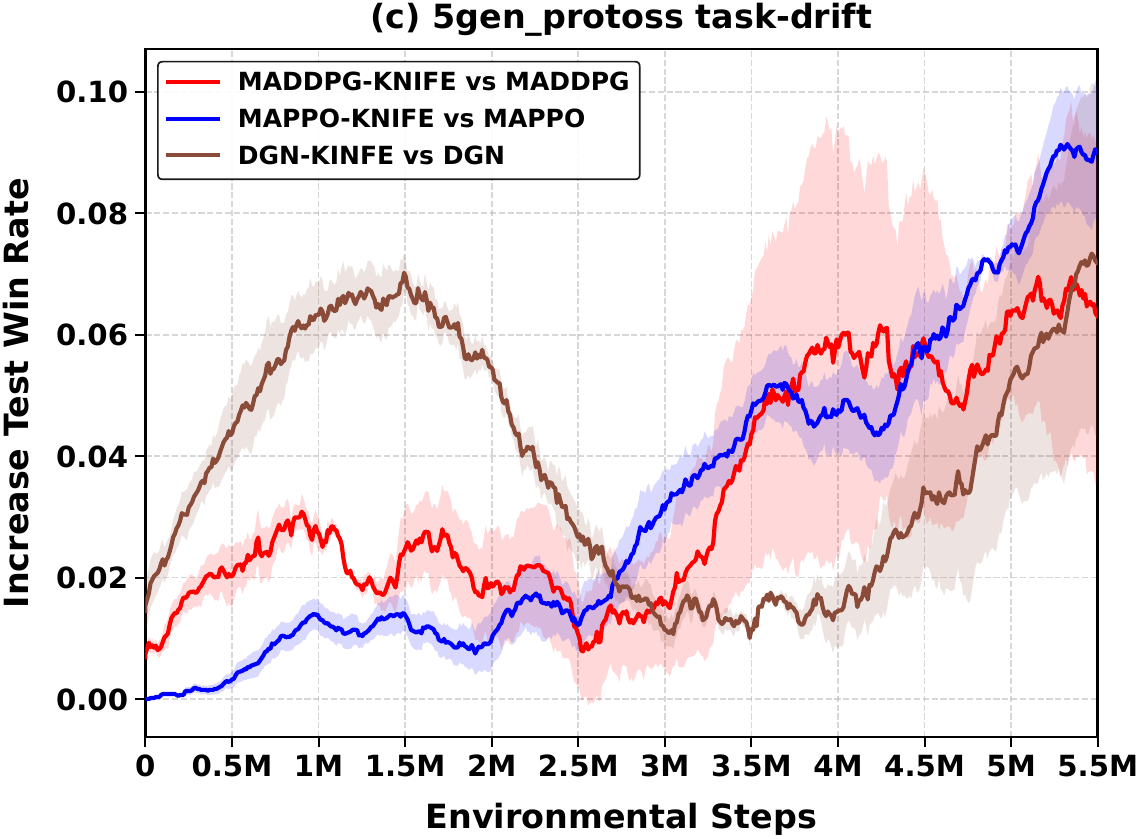} 
	\end{minipage}
            % \caption{(a) The bias for one-step payoff matrix, (b) return for Predator Prey, and (c) win rate for 3s\_vs\_5z of SMAC}
            \caption{KNIFE can reduce the plasticity loss for QPLEX and QMIX in (a) task-drift (SMACv2 5gen\_protoss), (b) task-drift (SMACv2 5gen\_terran), and for MADDPG, MAPPO and DGN in (c) 5gen\_protoss.} 
    \label{exp:compare:qplexdmix_main}
\end{figure*}

%s: \emph{stagnant-only}, \emph{volatile-only}, and \emph{both} neurons. The win rate and the neuron ratios for the Protoss environment of SMACv2 (task-switch) are plotted in Figure~\ref{exp:compare:ablation} (a) and (b), respectively. Repairing \emph{both} yields the best overall performance and the lowest stagnant ratio, while \emph{stagnant-only} recovers most of the gain. This indicates that activating the stagnant neurons is the main source of improvement, and additionally repairing volatile neurons help plasticity injection.
%helps re-balance updates under target changes, further mitigating stagnation.

% \begin{figure*}[!t] %移到附录
% 	\centering
%     \begin{minipage}[!t]{\linewidth} % 第一行画图结果
% 		\centering
%         \includegraphics[width=0.24\columnwidth]{ICML26/experiment/ablation/[decay]5gen_protoss_tasks_v6_500k_test_battle_won_mean.pdf} 	
%             \includegraphics[width=0.24\columnwidth]{ICML26/experiment/ablation/[T]5gen_protoss_tasks_v6_500k_test_battle_won_mean.pdf} 	
%             \includegraphics[width=0.24\columnwidth]{ICML26/experiment/ablation/[alpha]5gen_protoss_tasks_v6_500k_test_battle_won_mean.pdf} 		
% 		\includegraphics[width=0.24\columnwidth]{ICML26/experiment/ablation/[beta]5gen_protoss_tasks_v6_500k_test_battle_won_mean.pdf} 	
% 	\end{minipage}
%         \caption{Parameter Sensitivity: (a) Weight decaying factor, (b) Execution interval, (c) Stagnant threshold, (d) Volatile threshold}
% 	\label{exp:compare:parameters} 
% \end{figure*}

\subsection{KNIFE works on multiple MARL methods} \label{exp:apply}

%We apply KNIFE to QPLEX, and DMIX.
KNIFE can be used in MARL value factorization, such as QPLEX and DMIX. The experimental results are depicted in Figure~\ref{exp:compare:qplexdmix_main} (a) and (b). In these graphs, the performance curves of QPLEX and DMIX with KNIFE are depicted as QPLEX-KNIFE and DMIX-KNIFE, respectively. The curves of QPLEX-KNIFE and DMIX-KNIFE are higher than their counterparts, which demonstrates the applicability of KNIFE. Beyond MARL value factorization, we demonstrate KNIFE can lead to performance improvement for MAPPO, MADDPG, and DGN in Figure~\ref{exp:compare:qplexdmix_main} (c) and  Appendix \ref{sec:appendix_generality}.

\section{Conclusion}
In this work, we study plasticity loss in MARL value factorization under continual learning scenarios and identify stagnant neurons (large weights with negligible relative updates) as a key factor that accumulates in the MARL neural networks and slows learning. We propose KNIFE, a Knowledge-retentive Neuron-level PlastIcity Focusing InjEction method. It periodically detects stagnant neurons and applies a knowledge-invariant injection surgery with frozen knowledge/compensation neurons. Experiments on SMAC, SMACv2, one-step matrix games, and predator-prey show that KNIFE improves continual adaptation and reduces the stagnant-neuron ratio.

\newpage
\bibliography{QGraph}
\bibliographystyle{unsrtnat}

%%%%%%%%%%%%%%%%%%%%%%%%%%%%%%%%%%%%%%%%%%%%%%%%%%%%%%%%%%%%

\appendix
% ========== 附录编号 + 正确跳转 核心代码 ==========
\makeatletter
\@addtoreset{figure}{section}   % 让图片按附录章节重新计数
% \@addtoreset{definition}{section}
\makeatother
\newpage
\renewcommand{\thetable}{A\arabic{table}} % 设置编号格式为 A1, A2, ...
\setcounter{table}{0}  
\setcounter{figure}{0}
\setcounter{equation}{0}
\setcounter{theorem}{0}
\setcounter{definition}{0}
\renewcommand{\theequation}{\thesection.\arabic{equation}}
\renewcommand{\thefigure}{A\arabic{figure}} % 设置编号格式为 A1, A2, ...
\setcounter{figure}{0} % 重置图片计数器（如需从 A1 开始）
\renewcommand{\thedefinition}{A\arabic{definition}}
% 重置 definition 计数器，使附录中的定义从 1 开始编号
\setcounter{definition}{0}
\section{Background}
\subsection{Value Function Factorization}

In value factorization methods~\cite{QMIX,VDN,qplex,QTRAN}, per-agent utilities $Q_i$ are approximated by \emph{agent networks}, and then mixed by a \emph{mixer network} to form the joint state-action value function $Q_{\mathrm{tot}}$. 

%For value factorization, the Individual-Global-Max (IGM) principle~\cite{QTRAN} is a critical criterion that ensures the consistency between local and joint optimal action selections. It is defined as follows.

\begin{definition}[Individual-Global-Max (IGM)~\cite{QTRAN}]
\label{def:igm}
For a joint state-action value function $Q_{\mathrm{jt}}:\mathcal{T}^N\times A^N\mapsto\mathbb{R}$, where $\tau\in\mathcal{T}^N$ is a joint action-observation history and $\mathbf{a} \in A^N$ is the joint action, if there exist individual state-action functions $\{Q_i:\mathcal{T}_i\times \ A_i\mapsto\mathbb{R}\}_{i=1}^N$ such that
\begin{equation}
\label{eq:igm}
\begin{aligned}
\arg\max_{\mathbf{a}}\, Q_{\mathrm{jt}}(\tau,\mathbf{a})
&=
\Bigl(\arg\max_{a_1} Q_1(\tau_1,a_1),\,\ldots,\\
&\quad \arg\max_{a_N} Q_N(\tau_N,a_N)\Bigr).
\end{aligned}
\end{equation}
then we say that $\{Q_i\}_{i=1}^N$ satisfies \emph{IGM} for $Q_{\mathrm{jt}}$ under $\tau$, or equivalently, $Q_{\mathrm{jt}}(\tau,\mathbf{a})$ is factorized by $\{Q_i(\tau_i,a_i)\}_{i=1}^N$.
\end{definition}

\subsection{Knowledge Invariant Principle}

\begin{definition}[Knowledge Invariant Principle (KI)~\cite{reborn}]
\label{def:ki}
A joint state-action value function is represented as 
$Q_{tot}^{\theta,\phi}(\tau,\mathbf{a}) = f_{\theta}(Q_1^{\phi}(\tau_1,a_1),...,Q_N^{\phi}(\tau_N,a_N))$, where $f_{\theta}$ is the mixing function that mixes $Q_i$ into $Q_{tot}$, $\tau$ is joint observation-action history, $\mathbf{a} = [a_1,...a_N]$ is the joint action of multi-agent, $g: \mathbb{R} \mapsto \mathbb{R}$ is a function that maps weights in $\theta$ to $\hat{\theta}$, $g(\theta) = \hat{\theta}$. $h: \mathbb{R} \mapsto \mathbb{R}$, $h(\phi) = \tilde{\phi}$, $h$ map the weights in $\phi$ to $\tilde{\phi}$. If the following condition holds:
\begin{equation}
\label{eq:ki}
\begin{aligned}
Q_{\mathrm{tot}}^{\theta,\phi}(\tau,\mathbf{a})
&\ge
Q_{\mathrm{tot}}^{\theta,\phi}(\tau,\mathbf{a}')
\Rightarrow
Q_{\mathrm{tot}}^{\hat{\theta},\tilde{\phi}}(\tau,\mathbf{a})
\ge
Q_{\mathrm{tot}}^{\hat{\theta},\tilde{\phi}}(\tau,\mathbf{a}'),\\
&\exists!~k:\ a_k\neq a_k'.
\end{aligned}
\end{equation}
then, the two functions \( g \) and \( h \) satisfy the Knowledge Invariant Principle for \( Q_{\mathrm{tot}}^{\theta,\phi} \), where \(\left[Q_i(\tau_i, a_i)\right]_{i=1}^N\) is individual agent utility function, \( N \) is the number of agents, \( \tau_i \) and \( a_i \) are the observation-action history and action of agent \( i \), respectively. \( \exists! \) represents the concept of unique existence.
\end{definition}

The KI principle is used to assess whether a parameter perturbation operation preserves the learned cooperative knowledge. \citet{reborn} show that if a parameter perturbation violates the KI principle, then it breaks the IGM consistency and therefore cannot guarantee IGM-consistent value factorization.

\subsection{Knowledge Invariant++ (KI++) Principle}\label{app:ki++}

Inspired by~\cite{reborn}, we propose a stronger neuron-level parameter perturbation principle, Knowledge Invariant++ Principle. This principle removes the unique existence condition $\exists!~k:\ a_k\neq a_k'$ from the KI principle.

\begin{definition}[Knowledge Invariant++ Principle (KI++)]
\label{def:ki++}
A joint state-action value function is represented as 
$Q_{tot}^{\theta,\phi}(\tau,\mathbf{a}) = f_{\theta}(Q_1^{\phi}(\tau_1,a_1),...,Q_N^{\phi}(\tau_N,a_N))$, where $f_{\theta}$ is the mixing function that mixes $Q_i$ into $Q_{tot}$, $\tau$ is joint observation-action history, $\mathbf{a} = [a_1,...a_N]$ is the joint action of multi-agent, $g: \mathbb{R} \mapsto \mathbb{R}$ is a function that maps weights in $\theta$ to $\hat{\theta}$, $g(\theta) = \hat{\theta}$. $h: \mathbb{R} \mapsto \mathbb{R}$, $h(\phi) = \tilde{\phi}$, $h$ map the weights in $\phi$ to $\tilde{\phi}$. If the following condition holds:
\begin{equation}
\label{eq:ki}
\begin{aligned}
Q_{\mathrm{tot}}^{\theta,\phi}(\tau,\mathbf{a})
&\ge
Q_{\mathrm{tot}}^{\theta,\phi}(\tau,\mathbf{a}')
\Rightarrow
Q_{\mathrm{tot}}^{\hat{\theta},\tilde{\phi}}(\tau,\mathbf{a})
\ge
Q_{\mathrm{tot}}^{\hat{\theta},\tilde{\phi}}(\tau,\mathbf{a}'),
\end{aligned}
\end{equation}
then, the two functions \( g \) and \( h \) satisfy the Knowledge Invariant++ Principle for \( Q_{\mathrm{tot}}^{\theta,\phi} \), where \(\left[Q_i(\tau_i, a_i)\right]_{i=1}^N\) is individual agent utility function, \( N \) is the number of agents, \( \tau_i \) and \( a_i \) are the observation-action history and action of agent \( i \), respectively.
\end{definition}

The unique existence condition  $\exists!~k:\ a_k\neq a_k'$ from the KI principle states that two joint actions must differ in exactly one agent’s action. This requirement, however, is overly strict. It is specifically designed for monotonic increasing mixing functions, where the value comparison between two joint actions relies on the change in only one agent’s action. The KI++ principle is stronger than the KI principle in the sense that it is less restrictive and therefore more widely applicable.

KI++ less restrictive but implies a stronger guarantee over a larger set of action comparisons. 

%The unique existence condition $\exists!~k:\ a_k\neq a_k'$ from the KI principle requires that two joint actions differs only in one agent action $k$. However, it is too strict. Such condition is specially used for monotonic increasing mixing function, where we can compare the value of two joint action with only one agent's action is different.

%The KI principle is used to assess whether a parameter perturbation operation preserves the learned cooperative knowledge. \citet{reborn} show that if a parameter perturbation violates the KI principle, then it breaks the IGM consistency and therefore cannot guarantee IGM-consistent value factorization.

\subsection{Difference between ReBorn and KNIFE}

ReBorn is a neuron-level method designed to address dormant neurons in MARL and satisfies the KI principle for algorithms such as QMIX, QPLEX, and RMIX. In contrast, KNIFE targets stagnant neurons under periodic task switching and satisfies the stricter KI++ principle. Moreover, ReBorn does not satisfy the KI++ principle. The distinction between these two types of neurons is illustrated in Figure~\ref{exp:compare:main}. Empirically, KNIFE outperforms ReBorn across multiple environments, including SMACv2, Predator-prey, and the one-step matrix game.

\subsection{Detailed Comparison with Plasticity Injection (PI)}
\label{sec:appendix_pi}

Plasticity Injection (PI) and KNIFE both introduce new neurons and preserve learned knowledge via frozen weights. However, PI operates at the layer level, targeting the last $K$ layers of the network, whereas KNIFE acts at the neuron level, focusing specifically on stagnant and volatile neurons in the mixing network. They differ in when, where, and how plasticity is injected, as detailed in Table \ref{tab:pi_diff}.

\begin{table}[htbp]
\centering
\caption{Conceptual and Methodological Differences}
\label{tab:pi_diff}
\renewcommand{\arraystretch}{1.3}

\resizebox{\linewidth}{!}{

\begin{tabular}{p{0.18\textwidth} p{0.25\textwidth} p{0.25\textwidth} p{0.22\textwidth}}
\toprule
\textbf{Feature} & \textbf{Plasticity Injection (PI)} & \textbf{KNIFE (Ours)} & \textbf{Our Advantage} \\
\midrule
\textbf{Granularity} & Layer-level & Neuron-level & Finer-grained control. \\
\textbf{Trigger Time} & Once & Periodically & Can handle different environments without human intervention. \\
\textbf{Treatment Location} & All neurons of the last $K$ layers & Stagnant/volatile neurons in the mixing network & Avoids blindly duplicating multiple neural network layers. \\
\textbf{Increased Neurons} & $K \times N \times 2$ where $N$ is the average number of neurons in the last $K$ layers & $n \times 2$, where $n$ is the number of stagnant/volatile neurons & KNIFE introduces less neurons than PI. \\
\textbf{Recycling Mechanism} & No & Frozen neurons are recycled through pruning. & Memory-friendly. \\
\textbf{Initialization} & Kaiming Initialization (State Agnostic) & State-Aware & Actively prevents dead zones. \\
\bottomrule
\end{tabular}

}
\end{table}

\subsubsection{Detailed Comparison of Initialization Strategies}

While both Plasticity Injection (PI) and KNIFE introduce new parameters to restore plasticity, they fundamentally differ in their initialization strategies, which significantly impacts the effectiveness of the injected plasticity.

\textbf{Plasticity Injection (PI) Initialization:} PI typically employs standard Kaiming initialization for both the weights and biases of the injected units. This approach is purely \textit{state-agnostic}, meaning it completely ignores the current optimization state and weight distribution of the surrounding network. Consequently, there is a substantial risk that the newly injected neurons might immediately fall into activation dead zones (e.g., dying ReLUs) or fail to propagate meaningful gradient signals, rendering the plasticity injection ineffective.

\textbf{KNIFE State-Aware Initialization:} To ensure the newly injected active neurons immediately contribute to the network and avoid dead zones, KNIFE applies a \textit{state-aware} initialization strategy tailored to the current network conditions:
\begin{itemize}
    \item \textbf{Positive Input Bias:} Unlike PI, which typically initializes biases using a zero-centered distribution, KNIFE explicitly initializes the input bias of the active neuron to a positive value (e.g., proportional to the standard deviation of its input weights). This deliberate positive shift significantly increases the likelihood of positive-valued activations, ensuring the unit remains active and capable of receiving backpropagated gradients from the earliest stages.
    \item \textbf{Contextual Output Weights:} The active neuron's output weights are initialized using a zero-centered Gaussian distribution whose variance is dynamically tied to the standard deviation of the subsequent layer's weights (safeguarded by a minimum threshold). This ensures the magnitude of the forward signal is fully compatible with the existing network structure.
    \item \textbf{Input Weights:} Standard Kaiming initialization is retained solely for the input weights to maintain proper variance scaling.
\end{itemize}

By actively bridging the gradient flow and preventing new neurons from falling into optimization inertia, this state-aware initialization provides KNIFE with a fundamental advantage over PI's layer-level agnostic injection.

%\begin{figure}[!th]
%	\centering	
%    \begin{minipage}[b]{\linewidth} % 第一行画图结果
%		\centering
%         \includegraphics[width=0.49\columnwidth]{ICML26/experiment/stagnant/[7.1]5gen_protoss_tasks_v6_500k_Overlap_All_Combined.pdf} 
%        \includegraphics[width=0.49\columnwidth]{ICML26/experiment/stagnant/[7.2]3s_vs_5z_Overlap_All_Combined.pdf} 		

		% \includegraphics[width=0.32\columnwidth]{ICML26/experiment/smacv2/v6/[3.1]5gen_protoss_tasks_v6_500k_Ratio_Combined.pdf}
%	\end{minipage}
%        \caption{Overlap coefficient stagnant/volatile Neurons between the current iteration and the previous iteration: (a) SMACv2 5 gen\_protoss and (b) SMAC 3s\_vs\_5z}
%	\label{stagnant:overlap}
%\end{figure}

 %Unless stated otherwise, we report stagnation statistics on the mixing network, where stagnation is most pronounced.

\section{Principle and Theorem}
% \begin{theorem}
% \label{thm:pis_ki}
% After the plasticity injection stage, the learned value function of the QMIX~\cite{QMIX} value factorization method still satisfies the KI principle.
% \end{theorem}

% \begin{equation}
% Q_{tot}^{\theta,\phi}(\tau, \boldsymbol{u}) = f_\theta(Q_1^\phi(\tau_1, u_1), ..., Q_N^\phi(\tau_N, u_N)) + V_\theta(\tau) \quad \frac{\partial f}{\partial Q_i^\phi} \geq 0 
% \end{equation}

\begin{theorem}
\label{thm:pis_ki}
For any MARL value factorization method, after the plasticity injection stage of KNIFE, the state-action value function \( Q_{tot}^{\theta,\phi}(\ve \tau, \ve a) \) satisfies the KI++ principle.
\end{theorem}

% \begin{equation}
% Q_{tot}^{\theta,\phi}(\tau, \boldsymbol{a})
% =
% f_{\theta}\!\bigl(Q_1^\phi(\tau_1, a_1), \ldots, Q_N^\phi(\tau_N, a_N)\bigr)
% +
% V_\theta(\tau).
% \end{equation}

\begin{equation}
Q_{\mathrm{tot}}^{\theta,\phi}(\tau,\boldsymbol a)
=
f_{\theta}\Bigl(
\tau, Q_1^\phi(\tau_1,a_1),\ldots,Q_N^\phi(\tau_N,a_N)
\Bigr).
\end{equation}

where \( Q_{tot}^{\theta,\phi}(\ve \tau, \ve a) \) is the joint state-action value function, and \( f_{\theta} \) is the value factorization function of any value factorization method (e.g., QMIX~\cite{QMIX}, QPLEX~\cite{qplex}, DMIX~\cite{dmix}).

In the plasticity injection stage, $g$ maps the parameters $\theta$ of the mixing network parameters to $\hat{\theta}$.
Consider any neuron $i$ in layer $l$ of the injected mixing network. By construction, each pre-injection neuron is expanded into three injected neurons (knowledge / active / compensation), and satisfies:
\begin{equation}
x_i^{l,\mathrm{know}} + x_i^{l,\mathrm{active}} + x_i^{l,\mathrm{comp}}
=
x_i^{l,\mathrm{know}}
=
x_i^{l}.
\end{equation}
where $x_i^{l}$ denotes the pre-injection activation of neuron $i$ at layer $l$.

Therefore, for any layer $l$, the layer activation vector is unchanged:
\begin{equation}
x_{\mathrm{after}}^l = x_{\mathrm{before}}^l.
\end{equation}
Consequently, for the whole network, for any inputs $\tau$ and $\boldsymbol{q}$,
\begin{equation}
\label{eq:zero-per}
f_{\hat\theta}(\tau, \boldsymbol{q}) = f_{{\theta}}(\tau, \boldsymbol{q}).
\end{equation}

\begin{proof}
\begin{equation}
Q_{tot}^{\theta, \phi}(\tau, \boldsymbol{a})
\geq
Q_{tot}^{\theta, \phi}(\tau, \boldsymbol{a}^{\prime}),\end{equation}

\begin{equation}
Q_{tot}^{\theta, \phi}(\tau,[a_{1}, \ldots, a_{N}])
\geq
Q_{tot}^{\theta, \phi}(\tau,[a_{1}^{\prime}, \ldots, a_{N}^{\prime}]),
\quad \text{expand } \boldsymbol{a}.
\end{equation}

\begin{equation}
f_{\theta}\Bigl(\tau, Q_{1}^{\phi}(\tau_{1}, a_{1}), \ldots, Q_{N}^{\phi}(\tau_{N}, a_{N})\Bigr)
\geq
f_{\theta}\Bigl(\tau, Q_{1}^{\phi}(\tau_{1}, a_{1}^{\prime}), \ldots, Q_{N}^{\phi}(\tau_{N}, a_{N}^{\prime})\Bigr).
\end{equation}

% \begin{equation}
% \label{eq:moq1}
% f_{\theta}(Q_{1}^{\phi}(\tau_{1}, a_{1}), \ldots, Q_{N}^{\phi}(\tau_{N}, a_{N}))
% \geq
% f_{\theta}(Q_{1}^{\phi}(\tau_{1}, a_{1}^{\prime}), \ldots, Q_{N}^{\phi}(\tau_{N}, a_{N}^{\prime})).
% \end{equation}

\begin{equation}
f_{\hat{\theta}}\Bigl(\tau, Q_{1}^{\phi}(\tau_{1}, a_{1}), \ldots, Q_{N}^{\phi}(\tau_{N}, a_{N})\Bigr)
=
f_{\theta}\Bigl(\tau, Q_{1}^{\phi}(\tau_{1}, a_{1}), \ldots, Q_{N}^{\phi}(\tau_{N}, a_{N})\Bigr)
\quad \text{because of (\ref{eq:zero-per})}.
\end{equation}

\begin{equation}
f_{\hat{\theta}}\Bigl(\tau, Q_{1}^{\phi}(\tau_{1}, a_{1}^{\prime}), \ldots, Q_{N}^{\phi}(\tau_{N}, a_{N}^{\prime})\Bigr)
=
f_{\theta}\Bigl(\tau, Q_{1}^{\phi}(\tau_{1}, a_{1}^{\prime}), \ldots, Q_{N}^{\phi}(\tau_{N}, a_{N}^{\prime})\Bigr)
\quad \text{because of (\ref{eq:zero-per})}.
\end{equation}

\begin{equation}
f_{\hat{\theta}}\Bigl(\tau, Q_{1}^{\phi}(\tau_{1}, a_{1}), \ldots, Q_{N}^{\phi}(\tau_{N}, a_{N})\Bigr)
\geq
f_{\hat{\theta}}\Bigl(\tau, Q_{1}^{\phi}(\tau_{1}, a_{1}^{\prime}), \ldots, Q_{N}^{\phi}(\tau_{N}, a_{N}^{\prime})\Bigr).
\end{equation}

% \begin{equation}
% V_{\hat{\theta}}(\tau)=V_{\theta}(\tau) \quad \text{because of (\ref{eq:zero-per})}
% \end{equation}

% \begin{equation}
% f_{\hat{\theta}}(Q_{1}^{\phi}(\tau_{1}, a_{1}), \ldots, Q_{N}^{\phi}(\tau_{N}, a_{N}))+V_{\hat{\theta}}(\tau)
% \geq
% f_{\hat{\theta}}(Q_{1}^{\phi}(\tau_{1}, a_{1}^{\prime}), \ldots, Q_{N}^{\phi}(\tau_{N}, a_{N}^{\prime}))+V_{\hat{\theta}}(\tau).
% \end{equation}

\begin{equation}
Q_{tot}^{\hat{\theta}, \phi}(\tau, \boldsymbol{a})
\geq
Q_{tot}^{\hat{\theta}, \phi}(\tau, \boldsymbol{a}^{\prime}).
\end{equation}

Thus, we show that after the plasticity injection stage, the learned action preference does not change. 
\end{proof}

\begin{theorem}
\label{thm:pis_ki}
ReBorn does not satisfy the KI++ principle.
\end{theorem}

\begin{proof}
We prove this through an counter example that a value function exists, which make Reborn does not satisfy KI++. 

Let's assume a value function $Q_{tot} = k_1 \times Q_1 + k_2 \times Q_2$, which is modeled by a single neuron linear network (without output bias). Before Reborn, let's assume that $Q_{tot}^{\theta,\phi}(\ve \tau, \ve a) = 2 \times Q_1(a_1) + 2 \times Q_2(a_2)$. We ignore the observation-history for simplicity. 
Each agent has two identical actions: A and B. The individual utility function $Q_1(A) = Q_2(A)=1, Q_1(B)=Q_2(B)=2$. In this case, the optimal action is $(B, B)$.

$Q_{tot}^{\theta, \phi}(\tau, (B, B)) \geq Q_{tot}^{\hat{\theta}, \phi}(\tau, (B,A)).$

After parameter-perturbation of ReBorn, the value factorization function (not a monotonic increasing function) could become $ Q_{tot}^{\hat{\theta},\phi'}(\ve \tau, \ve a) = 2\times Q_1(a_1) - 3 \times Q_2(a_2)$. After the parameter-perturbation, the optimal action becomes $(B, A)$. The original optimal action is not longer than the optimal action. 
Thus, $Q_{tot}^{\hat{\theta}, \phi}(\tau, (B, B))
< Q_{tot}^{\hat{\theta}, \phi}(\tau, (B,A)).$

It indicates that ReBorn does not satisfy the KI++ principle.
\end{proof}

\noindent
\textbf{Stagnant neurons} in $\mathcal{S}_{stag}$ exhibit excessive optimization inertia
(e.g., $\lVert \mathbf{w}\rVert \gg 0$ while $\lVert \nabla \mathbf{w}\rVert$ remains small), representing
\emph{wasted capacity} that cannot adapt effectively.

\noindent
\textbf{Volatile neurons} in $\mathcal{S}_{vol}$ lie in the hyper-plastic tail of the layer-wise plasticity distribution: they update disproportionately relative to peers (large RUA), which empirically correlates with unstable feature drift under task switches. 

We additionally intervene on the \emph{high-RUA} tail
($\mathcal{S}_{vol}$) as a \emph{stabilizing auxiliary}: empirically, suppressing over-dominant updates reduces the
emergence rate of low-RUA (stagnant) neurons over time, leading to a lower overall stagnation ratio.
Importantly, we do not claim hyper-activity as the primary pathology; it is leveraged to indirectly mitigate inertia.

\section{Algorithm}\label{appendix:algo}
The KNIFE algorithm is described in Algorithm 1.

% \section{KNIFE Algorithm (Pseudocode)}
% \label{appendix:knife:pseudocode}

\begin{algorithm}[h]
\caption{KNIFE: Knowledge-retentive Neuron-level Plasticity Focusing Injection}
\label{alg:knife}
\begin{algorithmic}[1]
\REQUIRE network parameters $\Theta$; intervention interval $T$; thresholds $(\alpha,\beta)$; decay factor $\gamma\in(0,1)$; prune threshold $\varepsilon_{\mathrm{prune}}$ = 1e-4
\STATE Initialize frozen-neuron registry $\mathcal{F}\leftarrow\emptyset$
\FOR{each training step $t=1,2,\dots$}
    \STATE Update $\Theta$ by gradient descent \textbf{only} on trainable parameters
    \IF{$t \bmod T = 0$}
        \FOR{each layer $\ell$ to be intervened}
            \STATE Compute neuron-wise RUA scores $\{\mathrm{RUA}_i^{\ell}\}$
            \STATE $\mathcal{S}^{\ell}_{\mathrm{stag}} \leftarrow \{ i \mid \mathrm{RUA}_i^{\ell} < \alpha \}$
            \STATE $\mathcal{S}^{\ell}_{\mathrm{vol}} \leftarrow \{ i \mid \mathrm{RUA}_i^{\ell} > \beta \}$
            \STATE $\mathcal{S}^{\ell}_{\mathrm{repair}} \leftarrow \mathcal{S}^{\ell}_{\mathrm{stag}} \cup \mathcal{S}^{\ell}_{\mathrm{vol}}$
            \FOR{each neuron $i \in \mathcal{S}^{\ell}_{\mathrm{repair}}$}
                \STATE \textbf{(Injection)} Create three parallel branches: know (frozen), active (trainable), comp (frozen)
                \STATE Copy original params to knowledge branch:
                \STATE \hspace{0.7em}$\mathbf{w}_{in,i}^{\ell,\mathrm{know}}\leftarrow\mathbf{w}_{in,i}^{\ell}$, $b_i^{\ell,\mathrm{know}}\leftarrow b_i^{\ell}$, $\mathbf{w}_{out,i}^{\ell,\mathrm{know}}\leftarrow\mathbf{w}_{out,i}^{\ell}$
                \STATE Initialize active branch (e.g., Kaiming/He init):
                \STATE \hspace{0.7em}$\mathbf{w}_{in,i}^{\ell,\mathrm{active}} \sim \mathcal{N}(0, \delta_{in}^2)$, $b_i^{\ell,\mathrm{active}}\leftarrow 0.5\delta_{in}$, $\mathbf{w}_{out,i}^{\ell,\mathrm{active}}\sim\mathcal{N}(0,\delta_{out}^2)$
                \STATE Set compensation branch to cancel the active branch at injection:
                \STATE \hspace{0.7em}$\mathbf{w}_{in,i}^{\ell,\mathrm{comp}}\leftarrow\mathbf{w}_{in,i}^{\ell,\mathrm{active}}$, $b_i^{\ell,\mathrm{comp}}\leftarrow b_i^{\ell,\mathrm{active}}$, $\mathbf{w}_{out,i}^{\ell,\mathrm{comp}}\leftarrow-\mathbf{w}_{out,i}^{\ell,\mathrm{active}}$
                \STATE Freeze parameters of know/comp branches; keep active trainable
                \STATE Add know/comp to frozen registry: $\mathcal{F}\leftarrow\mathcal{F}\cup\{(\ell,i,\mathrm{know}),(\ell,i,\mathrm{comp})\}$
            \ENDFOR
        \ENDFOR

    \ENDIF
    \FOR{each frozen branch $f\in\mathcal{F}$}
        \STATE \textbf{(Pruning)} Decay frozen parameters: $\theta_f \leftarrow \gamma\,\theta_f$
        \IF{$\lVert\theta_f\rVert_2 < \varepsilon_{\mathrm{prune}}$}
            \STATE Remove/prune branch $f$ from the network and from $\mathcal{F}$
        \ENDIF
    \ENDFOR
\ENDFOR
\end{algorithmic}
\end{algorithm}

\section{Experimental Details}\label{appendix:exp}

\subsection{Experimental Setup}
\label{appendix:exp:details}

We implement these algorithms based on their open-source repositories to carry out performance analyses with hyperparameters consistent with those in PyMARL. Our methods are implemented within the PyMARL framework, and each is
evaluated using five random seeds with 95\% confidence intervals. Specific hyperparameters of different environments are
listed in Table 2. We conduct experiments on a cluster equipped with multiple NVIDIA GeForce RTX 4090 GPUs.

\begin{table}[t]
\centering
\small
\caption{Common hyperparameters used in our experiments.}
\label{tab:hyperparams_env}
\renewcommand{\arraystretch}{1.1}

\resizebox{\linewidth}{!}{

\begin{tabular}{lcccc}
\toprule
\textbf{Hyperparameter}& \textbf{SMACv2} & \textbf{SMACv2} & \textbf{One-step Matrix Game} & \textbf{Predator-Prey} \\
\midrule
Action Selector & epsilon greedy & epsilon greedy & epsilon greedy & epsilon greedy \\
Batch Size & 32 & 32& 32 & 32 \\
Buffer Size & 5000 & 5000 & 5000 & 5000 \\
Learning Rate & 0.0005 & 0.0005 & 0.0005 & 0.0005 \\
Hypernet Embed Dimension & 64 & 64 & 4 & 64 \\
Target Update Interval & 200 & 200 & 200 & 200 \\
KNIFE Interval $T$ & 0.2M & 0.2M & 10K & 0.2M \\
Stagnant Threshold $\alpha$ & 0.25 & 0.25 & 0.25 & 0.25 \\
Volatile Threshold $\beta$ & 3 & 3 & 3 & 3 \\
Decay factor $\gamma$ & 0.995 & 0.995 & 0.995 & 0.995 \\
\bottomrule
\end{tabular}

}
\end{table}

\paragraph{Benchmarks.}

We modify existing MARL benchmarks (SMAC, SMACv2, Predator Prey) for four continual learning setting: task-drift, task-switch, scaling up and scaling down. Moreover, we also evaluate non-continual learning settings. We describe the settings as follows.

(1) \textbf{The task-drift setting} comprises 11 SMACv2 tasks. Each team has 5 agents consisting two different type whose probability changes every task. (2) \textbf{The task-switch setting} switches 2 tasks periodically.  For example, in the SMACV2 5 protoss scenario, the tasks switch between 5 Stalker vs 5 Stalker and 5 Zealot vs 5 Zealot. Predator Prey is evaluated under similar setting. The first two settings can be viewed in the Table ~\ref{tab:smac_streams} in detail.(3) \textbf{The scaling up setting} comprise 3 SMAC tasks where the number of agents and the opponents increase from 3, to 5, and 8. (4) \textbf{The scaling down setting} evaluate 3 SMAC tasks that 5 agents fight against 6, 5, and 4 enemies respectively. (5) \textbf{Non-continue setting} evaluate the original SMAC and SMACv2.

For the original \textbf{SMAC}, we evaluate \textit{structural shifts} where the state and action spaces change due to team size variations every $2\times10^6$ environment steps. This includes \textbf{agent scaling} where the team size increases (e.g., 3m $\rightarrow$ 5m $\rightarrow$ 8m), and \textbf{opponent variation} where the number of enemies fluctuates across tasks (e.g., 5m\_vs\_6m $\rightarrow$ 5m $\rightarrow$ 5m\_vs\_4m).

\textbf{SMACv2.} We modify \textbf{SMACv2} to evaluate \textit{compositional shifts} with two controlled settings: \textbf{Task Switch} and \textbf{Task Drift}.
In \textbf{Task Switch}, the scenario alternates between two mirror-match tasks every $5\times10^5$ environment steps.
For example, in 5gen\_protoss, the task switches between 5 Stalker vs 5 Stalker and
5 Zealot vs 5 Zealot. We apply the same pattern to other races (Terran and Zerg) by switching between two unit types.
In \textbf{Task Drift}, the unit composition changes gradually in discrete stages:
every $5\times10^5$ steps, the spawn probability of the first unit type decreases by $0.1$ (from $1.0$ to $0.0$),
while that of the second increases by $0.1$ (from $0.0$ to $1.0$), producing 10 drift updates and thus 11 task stages in total.

\begin{table}[t]
\centering
\scriptsize
\caption{Task configurations of modified SMACv2. In all cases, the task is updated every $5\times10^5$ env steps.
The unit order is fixed per race; the third unit type has weight 0 throughout.}
\label{tab:smac_streams}
\renewcommand{\arraystretch}{1.15}

\resizebox{\linewidth}{!}{

\begin{tabular}{l l p{3.1cm} p{4.2cm}}
\toprule
\textbf{Race} & \textbf{Unit order} &
\textbf{Task Switch (2 phases)} &
\textbf{Task Drift (11 phases)} \\
\midrule
Protoss &
[Stalker, Zealot, Colossus] &
[1.0, 0.0, 0.0] $\leftrightarrow$ [0.0, 1.0, 0.0] &
[1.0,0.0,0.0] $\rightarrow$ [0.9,0.1,0.0] $\rightarrow \cdots \rightarrow$ [0.0,1.0,0.0] \\
\midrule
Terran &
[Marine, Marauder, Medivac] &
[1.0, 0.0, 0.0] $\leftrightarrow$ [0.0, 1.0, 0.0] &
[1.0,0.0,0.0] $\rightarrow$ [0.9,0.1,0.0] $\rightarrow \cdots \rightarrow$ [0.0,1.0,0.0] \\
\midrule
Zerg &
[Zergling, Baneling, Hydralisk] &
[1.0, 0.0, 0.0] $\leftrightarrow$ [0.0, 0.0, 1.0] &
[1.0,0.0,0.0] $\rightarrow$ [0.9,0.0,0.1] $\rightarrow \cdots \rightarrow$ [0.0,0.0,1.0] \\
\bottomrule
\end{tabular}

}
\end{table}

% \textbf{One-step Matrix Game.} We follow the same sequential setting as Figure~\ref{stagnant:matrix}:
% agents learn a periodic stream of four payoff matrices, where each matrix defines a task and the active matrix
% switches at fixed time intervals.

% \subsection{Detailed Payoff Matrices in the One-Step Matrix Game}

\textbf{One-step Matrix Game (Sequential Setting).}
\label{sec:appendix_matrix_details}
In the one-step matrix game, two agents must learn to coordinate their actions to maximize a shared global reward. Each agent has a discrete action space of size 3, resulting in a $3 \times 3$ joint payoff matrix. To create a challenging task transfer scenario that rigorously tests the plasticity of the mixing network, we designed five distinct payoff matrices where the optimal joint action shifts drastically.

The base matrix, denoted as $M_0$, features a relatively smooth reward gradient guiding the agents towards the optimal joint action at $(2, 2)$ with a maximum reward of $4$:
\begin{equation}
M_0 = \begin{bmatrix}
-6 & -4 & -2 \\
-3 & -1 &  1 \\
 0 &  2 &  4
\end{bmatrix}
\end{equation}

To explicitly evaluate the agents' adaptability and to induce optimization inertia, we introduce four additional tasks ($M_1$ to $M_4$). In these matrices, the global optimum (with a high reward of $12$) is sharply located at different coordinates, surrounded by severe penalties (large negative rewards). This sparse and deceptive reward landscape forces the network to completely overwrite its previous coordination conventions when the task switches:

\begin{equation}
M_1 = \begin{bmatrix}
-10 & -10 &  12 \\
 -9 & -12 &  -8 \\
 -4 &  -8 & -11
\end{bmatrix}, \quad
M_2 = \begin{bmatrix}
-11 &  -8 &  -4 \\
 -8 & -12 &  -9 \\
 12 & -10 & -10
\end{bmatrix}
\end{equation}

\begin{equation}
M_3 = \begin{bmatrix}
-11 &  -4 & -11 \\
 -4 &  12 &  -4 \\
-11 &  -4 & -11
\end{bmatrix}, \quad
M_4 = \begin{bmatrix}
 12 &  -8 & -10 \\
 -8 & -12 &  -6 \\
-10 &  -6 & -11
\end{bmatrix}
\end{equation}

Specifically, the optimal joint actions $(a_1, a_2)$ for $M_1$, $M_2$, $M_3$, and $M_4$ are located at $(0, 2)$, $(2, 0)$, $(1, 1)$, and $(0, 0)$, respectively. This extreme variance across tasks ensures that any accumulated parameter magnitude (i.e., stagnant neurons) from a previous matrix will directly and severely hinder the adaptation to the subsequent one, making it an ideal testbed for evaluating plasticity loss.

\textbf{One-step Matrix Game (Cyclic Setting).}
Unlike the sequential setup evaluated in the main text, each unique payoff matrix constitutes a task. The agents must learn a sequence of four tasks (Task 1–4), with the active payoff matrix switching periodically at fixed intervals, as illustrated in Figure~\ref{app:matrix_periodic_intro} left. 

% Figure~\ref{stagnant:matrix_periodic} middle depicts the overall approximation bias of QMIX across all tasks, respectively. The horizontal axis represents the cycle number, while the vertical axis indicates the bias. Each cycle consists of four consecutive tasks, represented by blue, yellow, green, and pink rectangles, with the color indicating the currently active task.

\textbf{Predator--Prey.} Predator--Prey is a grid-world cooperation task where predators must coordinate to capture preys;
a capture succeeds only when at least two predators execute the capture action simultaneously. We construct a non-stationary task stream by periodically switching the reward profile between summer and winter. In both phases, the time reward is $0$ and capturing a hare yields $+5$. The only change is the stag reward: it is $+10$ in summer but $-5$ in winter. Thus, the task shift purely comes from a sign flip in the incentive for stag captures, while the environment dynamics remain unchanged.

\paragraph{Baselines.}
We compare against neuron-level plasticity injection baselines including ReDo~\cite{redo}, ReBorn~\cite{reborn}, ReGraMa~\cite{regrama}, S\&P~\cite{s&p}, CBP~\cite{nature24plasticity}, ReSet~\cite{ReSet2}, and Plasticity Injection (PI)~\cite{plasticity2}. Note that although ReSet can reduce the stagnant ratio, its large-scale resetting may forget previously learned knowledge, leading to slow convergence and poor overall performance under task changes.

% \paragraph{KNIFE hyperparameters.}
% Unless otherwise stated, KNIFE is applied periodically every $T$ environment steps. Neurons are selected for repair using the RUA thresholds $(\alpha,\beta)$ (stagnant if $\mathrm{RUA}<\alpha$, volatile if $\mathrm{RUA}>\beta$). In the pruning stage, frozen branches are multiplicatively decayed by $\gamma\in(0,1)$ and pruned once their weight norms fall below a small threshold $\varepsilon_{\mathrm{prune}}$.

% \subsection{Experimental Results and Analysis}
% \subsubsection{Stagnant Neuron in MARL}
\subsection{Stagnant Neuron in MARL}
\label{app:staginmarl}

\paragraph{The Plasticity Loss with New MARL Tasks}

In our modified SMACv2 scenarios, two opposing teams of five agents engage in combat. To explicitly evaluate sequential adaptation, we periodically switch the unit composition for \emph{both} sides on a fixed schedule. The task sequences are designed as follows:
\begin{itemize}
    \item \textbf{Protoss Setting:} The environment alternates between purely homogeneous compositions: 100\% Stalkers, 100\% Zealots, and 100\% Colossi.
    \item \textbf{Terran Setting:} The sequence progresses from 100\% Marines to 100\% Marauders, and finally to a heterogeneous mixture of 80\% Marines and 20\% Medivacs.
\end{itemize}

Figure~\ref{app:fig:112311} illustrates the impact of these periodic shifts on standard value factorization methods (QMIX, QPLEX, and DMIX). Subfigures (a) and (c) track the overall win rates across the entire task sequence for the Protoss and Terran settings, respectively. To provide a clearer micro-perspective, subfigures (b) and (d) isolate the learning curves of a single recurring scenario (e.g., 5 Stalkers vs. 5 Stalkers, and 5 Marines vs. 5 Marines) throughout the training process. 

Across all evaluated baselines, we observe a severe and compounding performance degradation as new task instances are introduced. Even when returning to a previously encountered scenario, the agents' learning efficiency drops drastically. Together, these empirical results compellingly demonstrate that MARL systems suffer a fundamental and systemic loss of plasticity as the task sequence progresses.

\begin{figure*}[h]
	\centering
	
    \includegraphics[width=0.49\columnwidth]{ICML26/experiment/stagnant/figure_2.1_5gen_protoss_tasks_v7_500k_test_battle_won_mean.pdf} 
        \includegraphics[width=0.49\columnwidth]{ICML26/experiment/stagnant/figure_2.2_5gen_protoss_tasks_v7_500k_Task1_test_battle_won_mean.pdf} 
        \includegraphics[width=0.49\columnwidth]{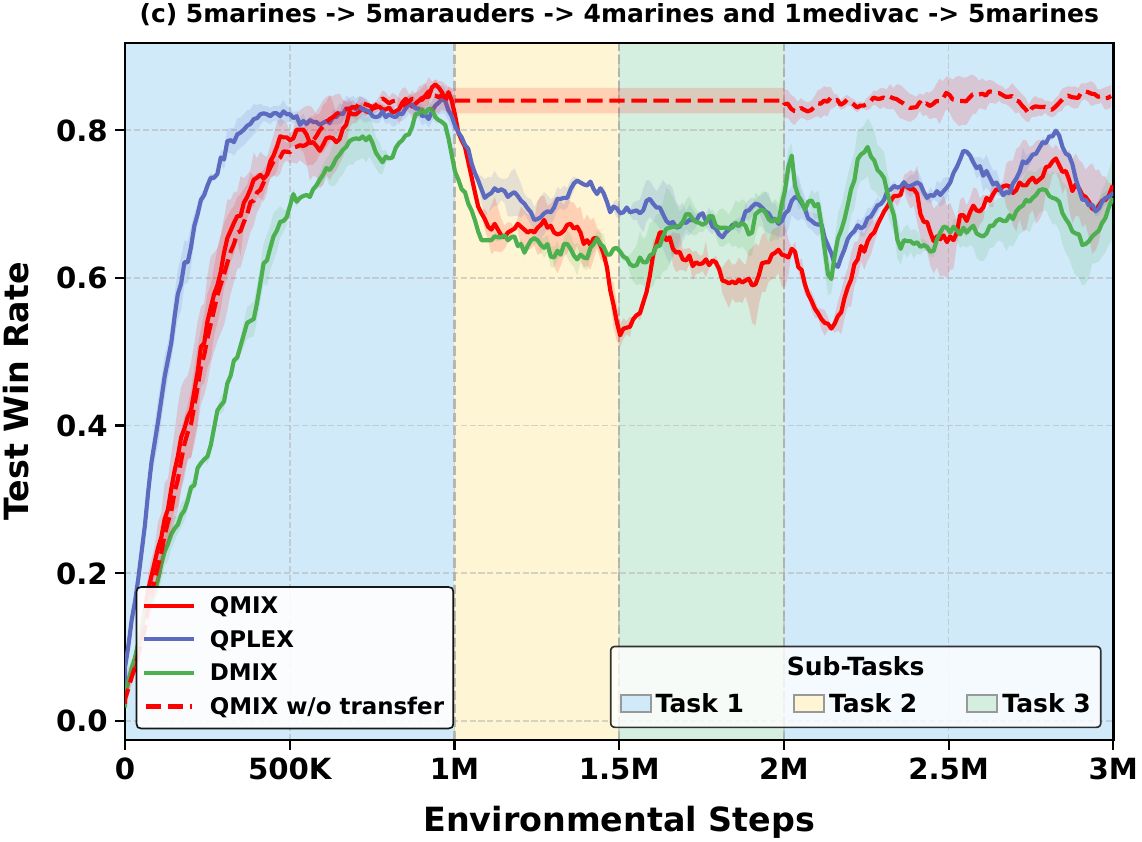} 
		\includegraphics[width=0.49\columnwidth]{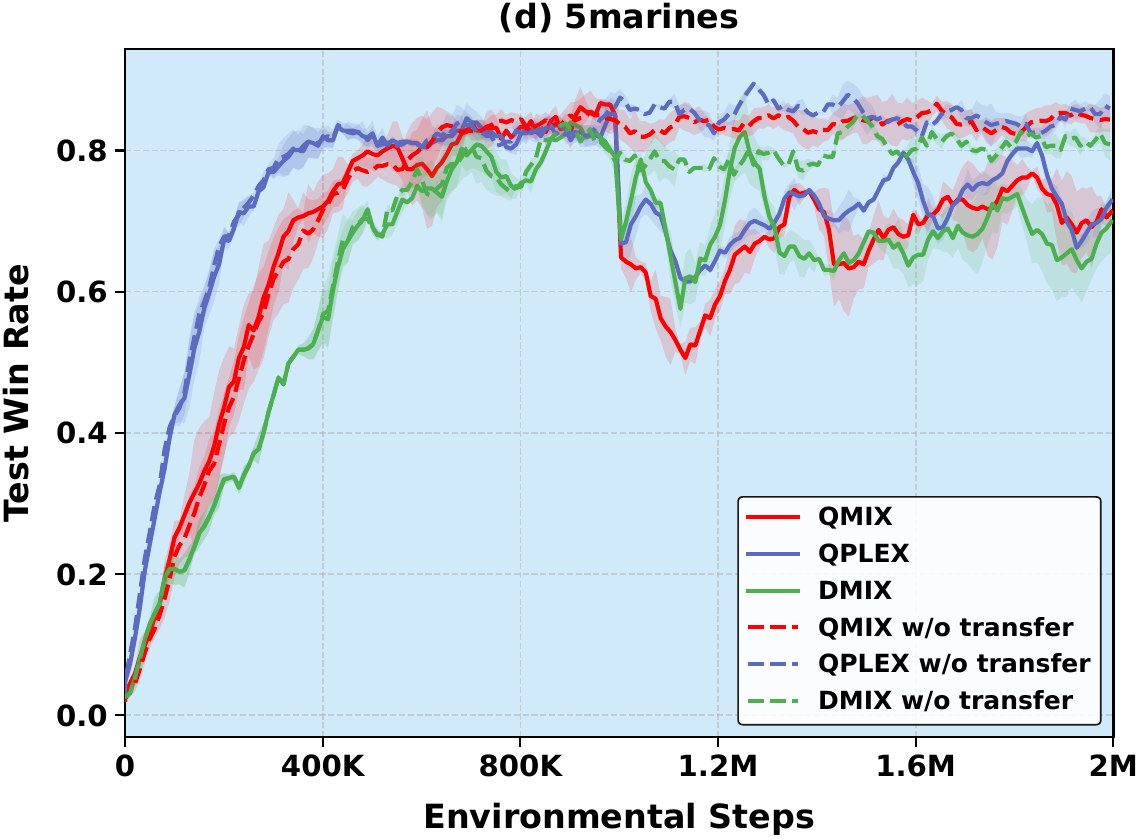} 
    
        \caption{Plasticity loss in modified SMACv2 with sequential tasks. (a) Overall win rate for all Protoss tasks, (b) Win rate for Protoss task 1, (c) Overall win rate for all Terran tasks, (d) Win rate for Terran task 1 }
	\label{app:fig:112311}
\end{figure*}

More examples of plasticity loss in MARL when facing new tasks are shown in the Figure~\ref{app:1}. After task switching, the neural network suffers a loss of plasticity, which leads to performance degradation.

\begin{figure*}[h]
	\centering
	
    \begin{minipage}[b]{\linewidth} % 第一行画图结果
		\centering
		\includegraphics[width=0.32\columnwidth]{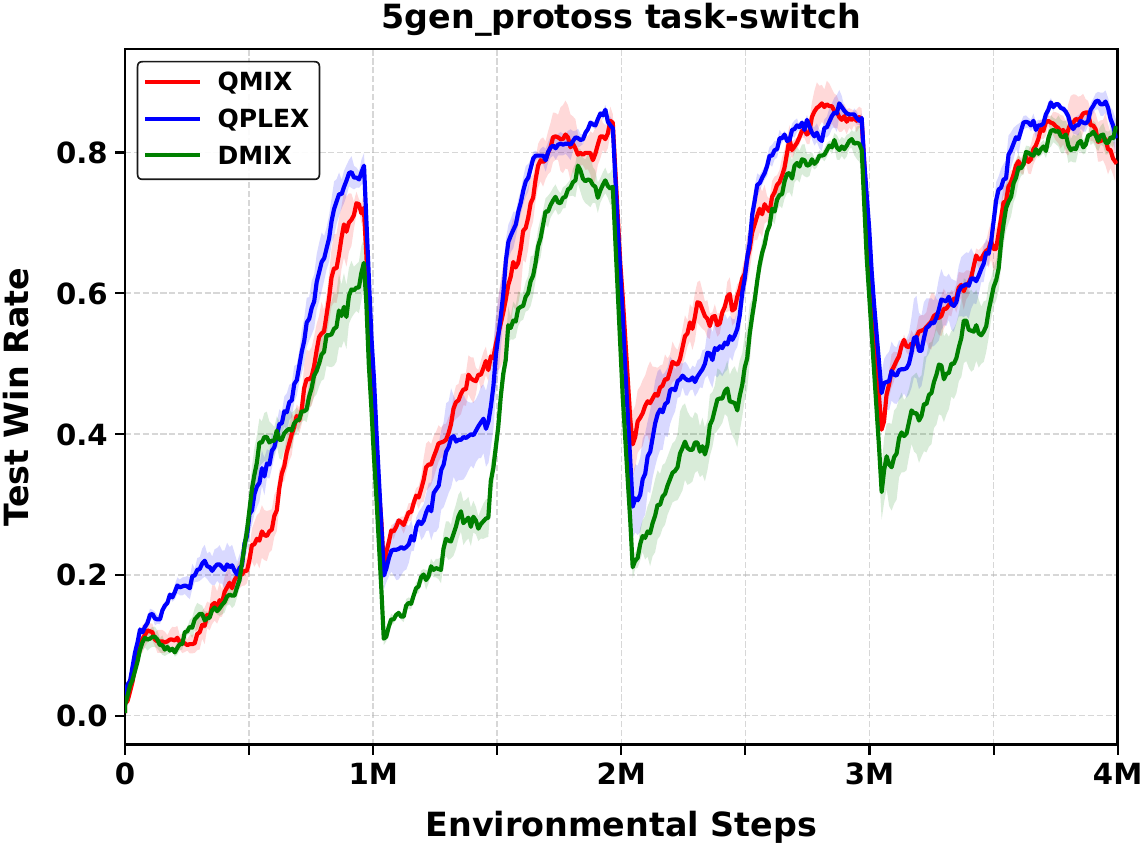} 		
		\includegraphics[width=0.32\columnwidth]{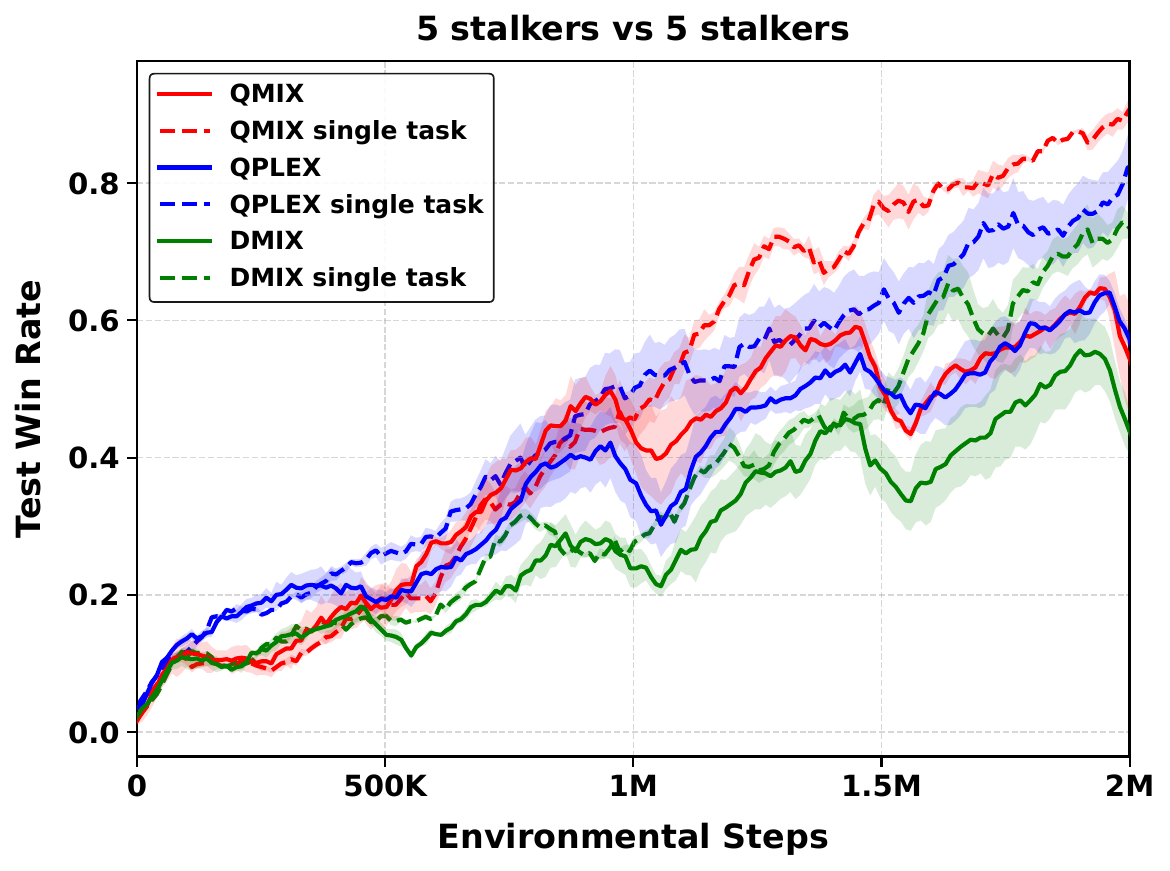} 	
		\includegraphics[width=0.32\columnwidth]{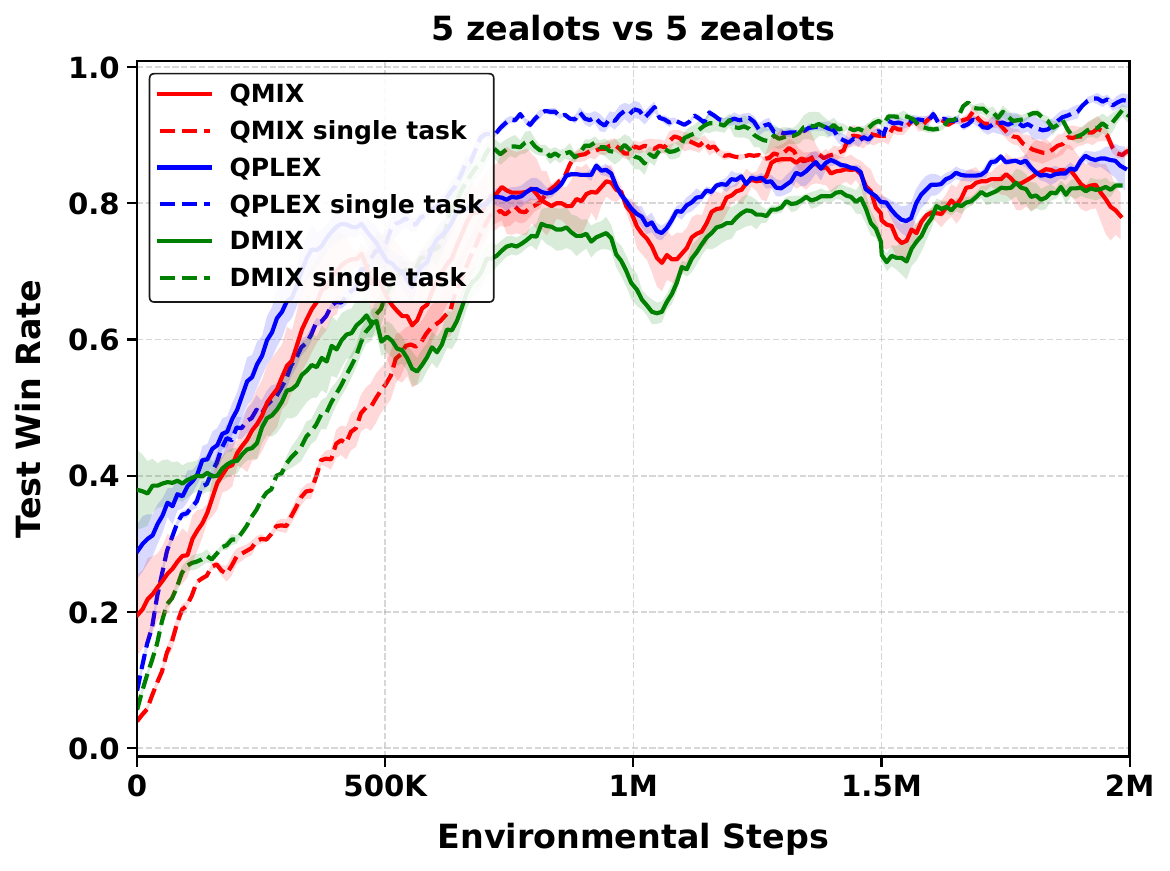}
	\end{minipage}
    \begin{minipage}[b]{\linewidth} % 第一行画图结果
		\centering
		\includegraphics[width=0.32\columnwidth]{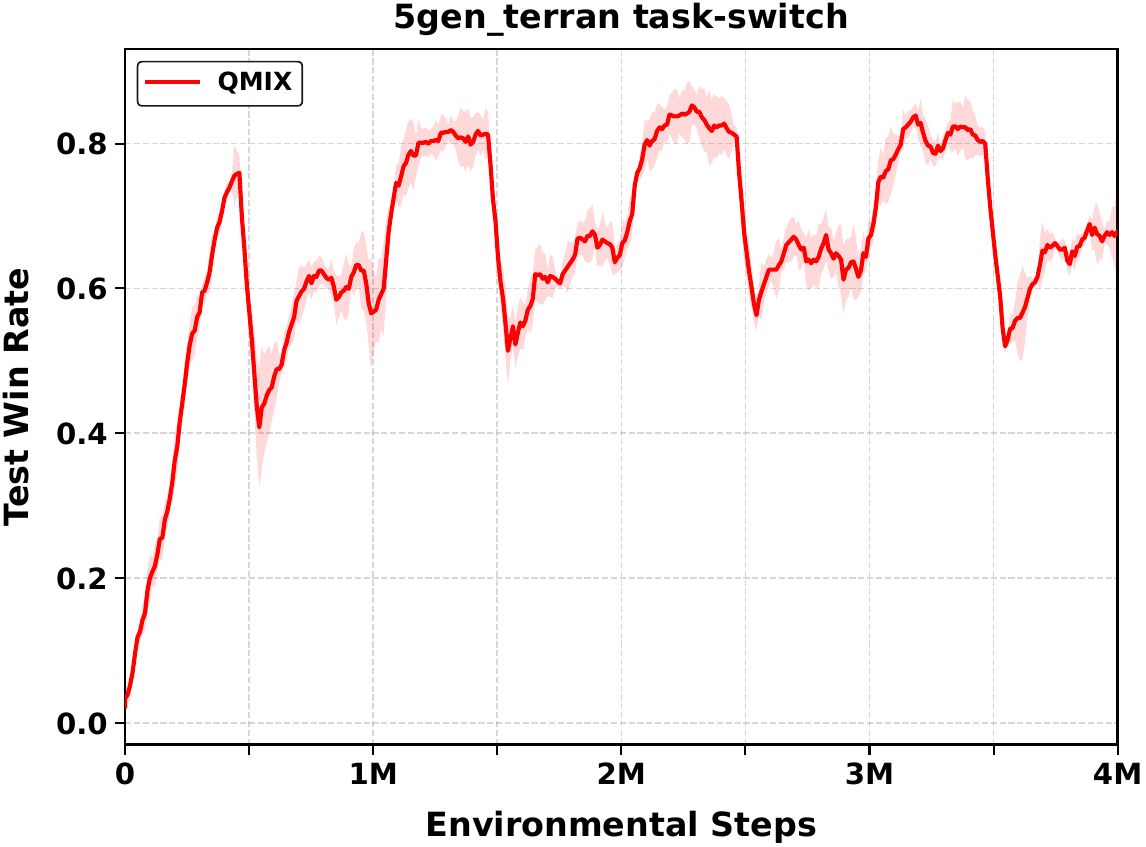} 		
		\includegraphics[width=0.32\columnwidth]{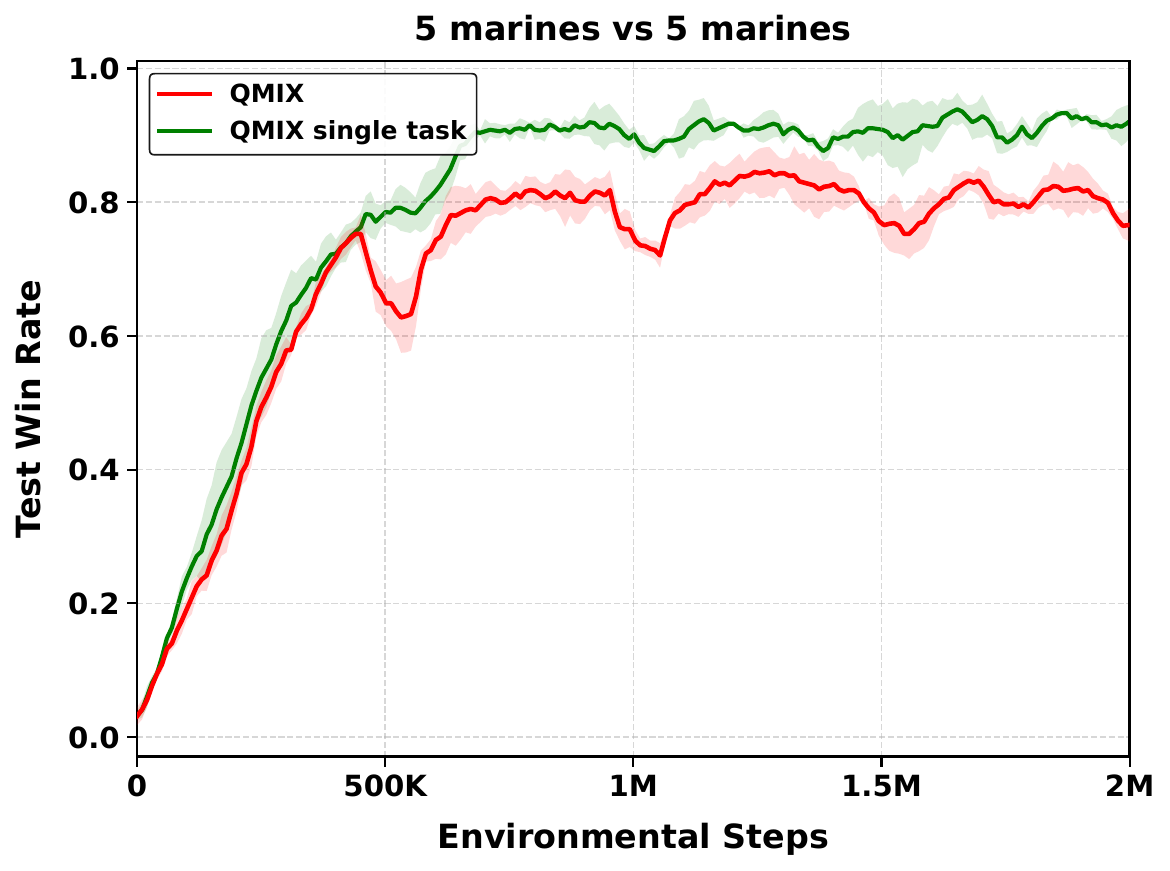} 	
		\includegraphics[width=0.32\columnwidth]{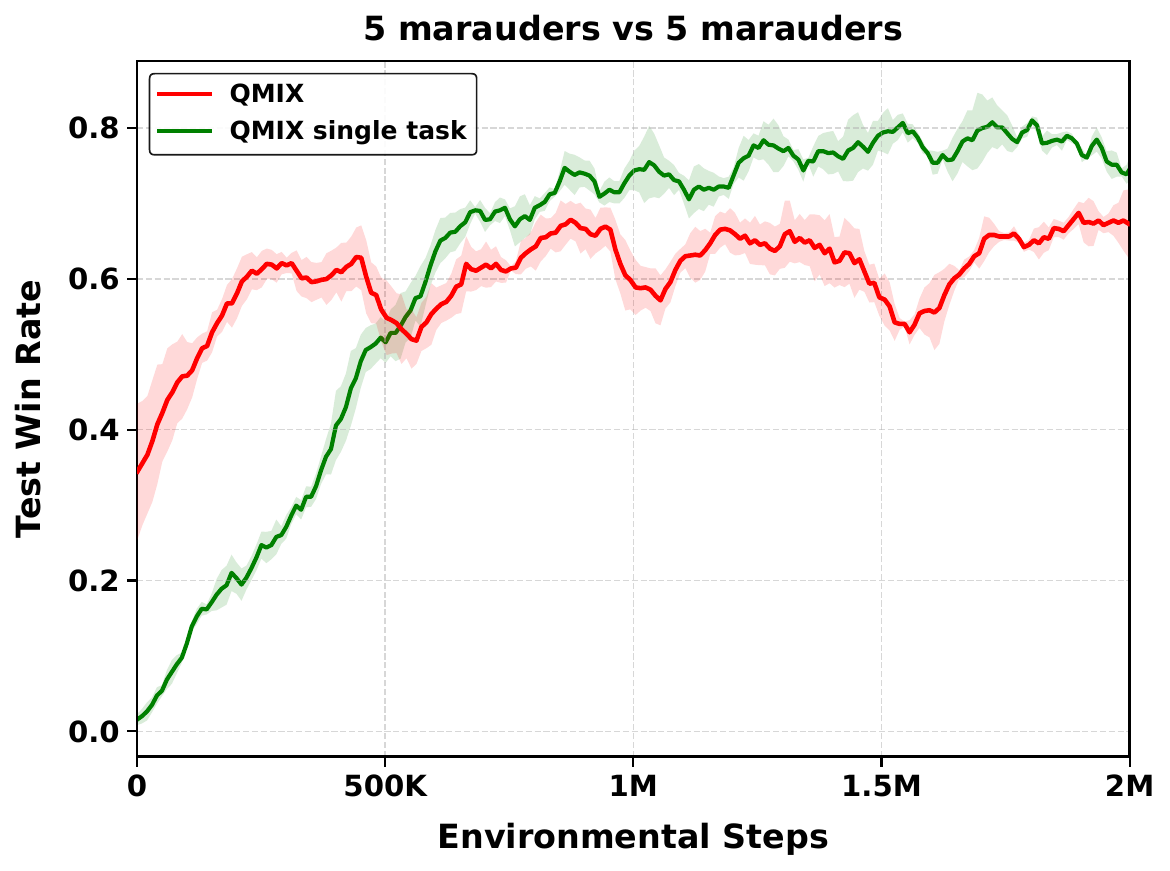}
	\end{minipage}
    \begin{minipage}[b]{\linewidth} % 第一行画图结果
		\centering
		\includegraphics[width=0.32\columnwidth]{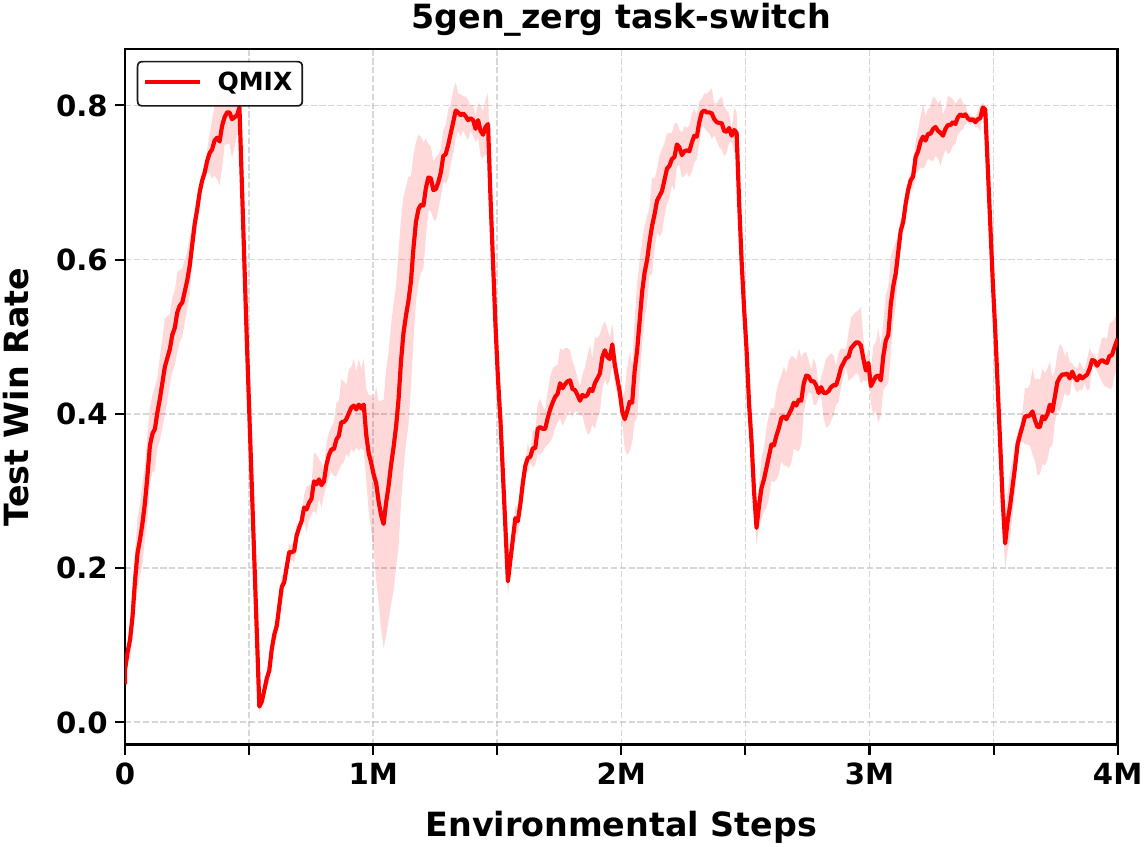} 		
		\includegraphics[width=0.32\columnwidth]{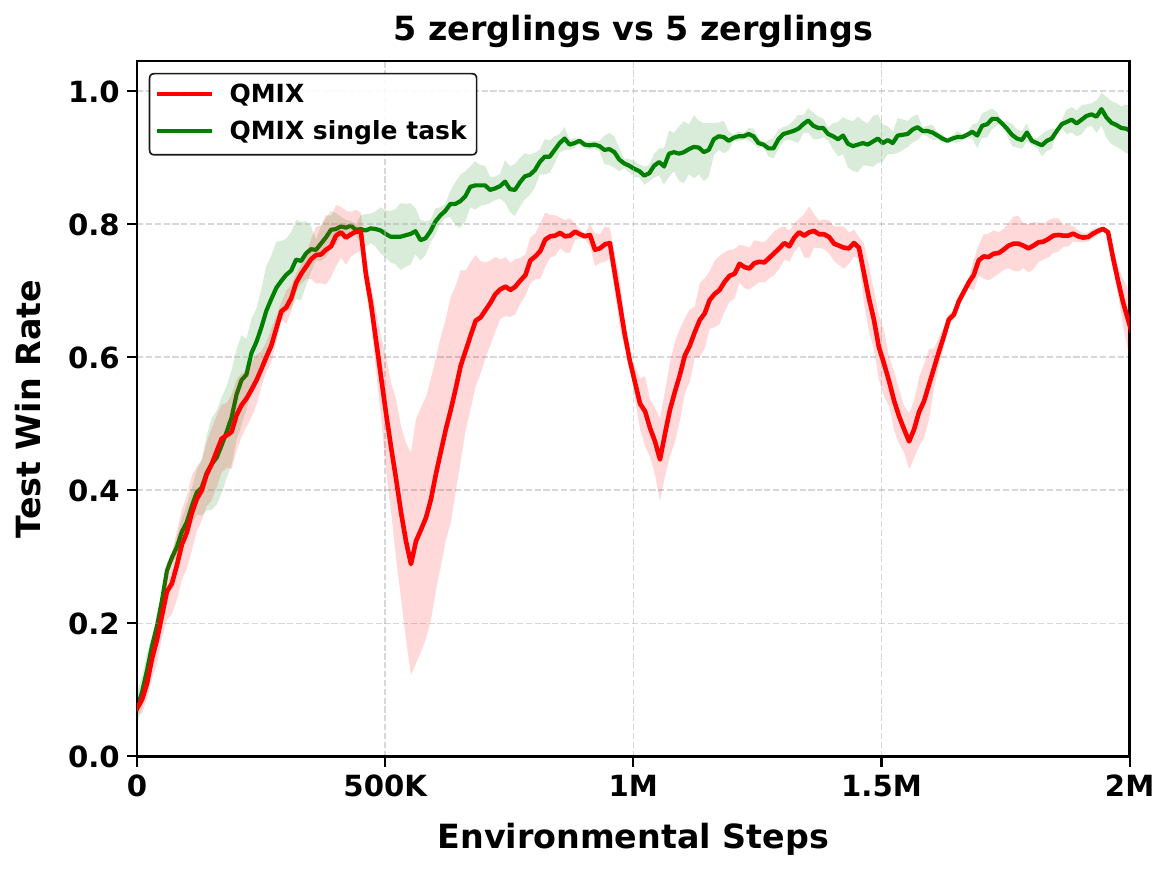} 	
		\includegraphics[width=0.32\columnwidth]{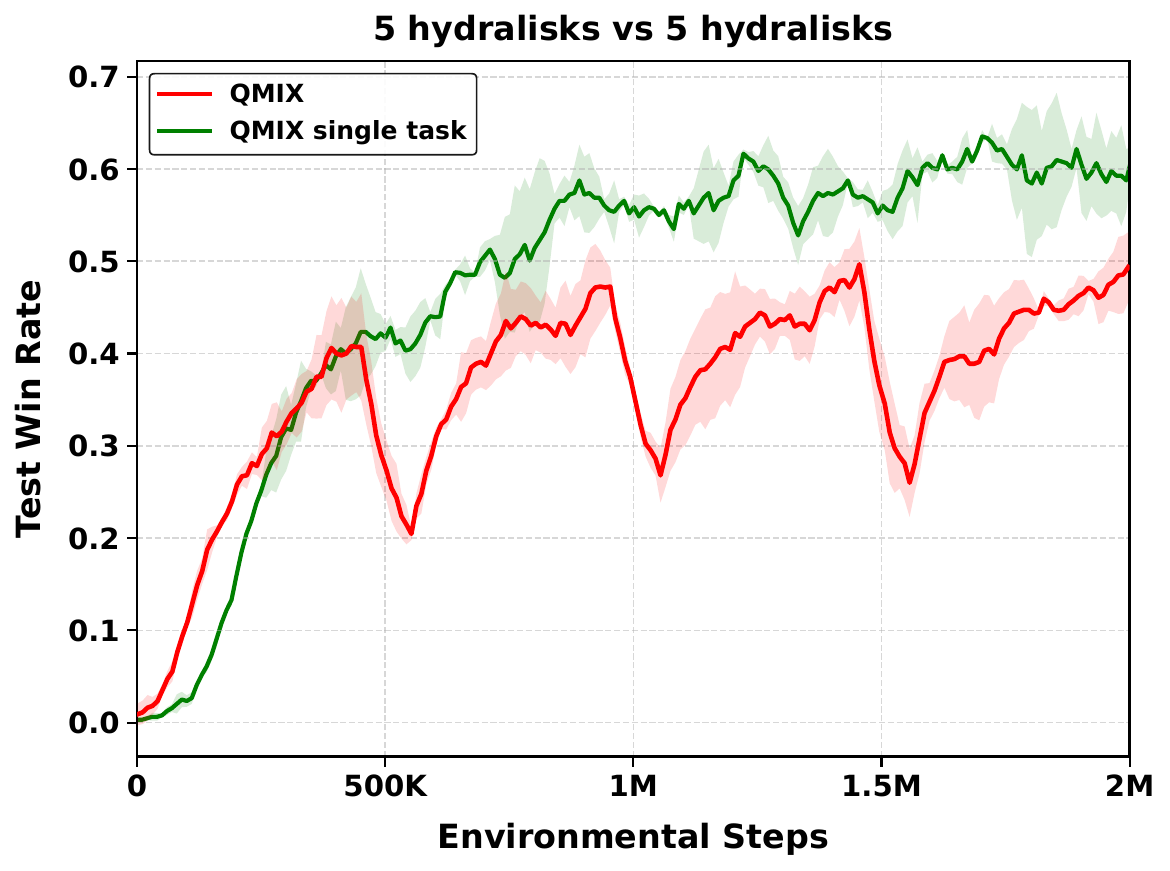}
	\end{minipage}

        \caption{The plasticity loss with new MARL tasks}
	\label{app:1}
\end{figure*}

\begin{figure*}[h]
	\centering

	\begin{minipage}[b]{\linewidth} % 第一行画图结果
		\centering
		\includegraphics[width=0.32\columnwidth]{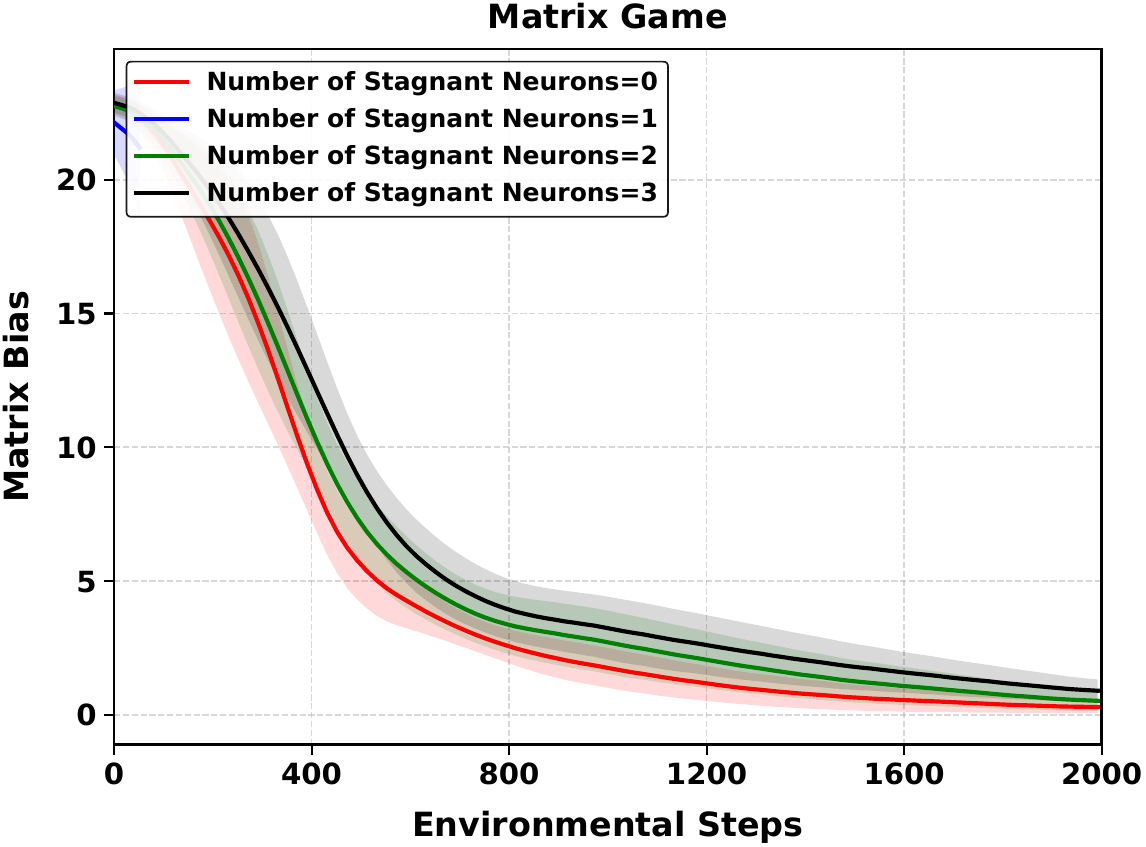} 		
		\includegraphics[width=0.32\columnwidth]{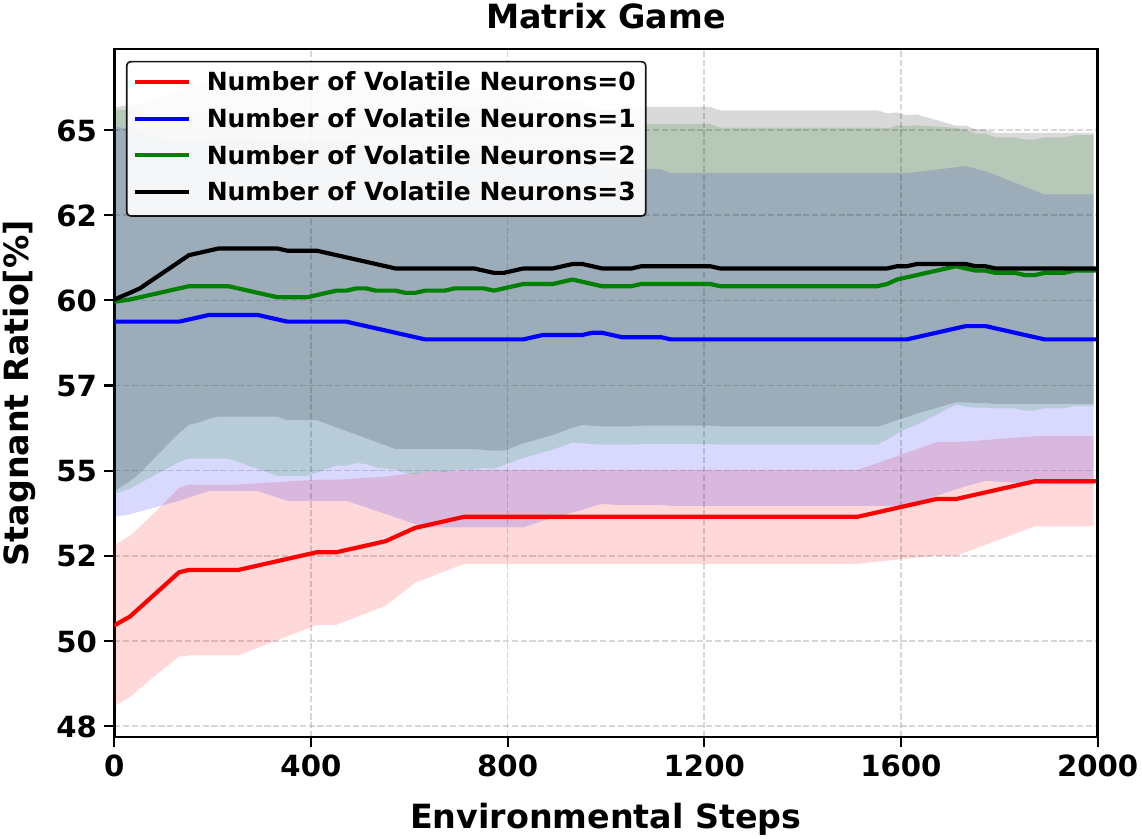} 	
		\includegraphics[width=0.32\columnwidth]{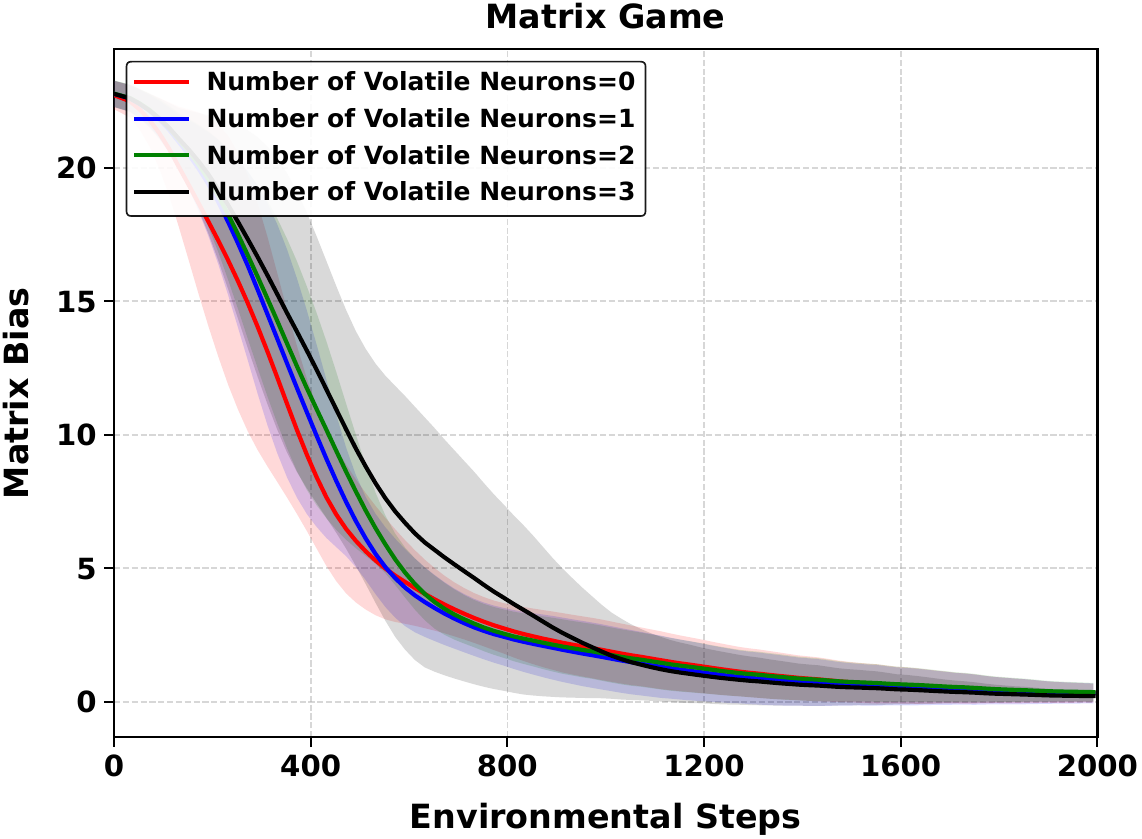}
	\end{minipage}

        \caption{The influence of stagnant and volatile neurons}
	\label{app:fig:influOfStag}
\end{figure*}
\paragraph{Expressiveness study: injecting stagnant/volatile neurons.}
\label{app:sec:influOfStag}

To study how stagnant neurons affect the expressive power of value factorization networks, we create controlled stagnant/volatile neurons by directly manipulating gradients.
Specifically, we select a subset of neurons in the mixer network and modify their backpropagated gradients during training.
For \textbf{stagnant neurons}, we freeze learning by masking their gradients to zero (i.e., gradient freezing).
For \textbf{volatile neurons}, we amplify their gradients by a factor of 100 before the optimizer update.
This intervention changes the effective update magnitude (and thus the RUA) without changing the forward computation graph.The detailed results are shown in the Figure~\ref{app:fig:influOfStag}. Increasing the number of volatile neurons leads to a rise in the number of lazy neurons, which may be an underlying cause of the performance degradation.

% \begin{figure}[th] %
% 	\centering
	
%         \begin{minipage}[b]{\linewidth} % 第一行画图结果
% 		\centering
        
%         \includegraphics[width=0.49\columnwidth]{ICML26/experiment/stagnant/[5.1]3s_vs_5z_All_Ratios.pdf} 		
%         \includegraphics[width=0.49\columnwidth]{ICML26/experiment/stagnant/[5.2]stag_hunt_All_Ratios.pdf} 
% 		% \includegraphics[width=0.32\columnwidth]{ICML26/experiment/smacv2/v6/[3.1]5gen_protoss_tasks_v6_500k_Ratio_Combined.pdf}
% 	\end{minipage}

%         \caption{The ratio of the Stagnant, Dormant, and GraMa neurons in (a) 3s\_vs\_5z of SMAC and (b) Predator-prey.}
% 	\label{stagnant:otherneurons}
%     \vspace{-0.3cm}
% \end{figure}

%%智能体网络的摆烂更多
\paragraph{The Distribution of Stagnant Neurons}
We find that stagnant neurons concentrate in the
mixing network in the Figure~\ref{app:fig:StagnantInMixing}. 

\begin{figure}[th] %
	\centering
	
        \includegraphics[width=0.49\columnwidth]{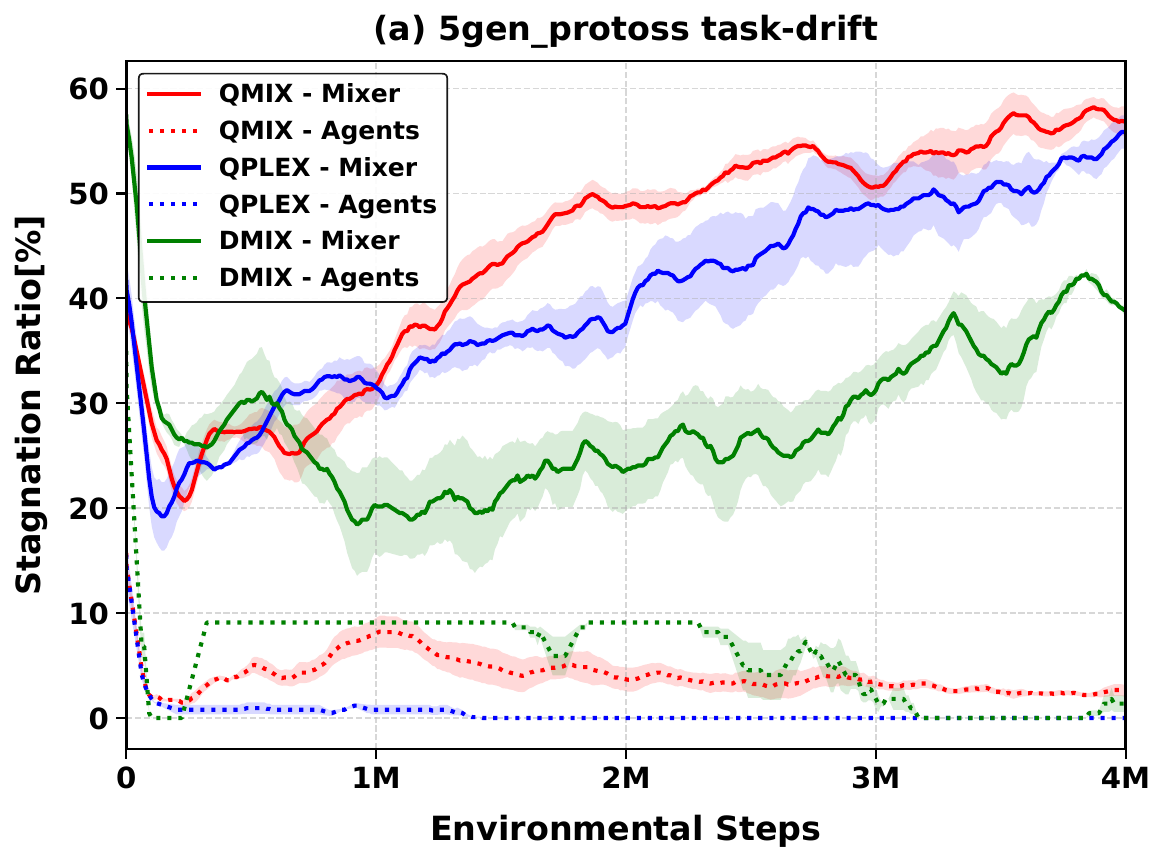}  \includegraphics[width=0.49\columnwidth]{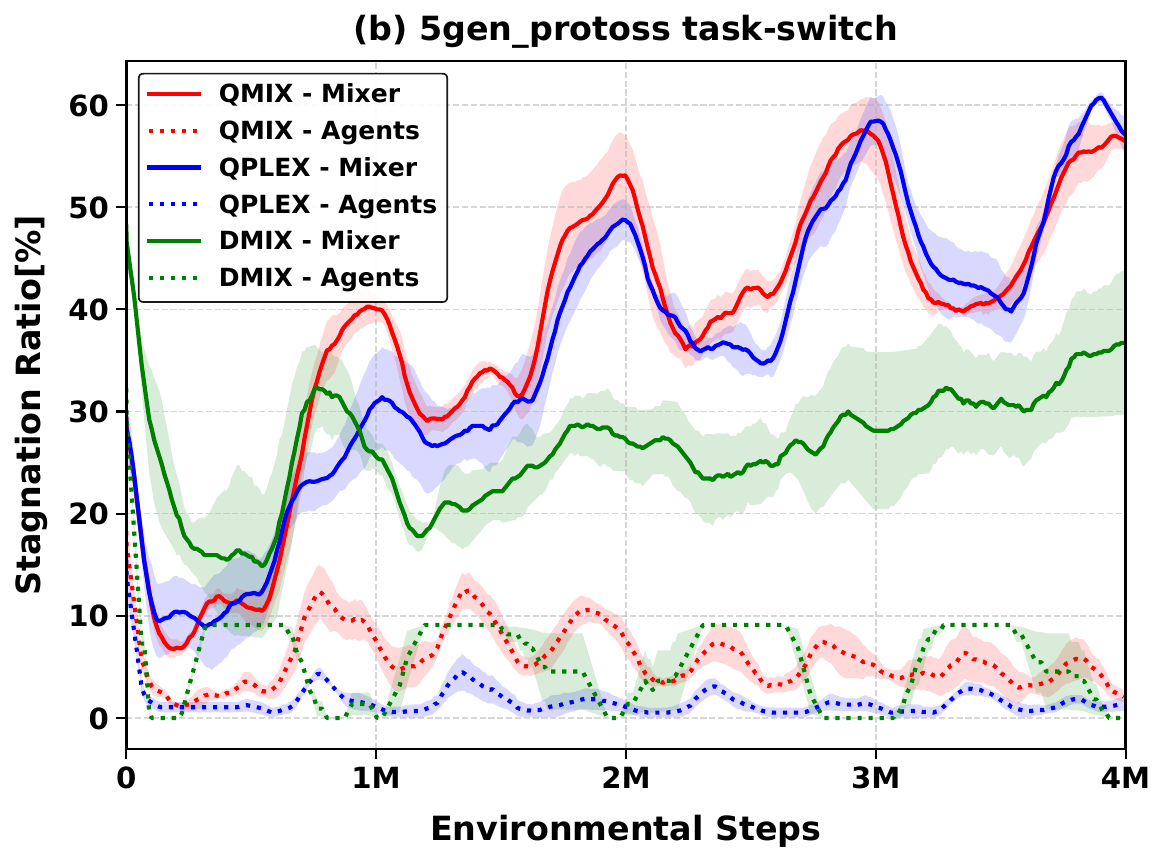}
        \includegraphics[width=0.49\columnwidth]{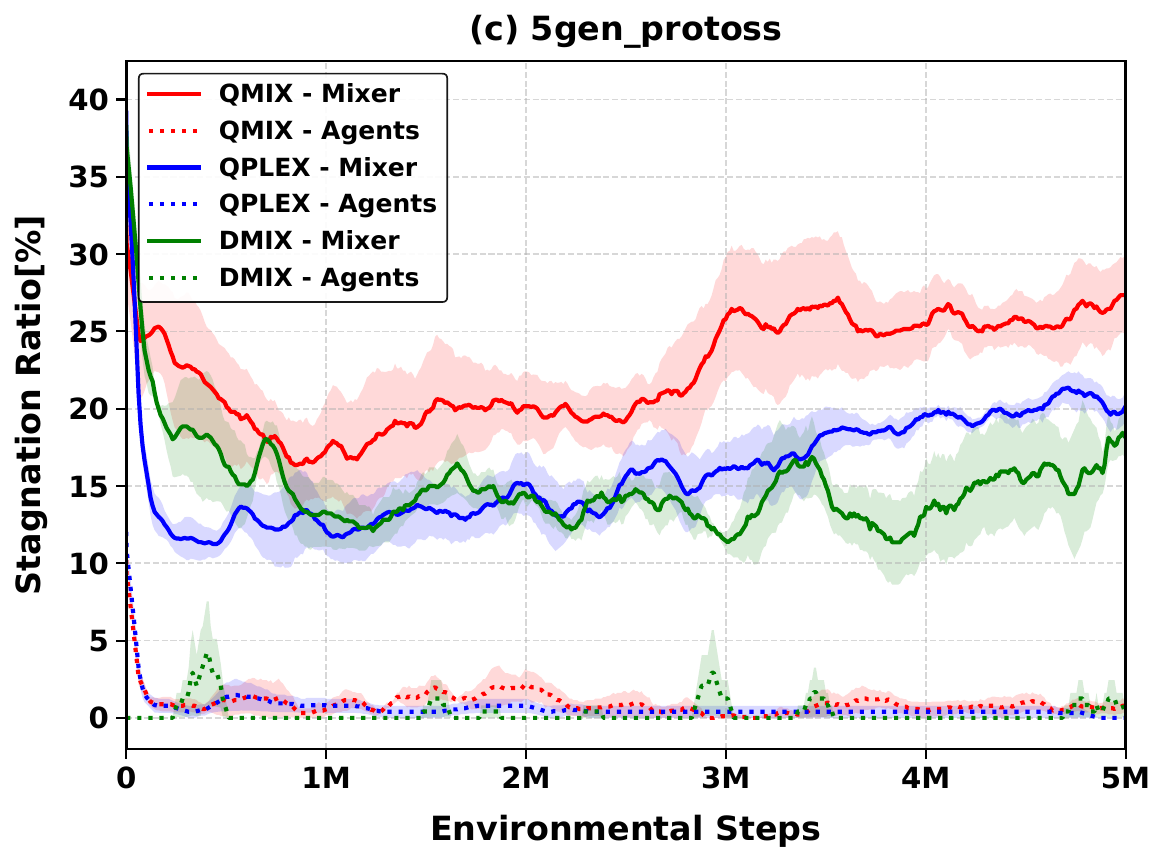} 
		\includegraphics[width=0.49\columnwidth]{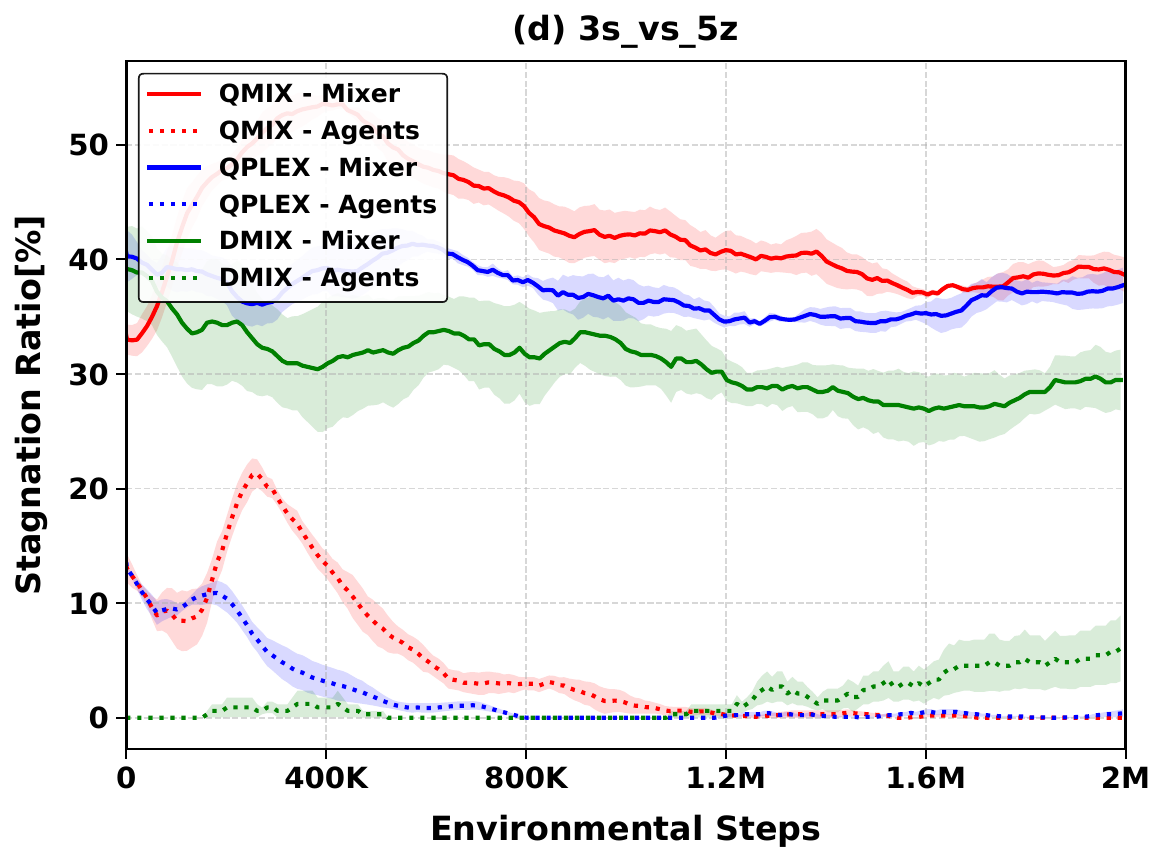}
%	\end{minipage}
        \caption{Stagnant neurons are prevalent across environments, and are consistently more concentrated in the mixer network than in the agent networks: (a)SMACv2 5 vs 5 protoss task, (b) SMAC 3s\_vs\_5z.}
	\label{app:fig:StagnantInMixing}
	\vspace{-0.3cm}
\end{figure}

% \subsubsection{Target Update Interval and Task Update Interval}
\paragraph{Target Update Interval and Task Update Interval.}
We investigate how \textbf{target non-stationary} amplifies stagnation. We consider two
practical sources: \emph{implicit drift} from TD bootstrapping (controlled by the target-network update interval) and
\emph{explicit shifts} from task switching (controlled by the switching period). In Figure~\ref{app:fig:target_non_stationary} (a), the stagnant ratio under different target-network update interval are shown: an intermediate interval yields the lowest stagnant ratio ($200$ steps), followed by $50$, whereas very frequent updates ($10$) lead to substantially higher stagnation. This indicates that even without task switches, the implicit non-stationary introduced by TD learning can drive the accumulation of low-plasticity units (stagnant neurons). In Figure~\ref{app:fig:target_non_stationary} (b), we vary the task switch periods. It shows that more frequent switching leads to faster growth and a higher stagnant ratio, demonstrating that frequent explicit task changes can exacerbate stagnation.

We study whether the detected stagnant/volatile neurons persist across time. To this end, at each iteration (every time interval), we calculate the overlap coefficient between the current and previous iterations for stagnant/volatile neurons. As shown in Figure~\ref{app:fig:overlap}, there are large overlaps between stagnant neurons across iterations, as well as the volatile neurons. This indicates that once a neuron becomes stagnant/volatile, it will remain the same type of neuron over time with high probability. As we show that these neurons can affect plasticity, \emph{a surgical intervention on these neurons is needed to alleviate the plasticity loss problem.}

\begin{figure}[!th]
	\centering
	
    \begin{minipage}[b]{\linewidth} % 第一行画图结果
		\centering
        
        \includegraphics[width=0.49\columnwidth]{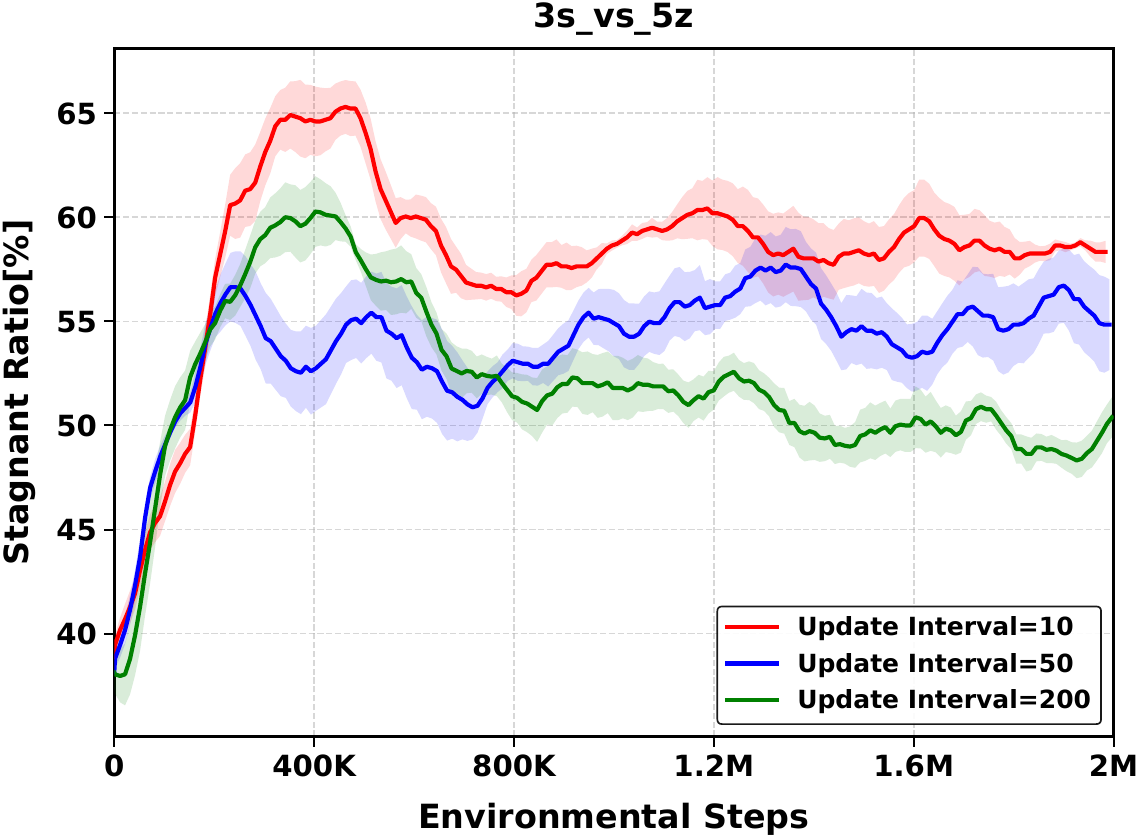} 
        \includegraphics[width=0.49\columnwidth]{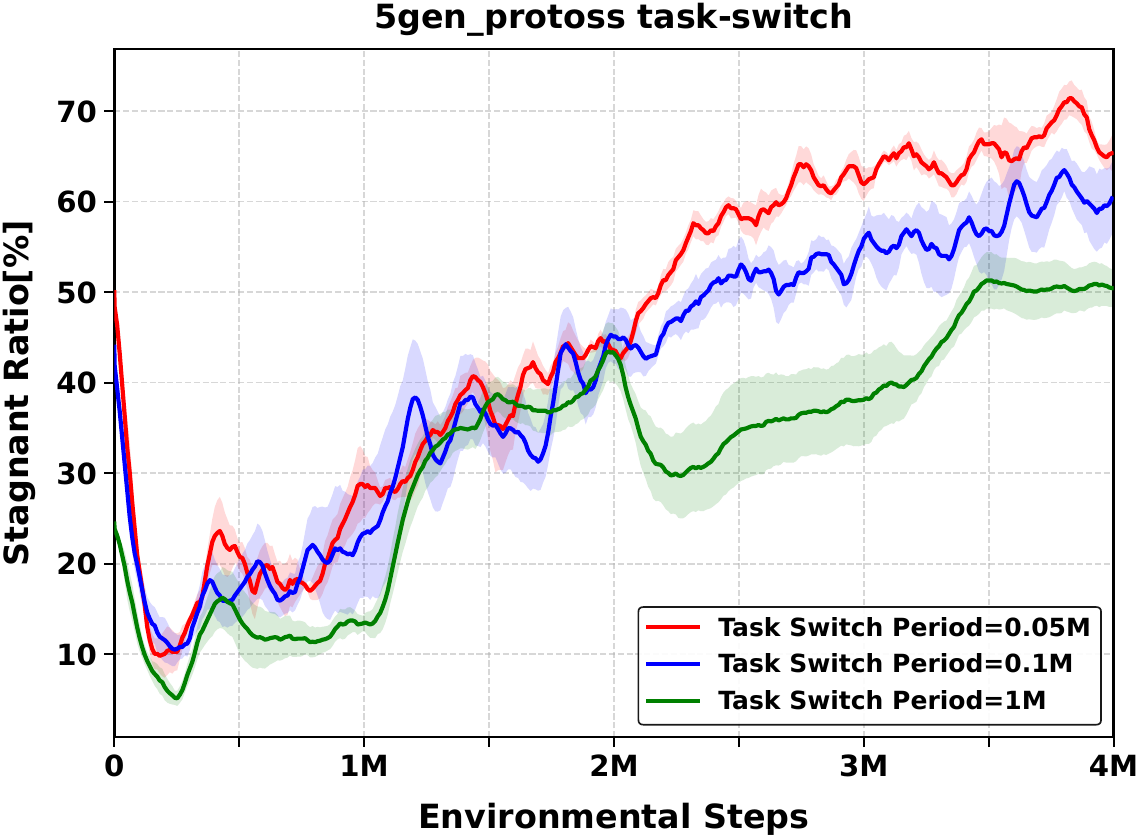}

	\end{minipage}
        \caption{Stagnant ratio under different (left) target network update interval (right) task switching period.}
	\label{app:fig:target_non_stationary}
    \begin{minipage}[b]{\linewidth} % 第一行画图结果
		\centering
        
         \includegraphics[width=0.32\columnwidth]{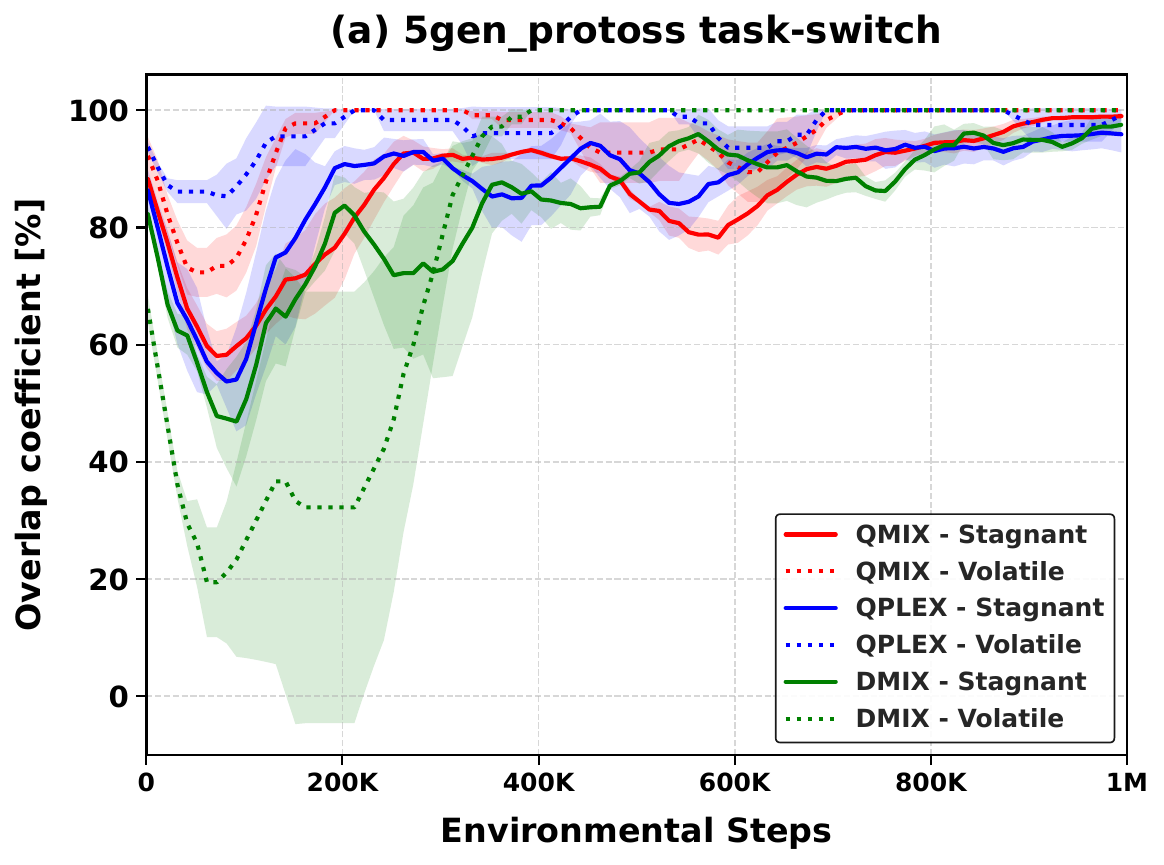} 
        \includegraphics[width=0.32\columnwidth]{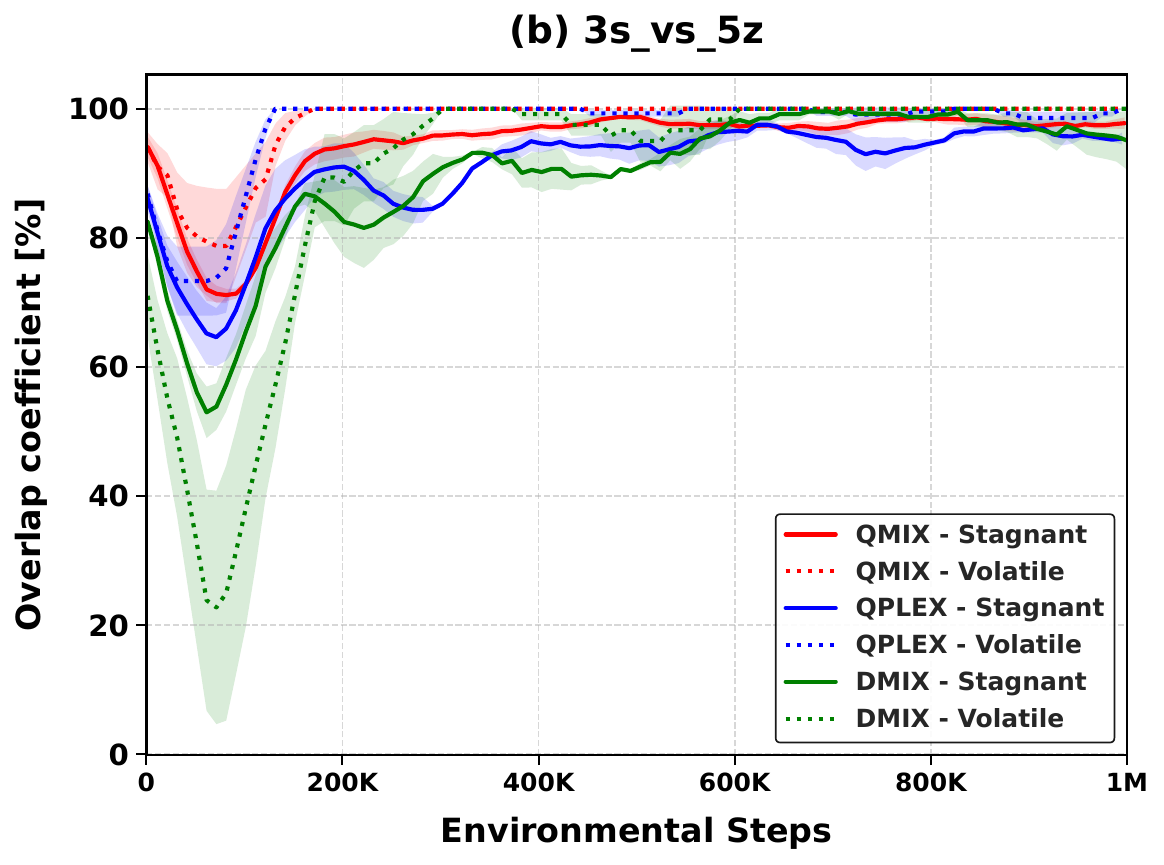} 		
        \includegraphics[width=0.32\columnwidth]{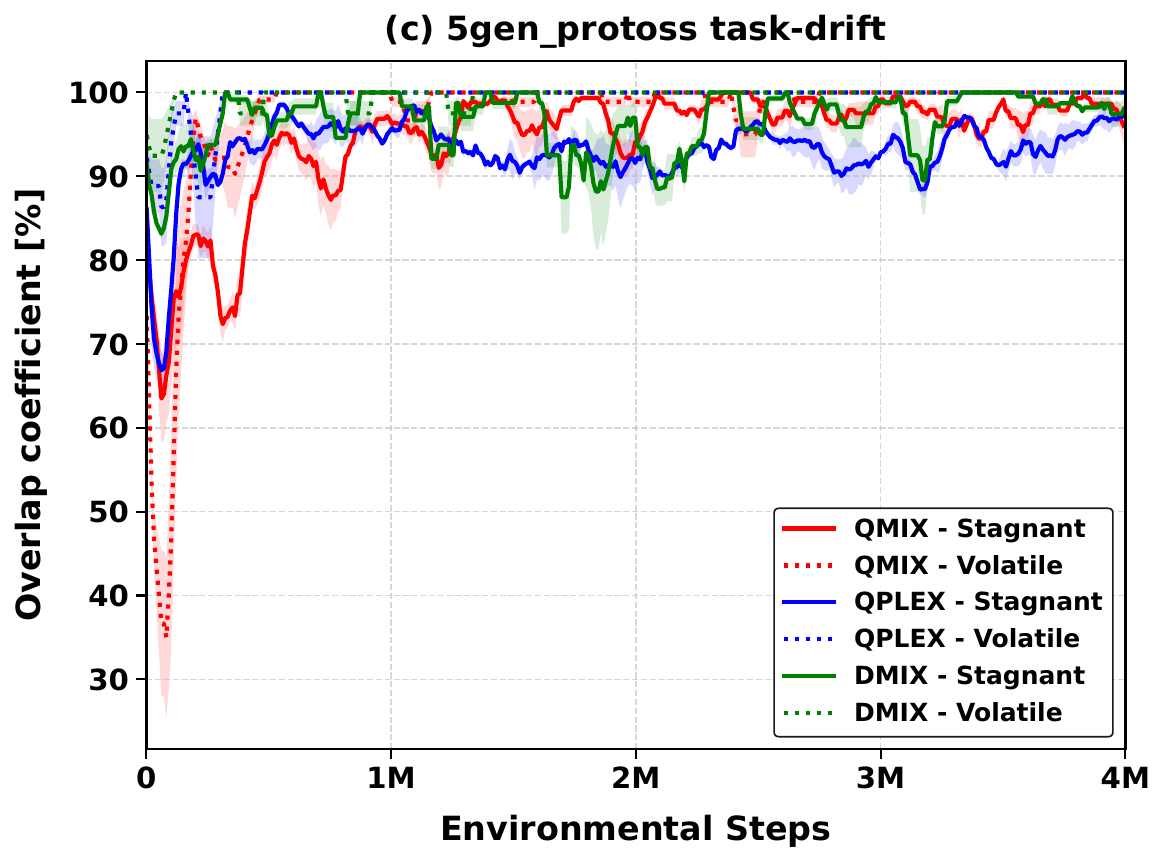} 
	\end{minipage}
        \caption{Overlap coefficient stagnant/volatile Neurons between the current iteration and the previous iteration: (a) SMACv2 5 gen\_protoss task-switch and (b) SMAC 3s\_vs\_5z and (right) SMACv2 5 gen\_protoss task-drift }
	\label{app:fig:overlap}
\end{figure}

\subsection{KNIFE Performs Better than other Plasticity Injection Methods for MARL Tasks}

In the One-step Matrix Game with Sequential tasks at Figure~\ref{app:fig:112311}, KNIFE can solve the stagnant phenomenon and get the lowest matrix bias.

In the one-step matrix game with cycling tasks at Figure~\ref{app:matrix_periodic_intro} and Figure~\ref{app:fig:matrix_periodic}, KNIFE mitigates plasticity loss and enhances the learning capacity for different tasks. The detailed performance on different tasks is shown in Figure~\ref{app:matrix_periodic_single}.

The detailed data for other benchmarks are shown in Figure~\ref{app:bench1} and Figure~\ref{app:bench2}.

\begin{figure}[!th]
	\centering
%	\begin{minipage}[b]{0.49\linewidth} % 第一行画图结果 一个现象，两个指标
%		\centering
		\includegraphics[width=0.32\columnwidth]{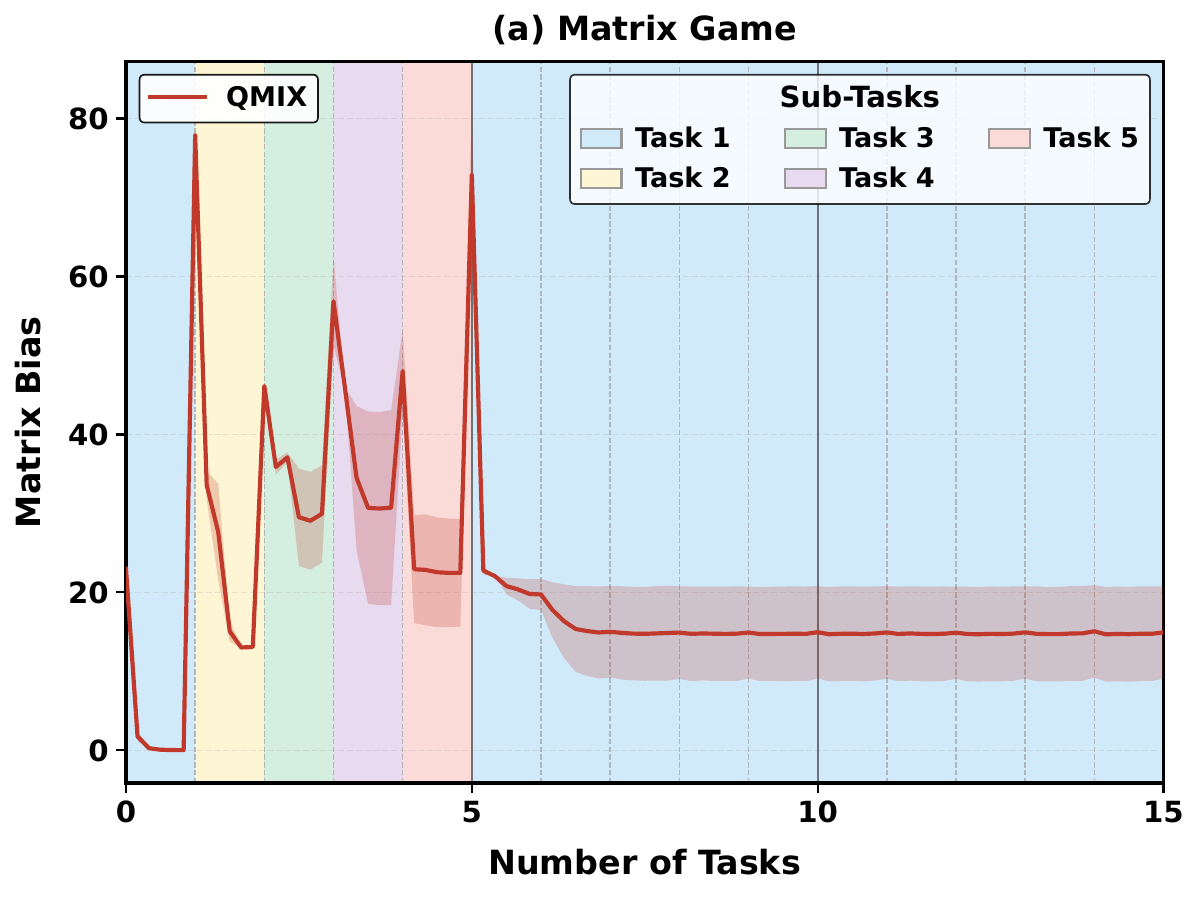} 
        \includegraphics[width=0.32\columnwidth]{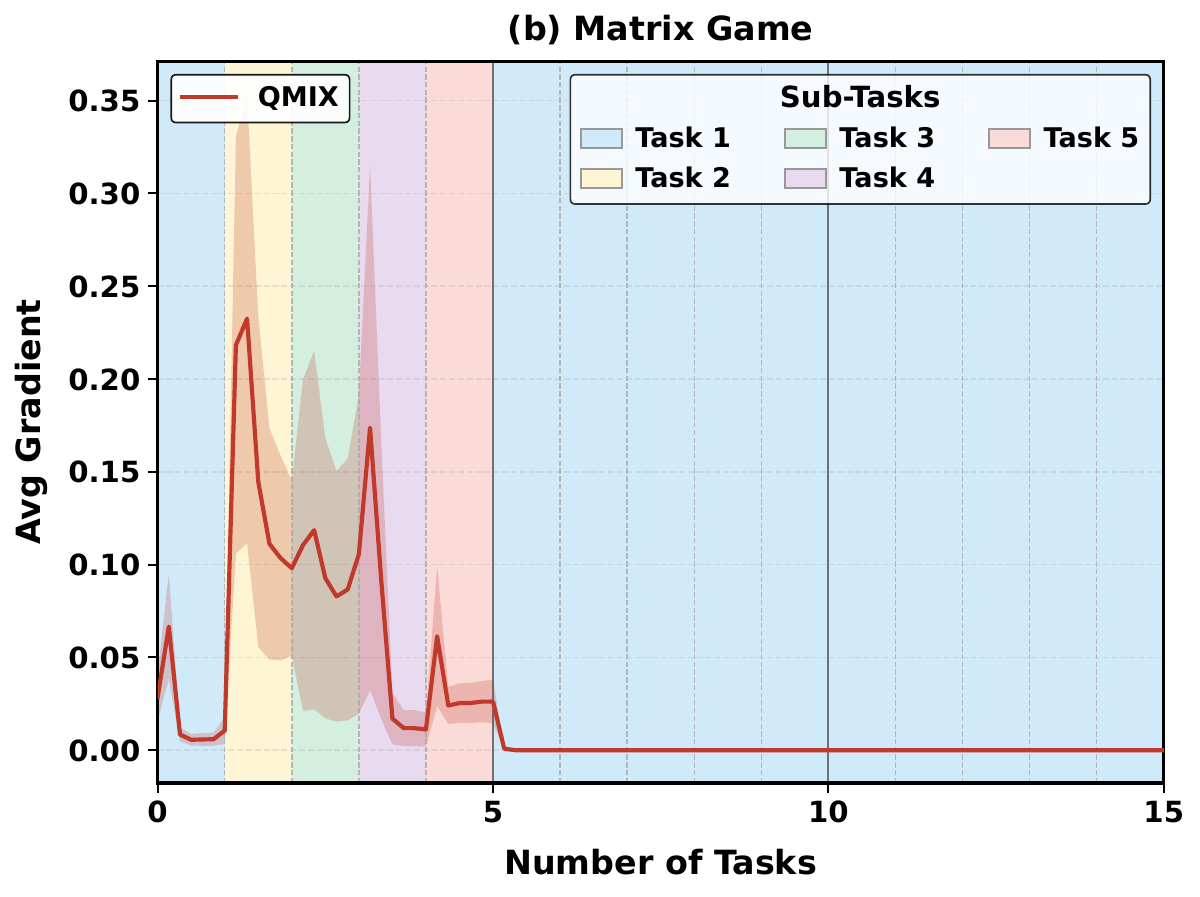} 		
        \includegraphics[width=0.32\columnwidth]{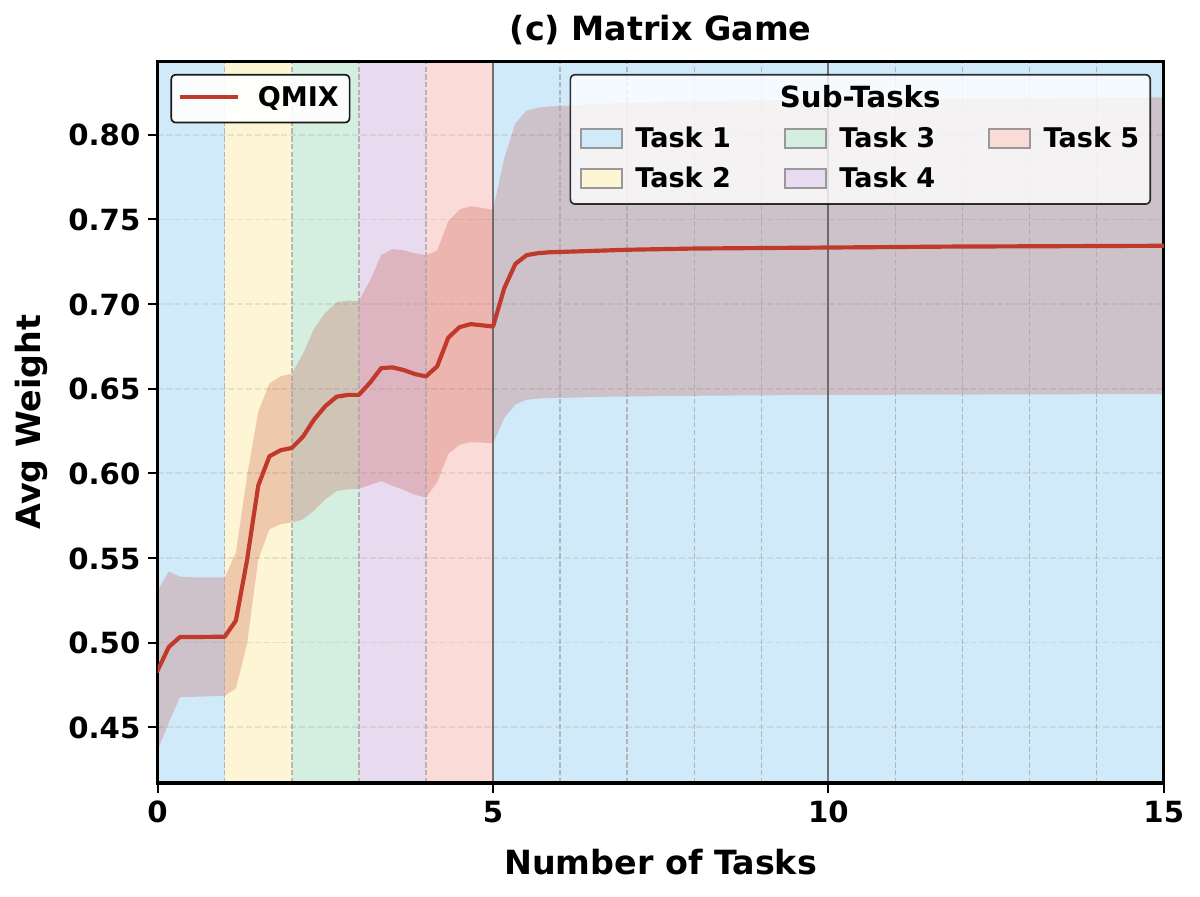} 		
  %       \includegraphics[width=0.49\columnwidth]{ICML26/experiment/stagnant/[1.4]qmix_multi_matrix_game_matrix_bias_task0.pdf}
		% \includegraphics[width=0.49\columnwidth]{ICML26/experiment/stagnant/[1.5]multi_matrix_game_matrix_bias_task1.pdf} 	
        % \includegraphics[width=0.49\columnwidth]{ICML26/experiment/stagnant/[1.3]multi_matrix_game_matrix_bias_task1.pdf}
%	\end{minipage}
        \caption{Two agents sequentially learn 5 different tasks, and then they go back to Task 1:
        (a) the matrix approximation bias (b) the average gradient norm (c) the average weight norm.}
  %        \caption{Two agents learns a sequence of 4 new tasks periodically. (a) Four one-step matrix games, the matrix approximation bias under all tasks across cycles for (b) QMIX and (c) QPLEX (d) Bias for QMIX under Task 1 across cycles.}
	\label{app:stagnant:matrix}
	\vspace{-0.3cm}
\end{figure}

\begin{figure*}[!t]
	\begin{minipage}[!t]{\linewidth} % 第一行画图结果
		\centering
        
    \includegraphics[width=0.80\columnwidth]{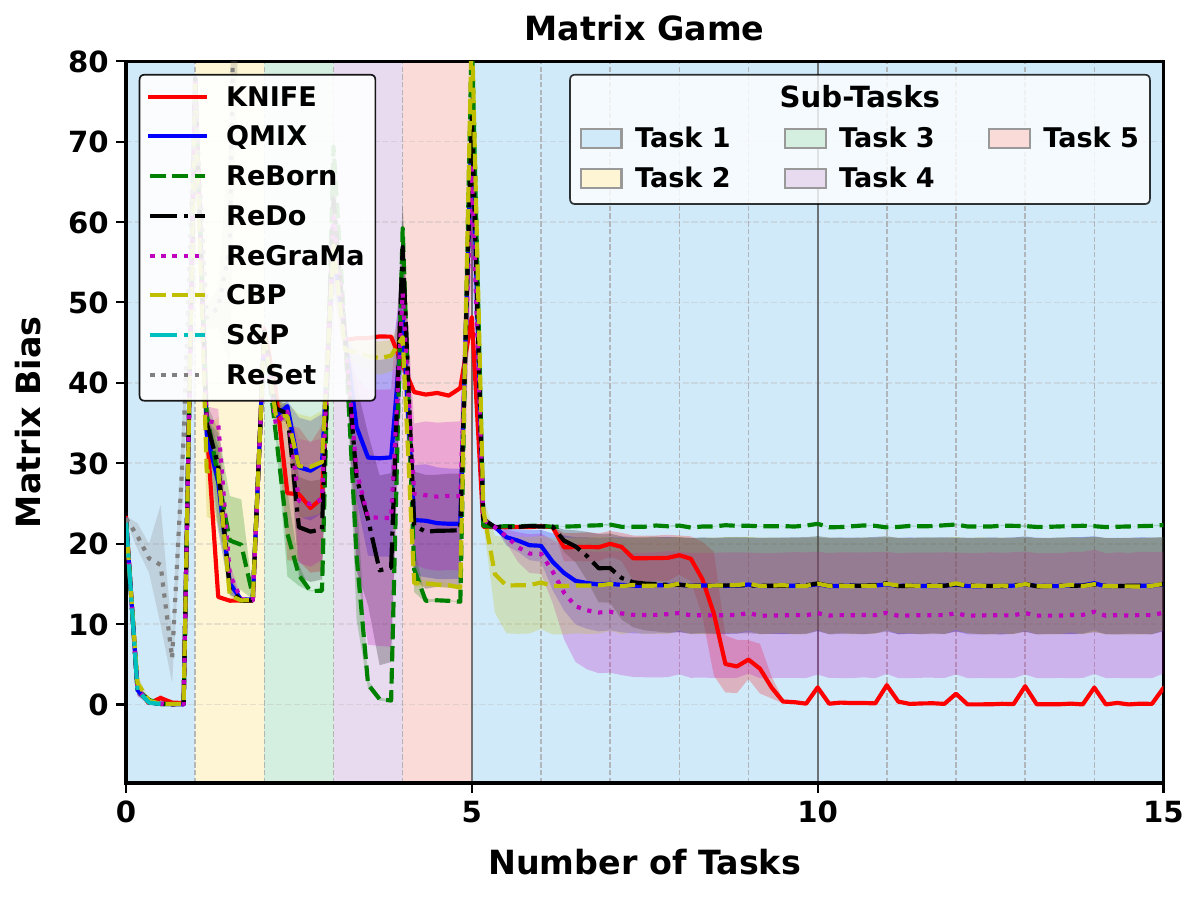} 

	\end{minipage}
            % \caption{(a) The bias for one-step payoff matrix, (b) return for Predator Prey, and (c) win rate for 3s\_vs\_5z of SMAC}
            \caption{KNIFE performs better than other plasticity injection methods in the one-step matrix game sequential tasks } 
    \label{app:fig:matrix123451}
\end{figure*}   

\begin{figure*}[!t]
	\begin{minipage}[!t]{\linewidth} % 第一行画图结果
		\centering
    
    \includegraphics[width=0.49\columnwidth]{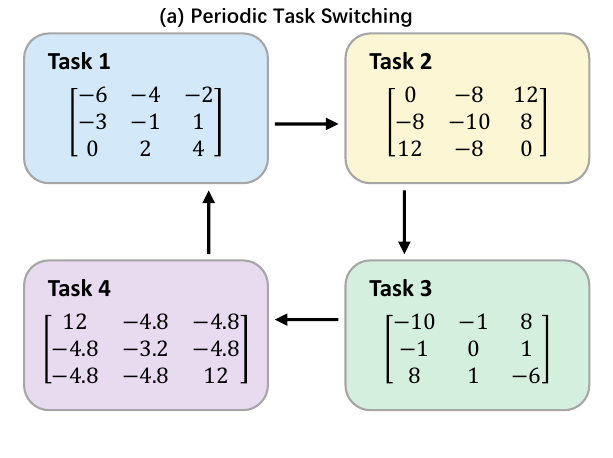} 
    \includegraphics[width=0.49\columnwidth]{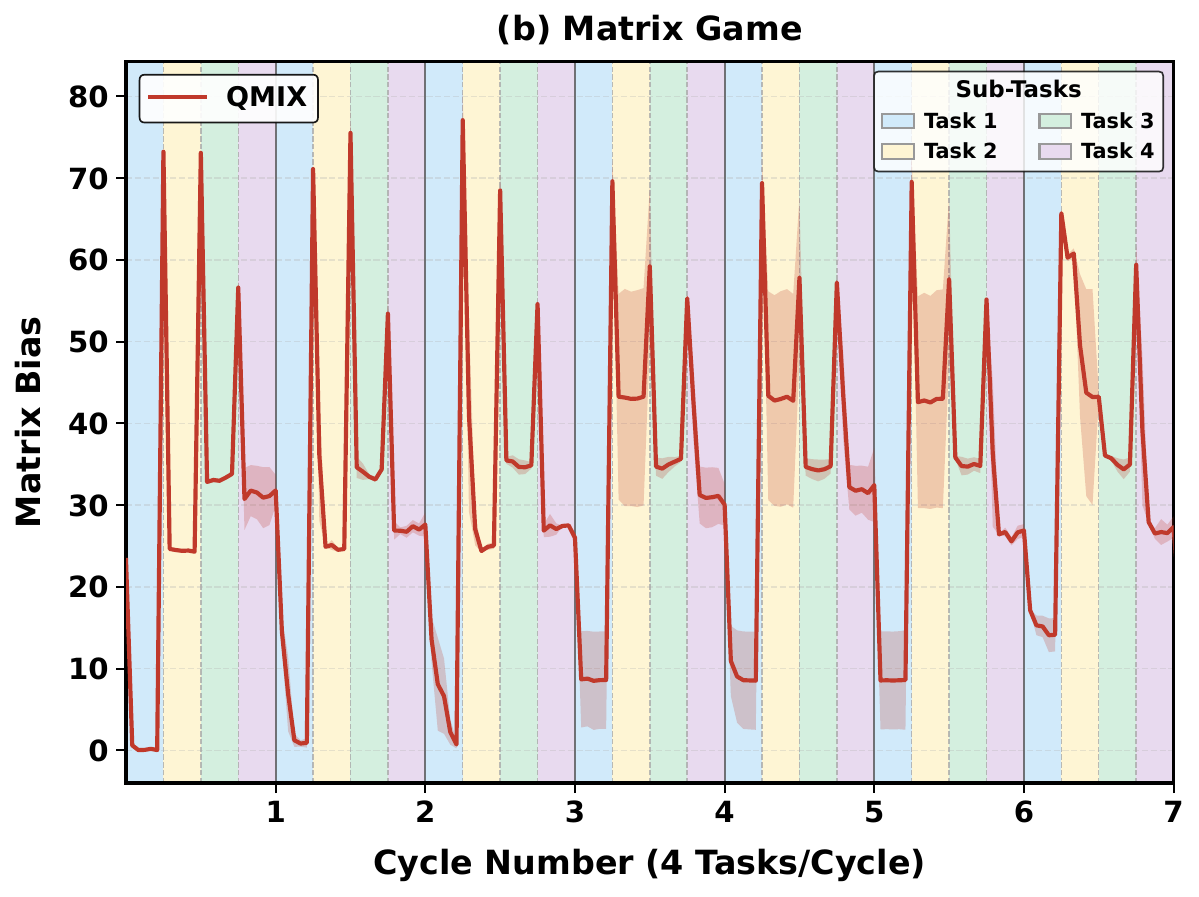} 

	\end{minipage}
            % \caption{(a) The bias for one-step payoff matrix, (b) return for Predator Prey, and (c) win rate for 3s\_vs\_5z of SMAC}
            \caption{The one-step matrix game Cyclic Setting) } 
    \label{app:matrix_periodic_intro}
\end{figure*}

\begin{figure*}[!t]
	\begin{minipage}[!t]{\linewidth} % 第一行画图结果
		\centering

    \includegraphics[width=0.8\columnwidth]{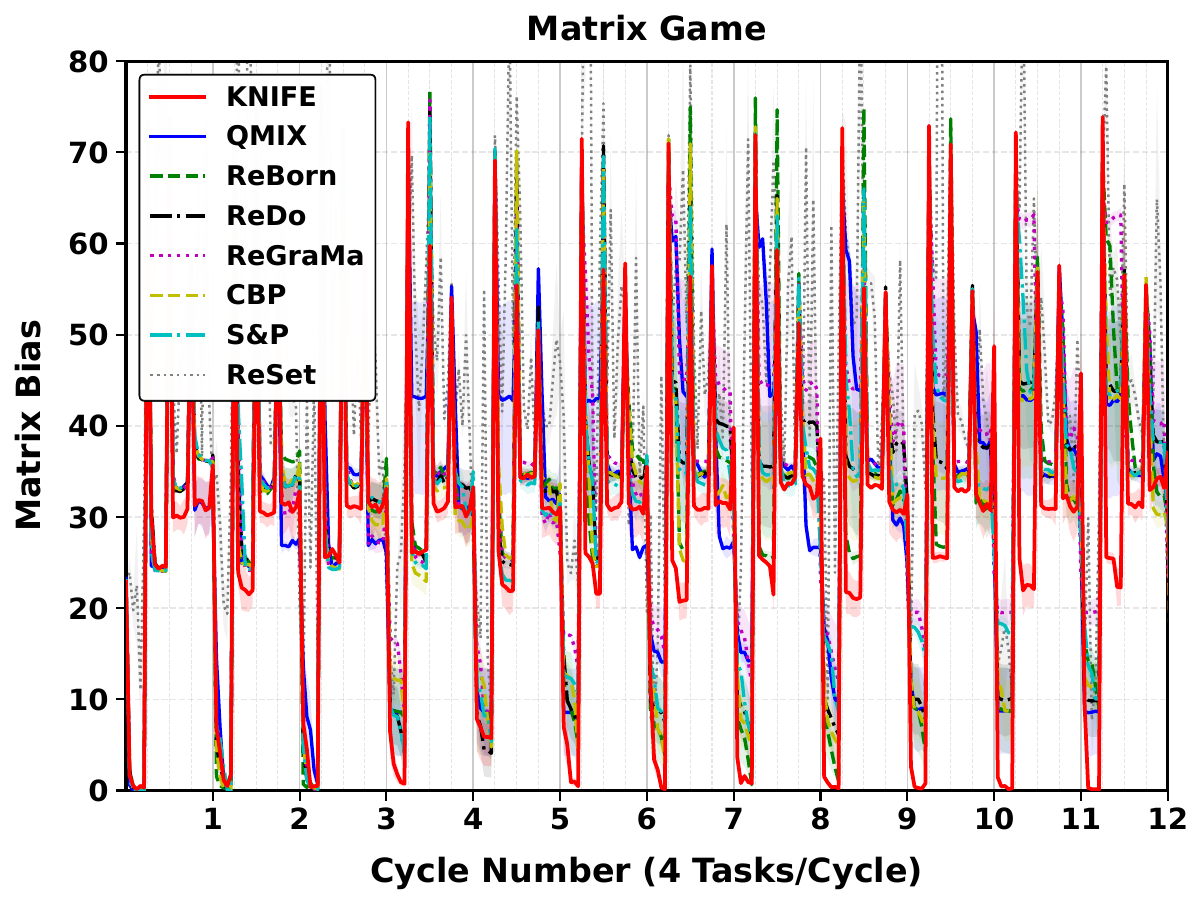} 

	\end{minipage}
            % \caption{(a) The bias for one-step payoff matrix, (b) return for Predator Prey, and (c) win rate for 3s\_vs\_5z of SMAC}
            \caption{KNIFE performs better than other plasticity injection methods in the one-step matrix game cyclic setting)} 
    \label{app:fig:matrix_periodic}
\end{figure*}   

\begin{figure*}[!t]
	\begin{minipage}[!t]{\linewidth} % 第一行画图结果
		\centering
    \includegraphics[width=0.32\columnwidth]{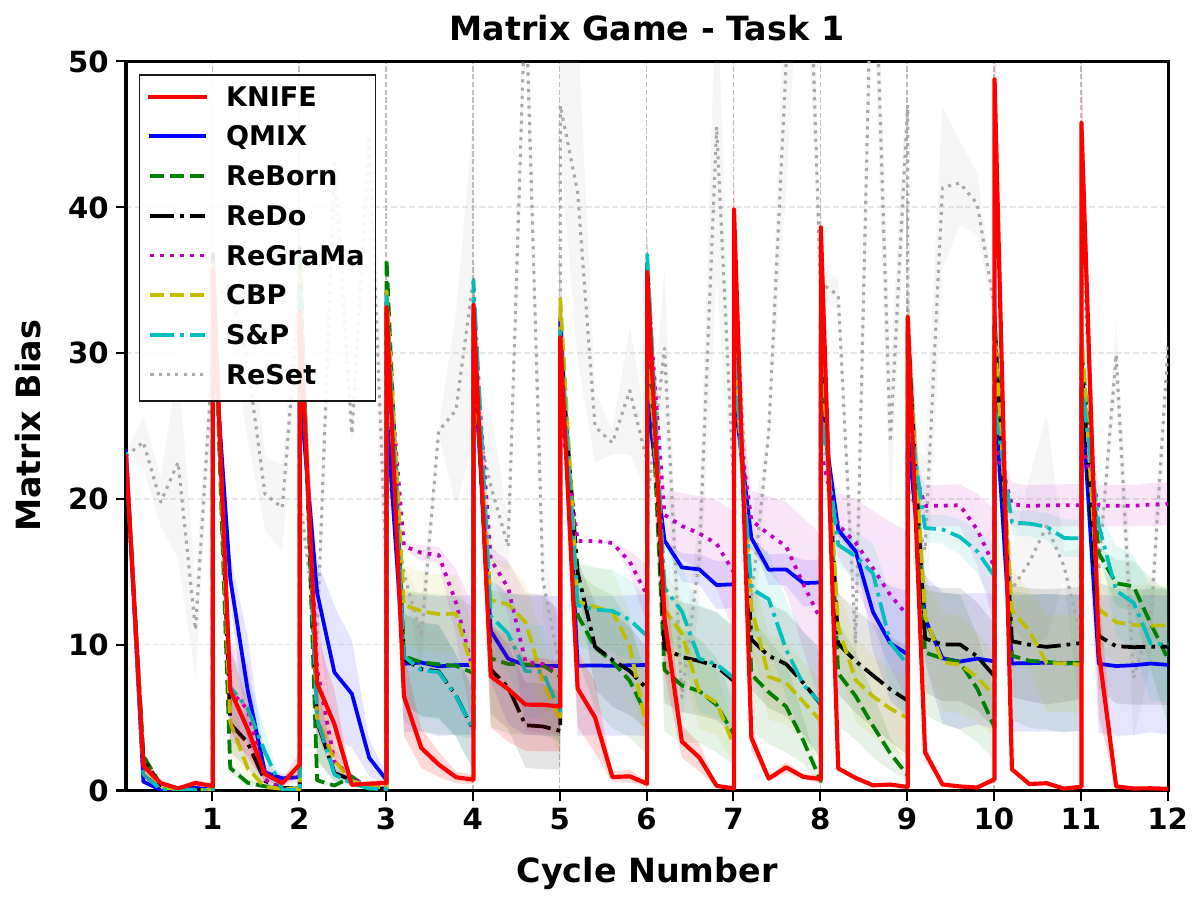}
    \includegraphics[width=0.32\columnwidth]{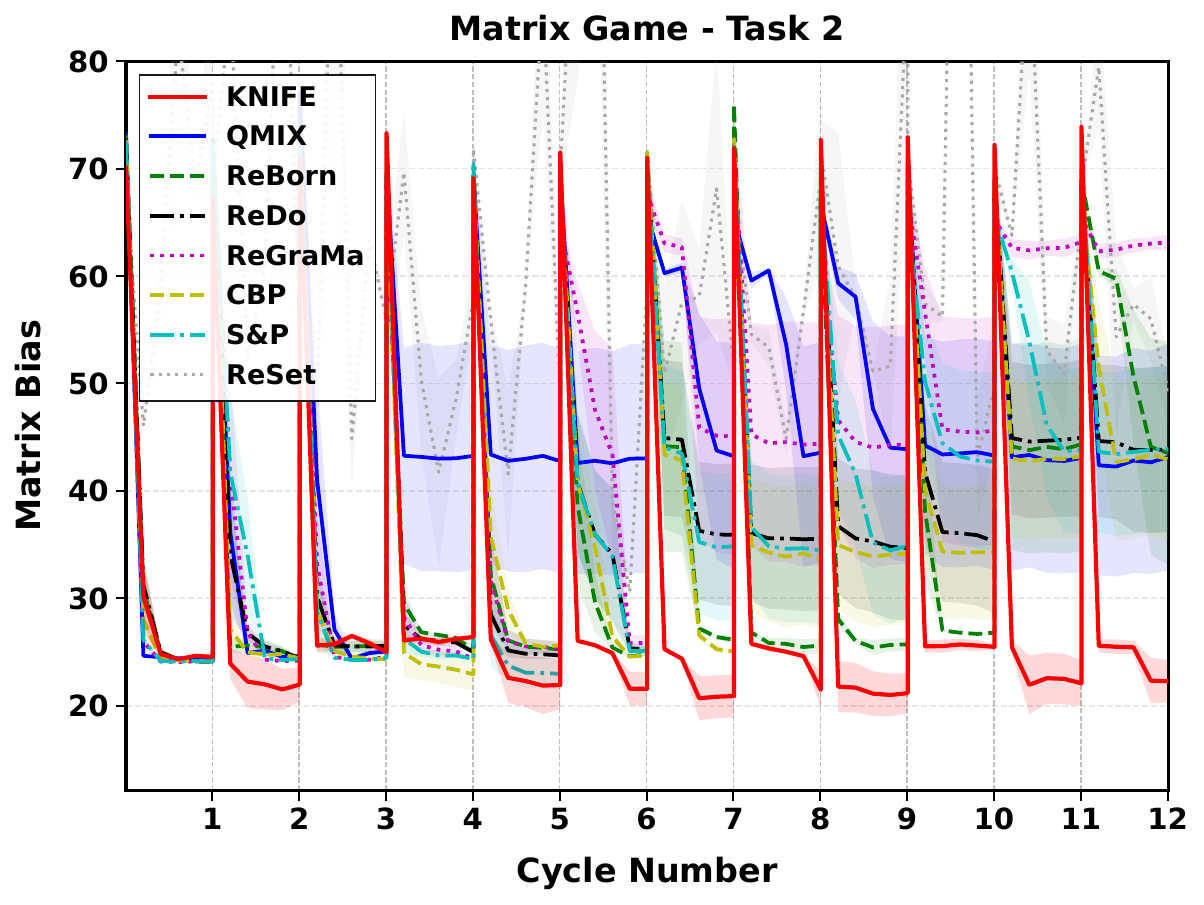} 		
	\includegraphics[width=0.32\columnwidth]{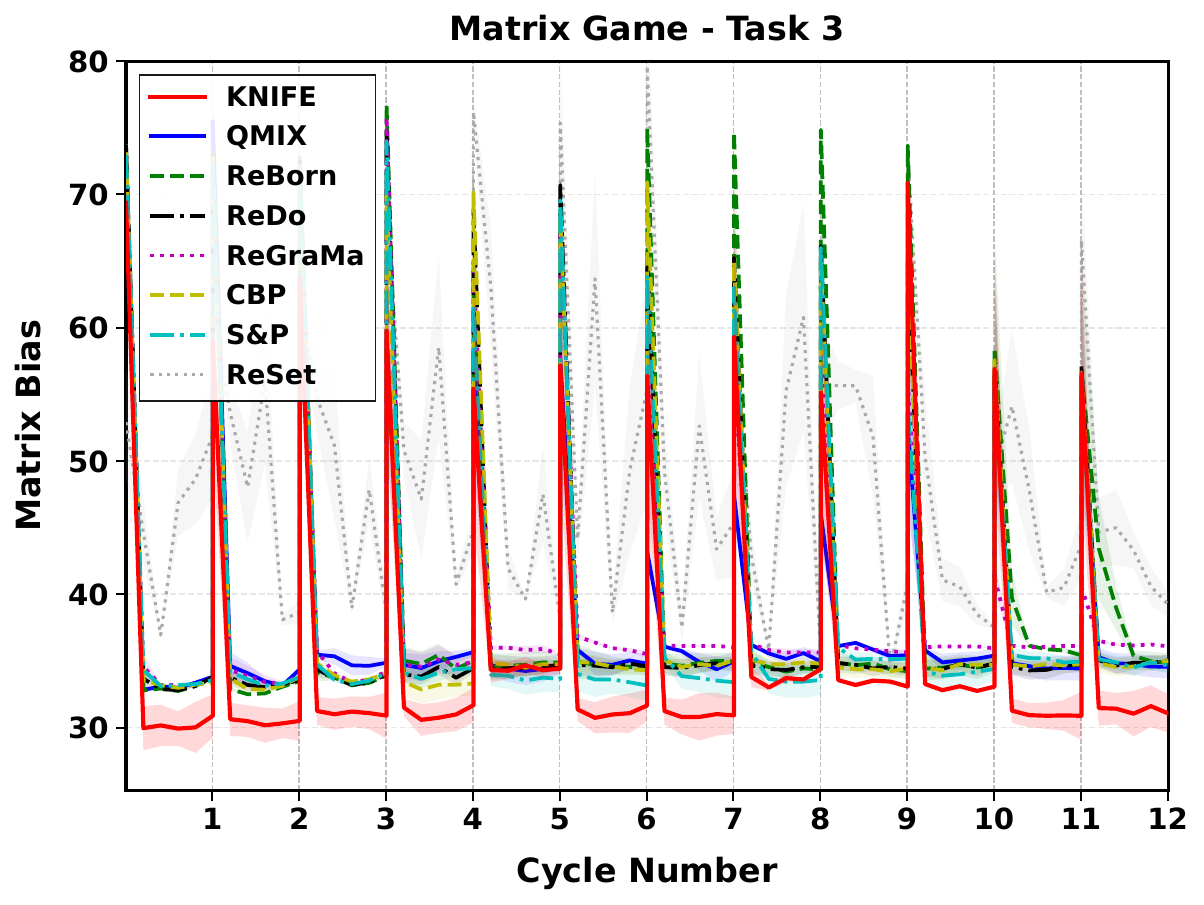} 	
	\end{minipage}
            \caption{Performance in different tasks of one-step matrix game cyclic setting.}
    \label{app:matrix_periodic_single}
    
\end{figure*}   

% \subsubsection{KNIFE Can Improve the Performance of Various
% Value Factorization Algorithms}

\begin{figure*}[!t]
	\begin{minipage}[!t]{\linewidth} % 第一行画图结果
		\centering
    \includegraphics[width=0.32\columnwidth]{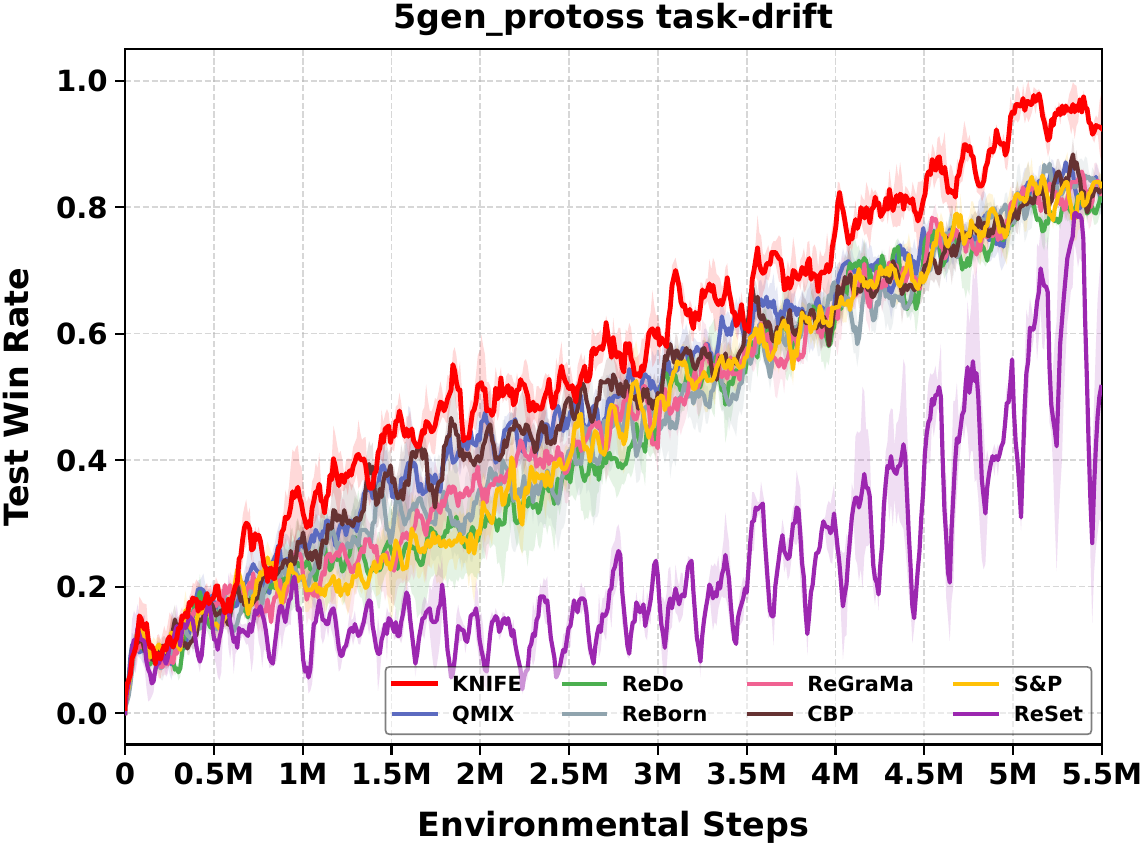}
    \includegraphics[width=0.32\columnwidth]{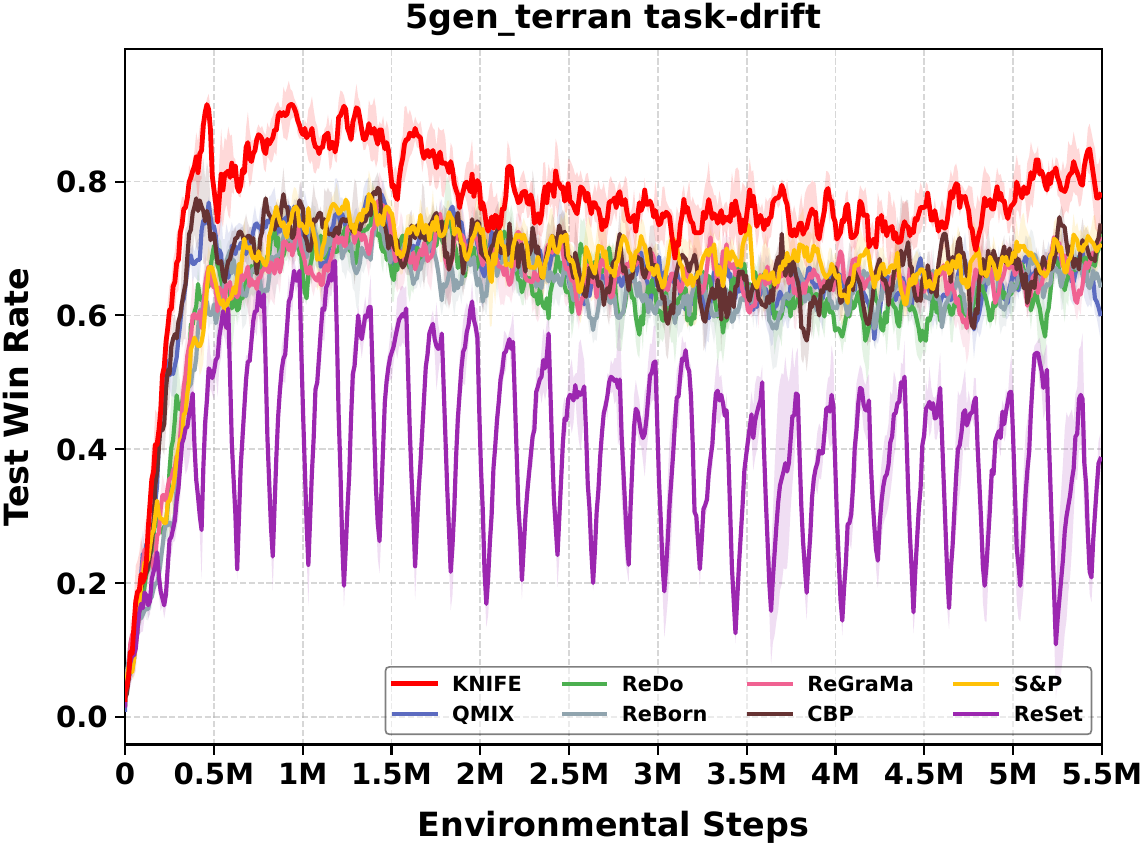} 		
	\includegraphics[width=0.32\columnwidth]{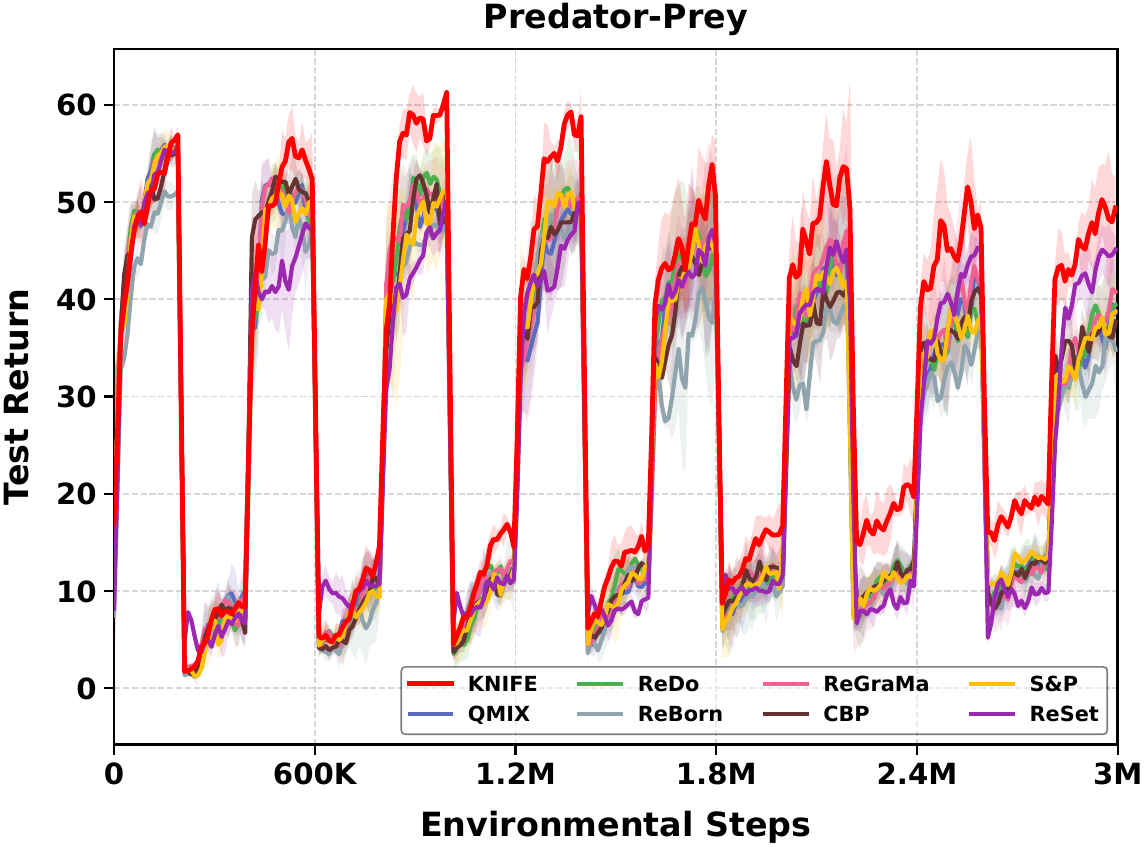} 	
	\end{minipage}
    \begin{minipage}[!t]{\linewidth} % 第一行画图结果
		\centering
    \includegraphics[width=0.32\columnwidth]{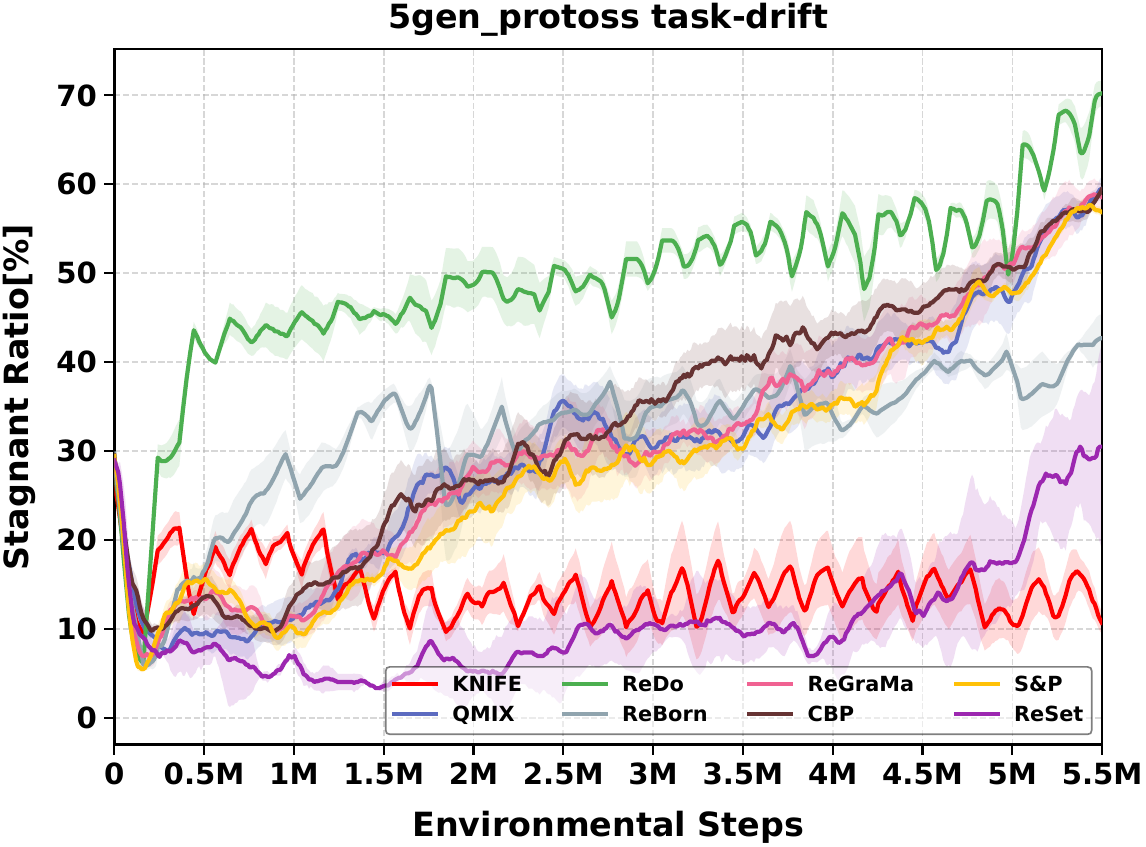}
    \includegraphics[width=0.32\columnwidth]{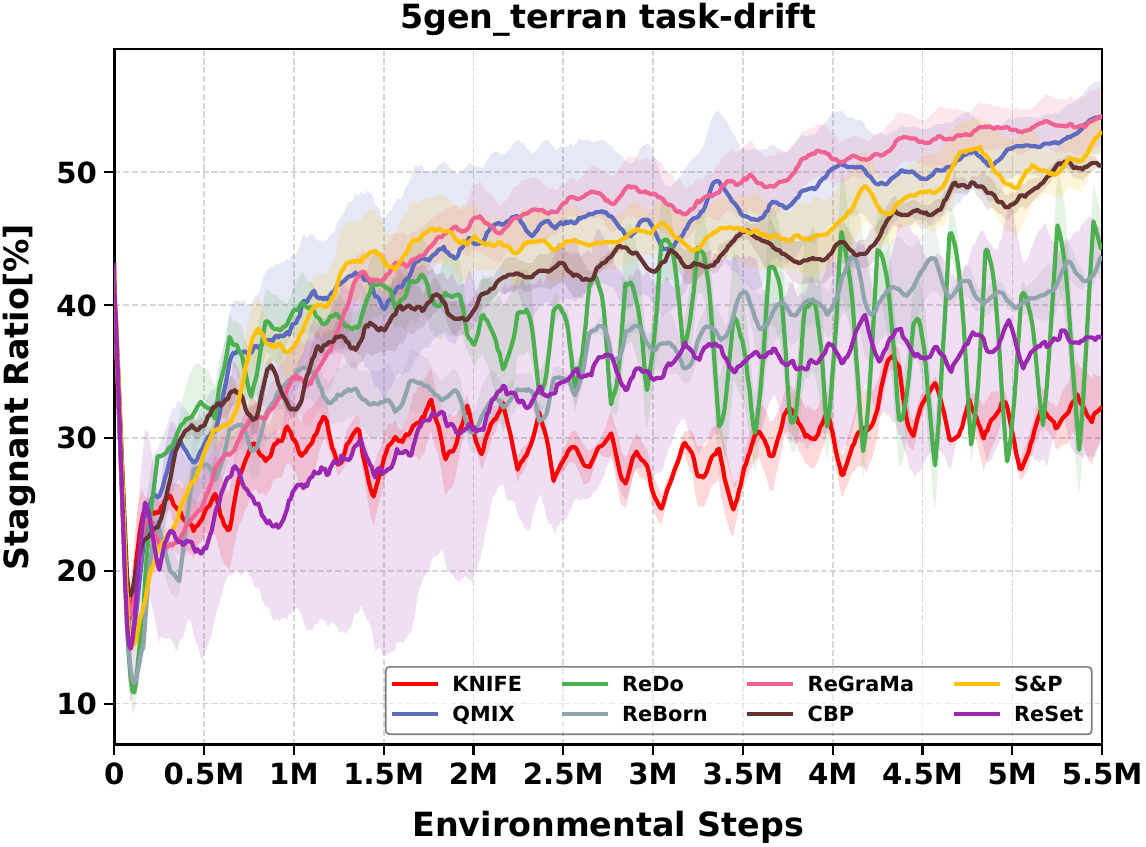} 		
	\includegraphics[width=0.32\columnwidth]{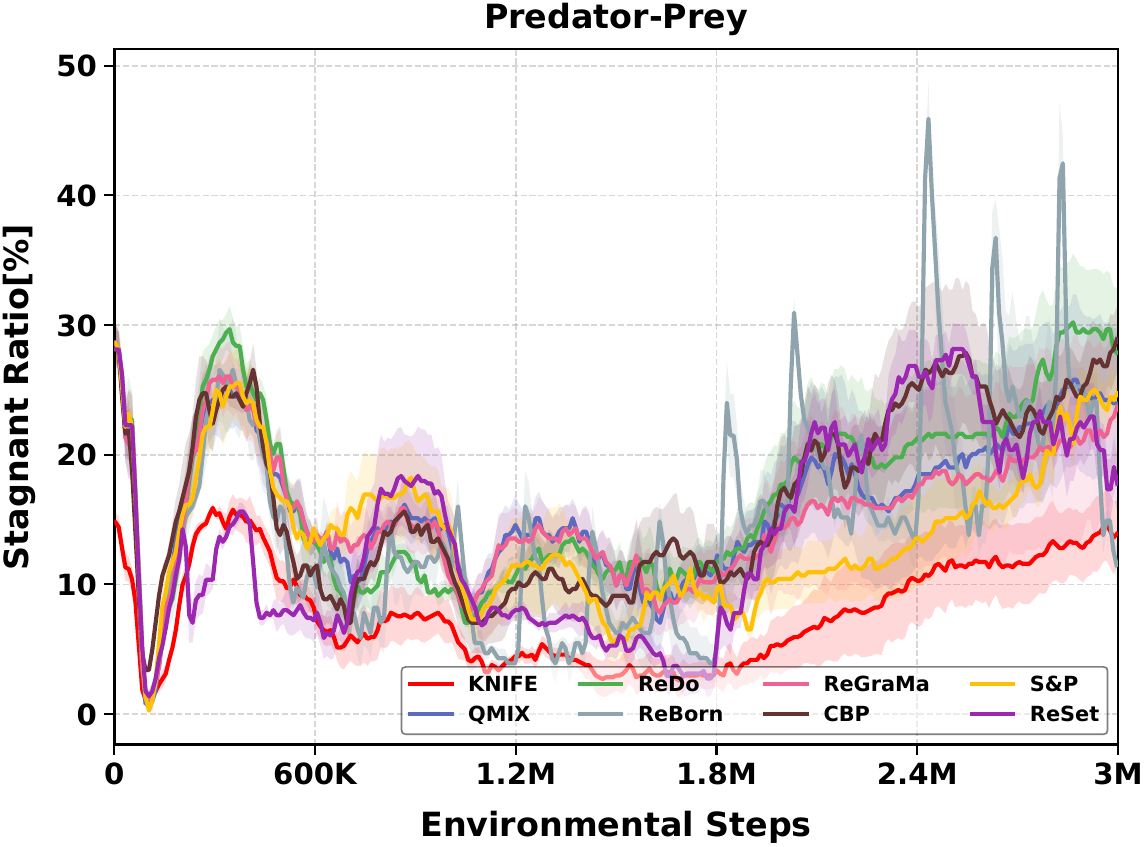} 	
	\end{minipage}
            \caption{KNIFE performs better than other plasticity injection methods for MARL tasks}
    \label{app:bench1}
\end{figure*}

%%%%%smacv2
\begin{figure*}[!t]
	\centering
	\begin{minipage}[b]{\linewidth} % 第一行画图结果
		\centering
		\includegraphics[width=0.32\columnwidth]{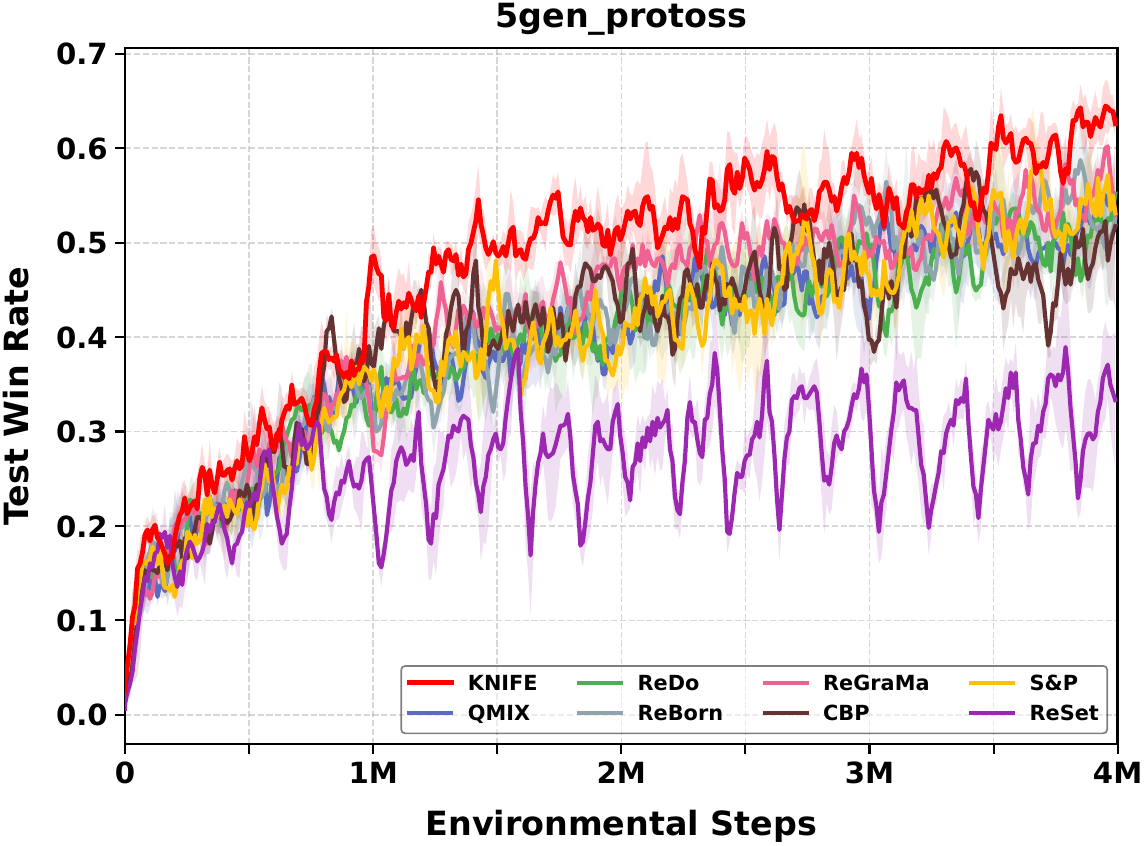} 		
		\includegraphics[width=0.32\columnwidth]{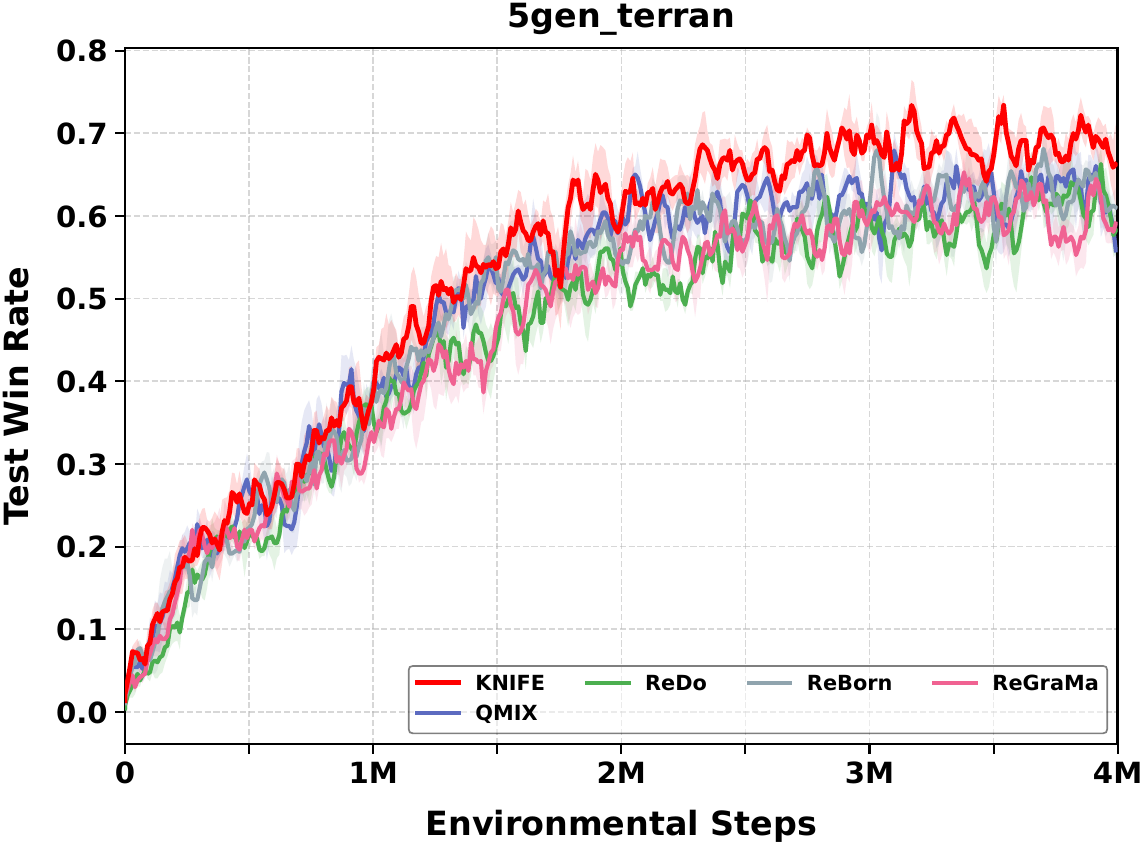} 	
		\includegraphics[width=0.32\columnwidth]{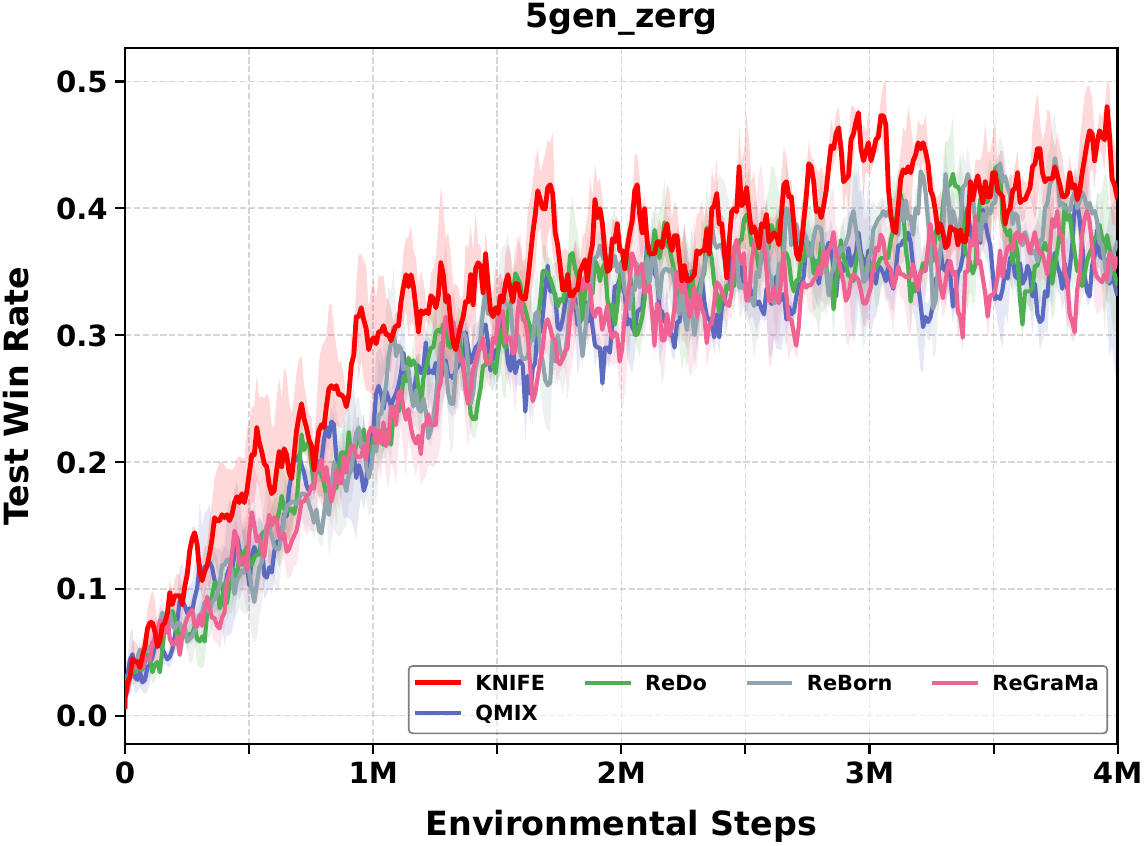}
	\end{minipage}

	\begin{minipage}[b]{\linewidth} % 第一行画图结果
		\centering
		\includegraphics[width=0.32\columnwidth]{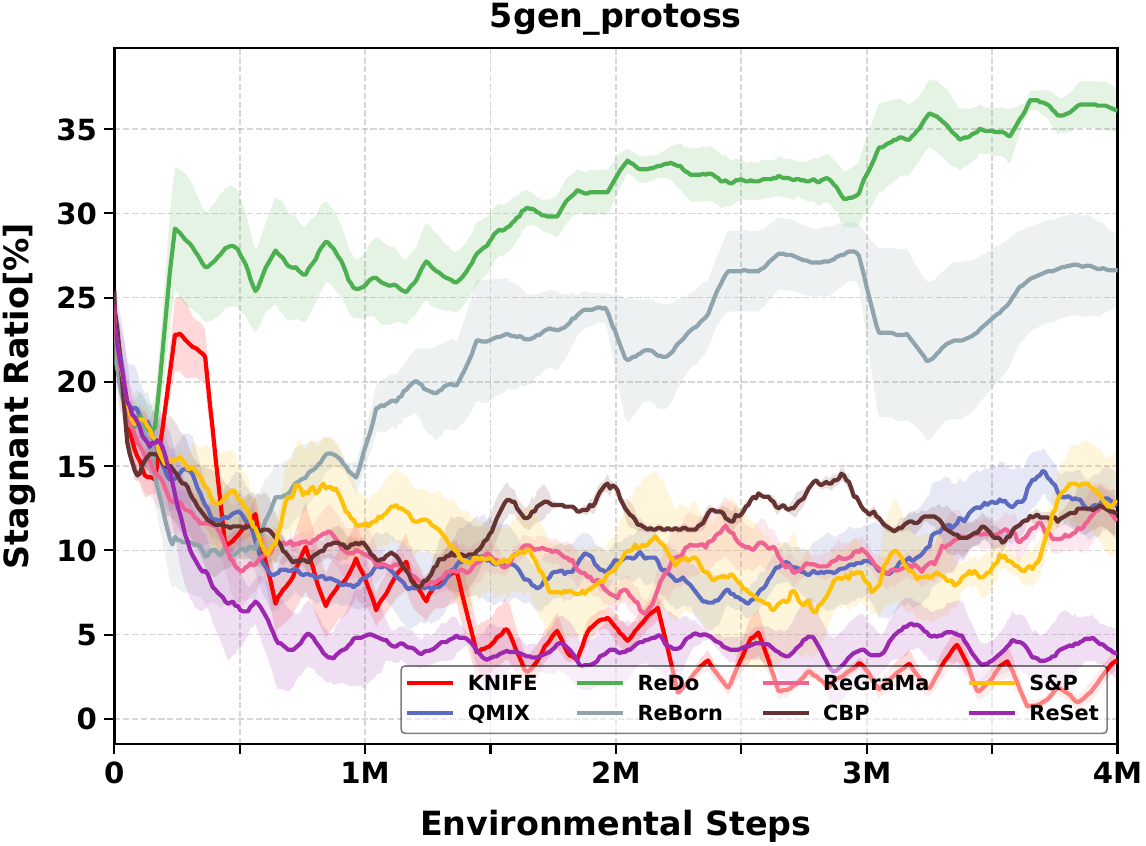} 		
		\includegraphics[width=0.32\columnwidth]{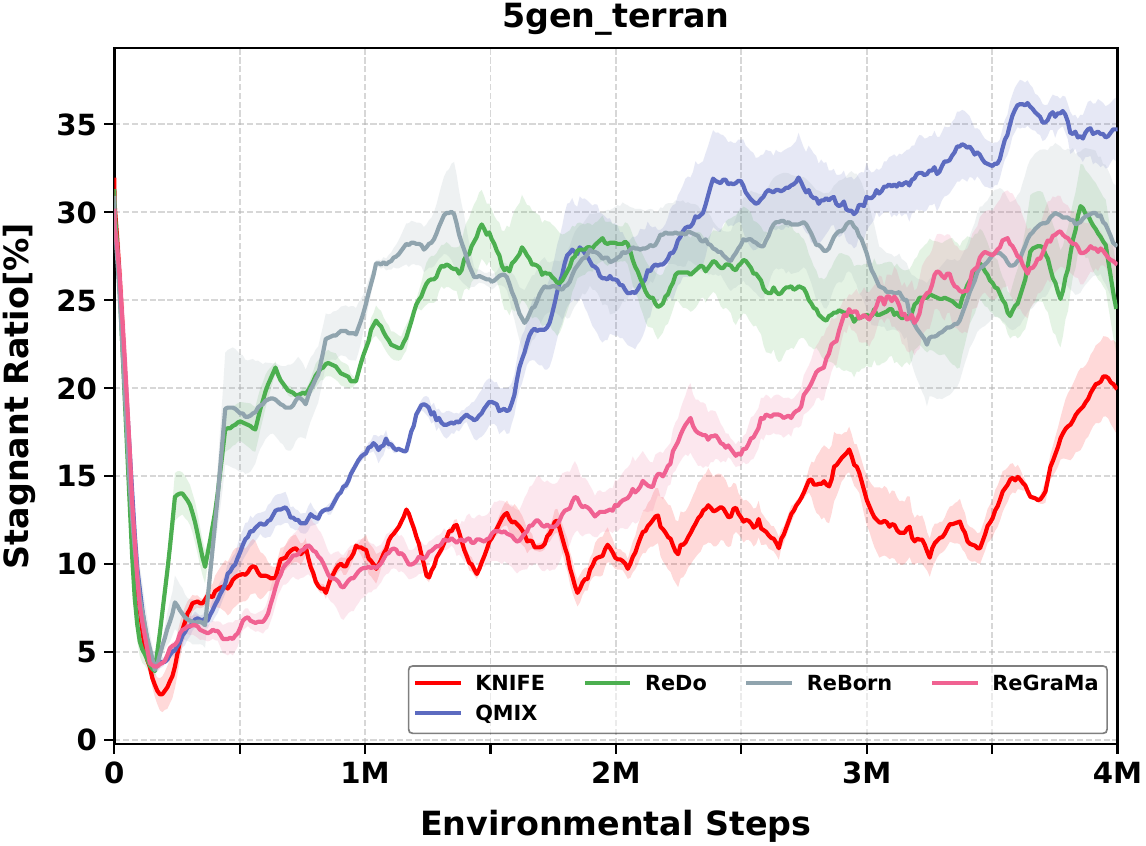} 	
		\includegraphics[width=0.32\columnwidth]{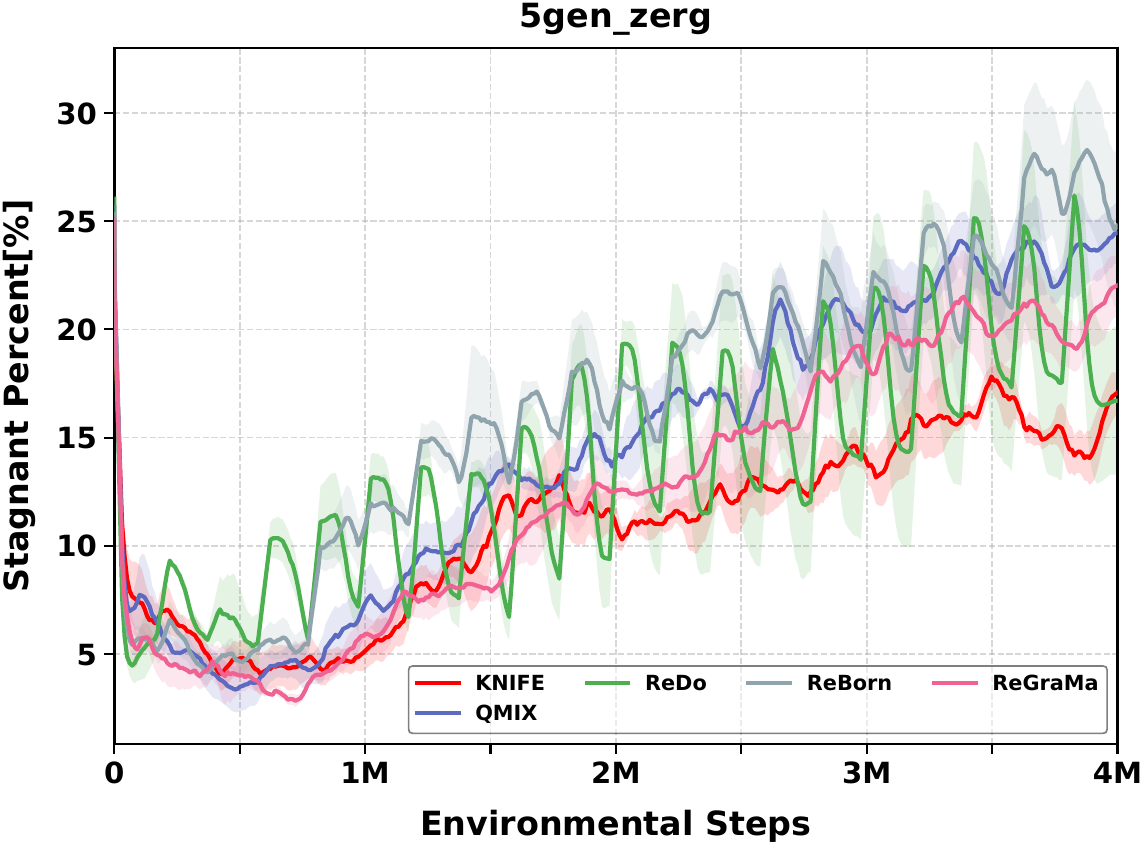}
	\end{minipage}
 \caption{KNIFE performs better than other Neuron-level plasticity injection method in standard SMACv2}
 %    	\begin{minipage}[b]{\linewidth} % 第一行画图结果
	% 	\centering
	% 	\includegraphics[width=0.32\columnwidth]{ICML25/experiment/smac/5m_vs_6m.pdf} 		
	% 	\includegraphics[width=0.32\columnwidth]{ICML25/experiment/smac/MMM2.pdf} 	
	% 	\includegraphics[width=0.32\columnwidth]{ICML25/experiment/smac/1c3s8z_vs_1c3s9z.pdf}
	% \end{minipage}
    
       % \caption{KNIFE performs better than other PS methods: (a-c) The test win rate and the futile neuron percentage (d-f) of the 2c\_vs\_64zg, 27m\_vs\_30m, 3s5z\_vs\_3s6z environment. (g-i) The test win rate of the 5m\_vs\_6m, MMM2, and 1c3s8z\_vs\_1c3s9z environments.}
	\label{app:bench2}
\end{figure*}

\subsection{Generality of KNIFE Across Diverse MARL Architectures}
\label{sec:appendix_generality}

While the primary evaluations in the main text focus on value factorization methods (e.g., QMIX and QPLEX), the underlying mechanism of KNIFE is inherently architecture-agnostic. The accumulation of stagnant neurons is a fundamental neural network pathology that emerges during sequential task transfers, and it is not strictly limited to hypernetwork-generated mixing weights.

To demonstrate the broad applicability of our method, we seamlessly integrated KNIFE into non-value-decomposition architectures, specifically a multi-agent proximal policy optimization method (MAPPO), a multi-agent deep deterministic policy gradient method (MADDPG), and a communication-centric actor-critic method (DGN). For these diverse architectures, KNIFE continuously monitors and intervenes on the stagnant neurons within the fully connected layers of their respective actor and critic networks.

Table \ref{tab:generality_marl} summarizes the performance across these different algorithms under the sequential task transfer setting. The results clearly indicate that equipping these baselines with KNIFE consistently yields substantial performance improvements. Furthermore, KNIFE significantly outperforms other neuron-level reset interventions, such as ReDo and ReBorn. This empirical evidence confirms that KNIFE serves as a versatile, plug-and-play plasticity injection module for the broader continual MARL community.

\begin{table}[htbp]
\centering
\caption{Performance (Win Rate \%) of KNIFE integrated into diverse MARL architectures under sequential task transfers.}
\label{tab:generality_marl}
\renewcommand{\arraystretch}{1.2}
\resizebox{\linewidth}{!}{
\begin{tabular}{lcccc}
\toprule
\textbf{Baseline} & \textbf{Baseline Performance} & \textbf{+ ReDo} & \textbf{+ ReBorn} & \textbf{+ KNIFE (Ours)} \\
\midrule
\textbf{MAPPO}  & $18.84 \pm 2.59$ & $20.26 \pm 4.25$ & $17.61 \pm 0.91$ & $\mathbf{25.44 \pm 2.69}$ \\
\textbf{MADDPG} & $20.13 \pm 2.66$ & $10.48 \pm 5.26$ & $29.16 \pm 2.27$ & $\mathbf{34.55 \pm 3.10}$ \\
\textbf{DGN}    & $17.48 \pm 1.66$ & -                & -                & $\mathbf{22.73 \pm 1.04}$ \\
\bottomrule
\end{tabular}
}
\end{table}

\begin{figure*}[!t]
	\begin{minipage}[!t]{\linewidth} % 第一行画图结果
		\centering
    \includegraphics[width=0.32\columnwidth]{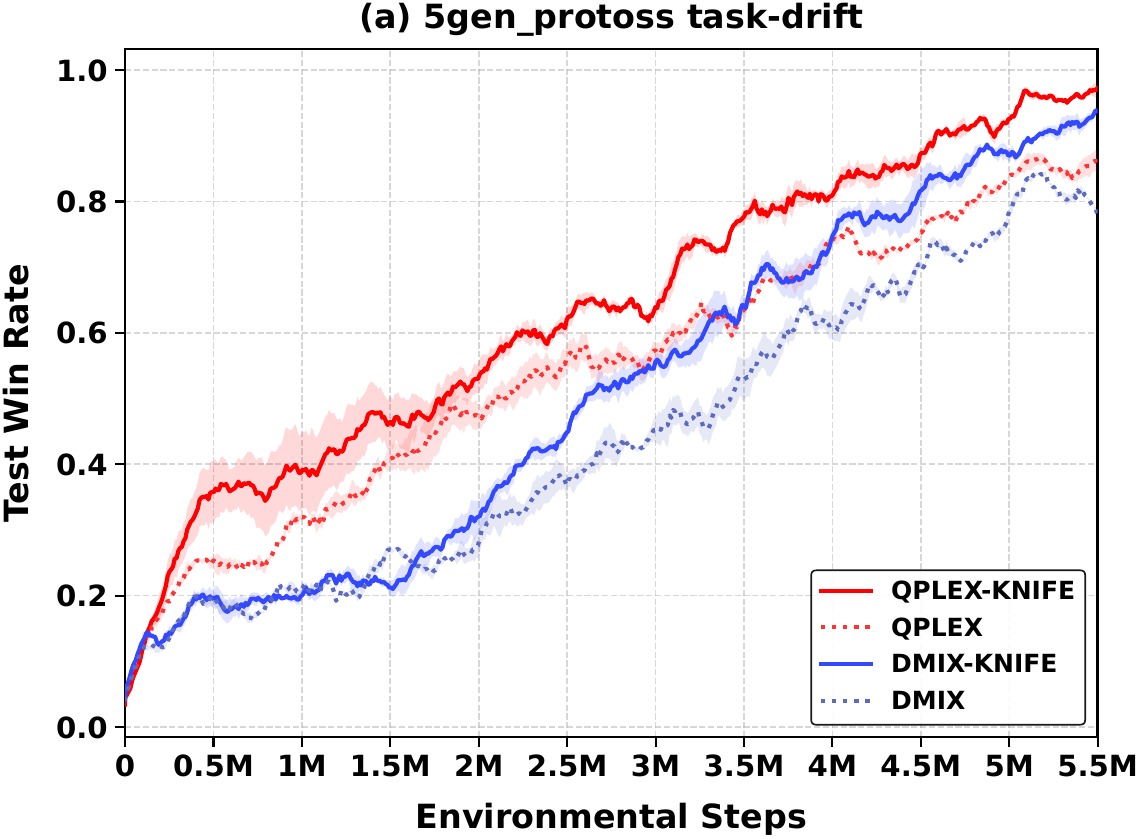}  
        \includegraphics[width=0.32\columnwidth]{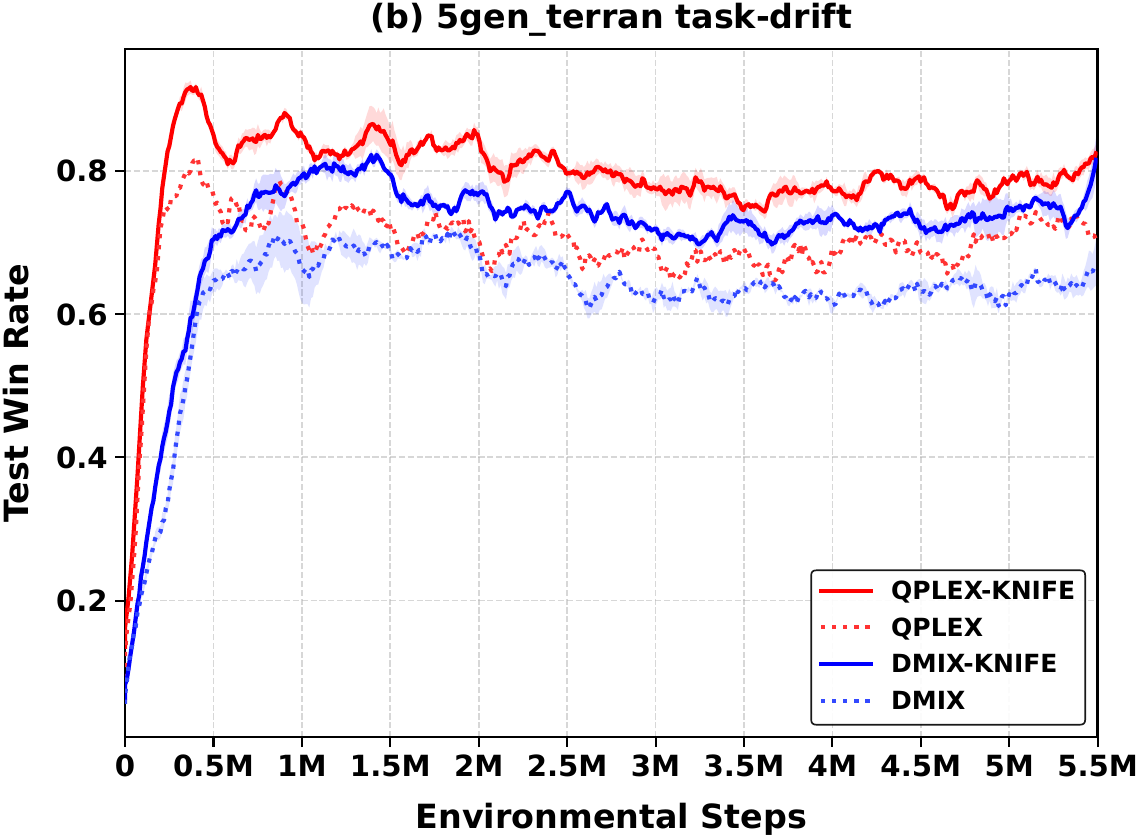}
         \includegraphics[width=0.32\columnwidth]{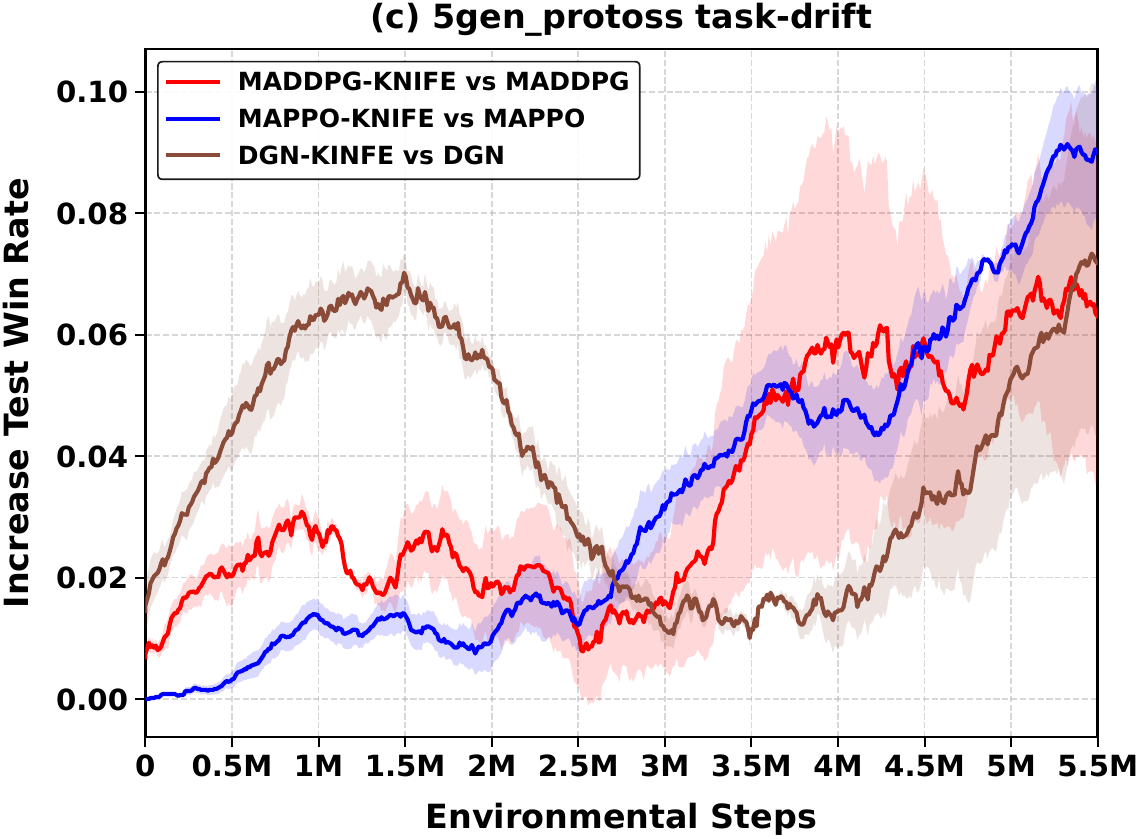} 
	\end{minipage}
            % \caption{(a) The bias for one-step payoff matrix, (b) return for Predator Prey, and (c) win rate for 3s\_vs\_5z of SMAC}
            \caption{KNIFE can reduce the plasticity loss for QPLEX and QMIX in (left) 5gen\_protoss SMACv2 (task-drift), (middle) 5gen\_terran environment SMACv2, (right) 5gen\_protoss environment SMACv2 (task-switch), for MADDPG, MAPPO and DGN } 
    \label{exp:compare:qplexdmix}
\end{figure*}

\subsection{De-hypernetworked Ablation: Structural vs. Training Dynamics}
\label{sec:appendix_dehypernetwork}

In value factorization methods like QMIX, mixing weights are typically generated by hypernetworks. To investigate whether the emergence of stagnant neurons is merely an architectural artifact of hypernetworks or a fundamental issue driven by training dynamics, we conducted a ``de-hypernetworked'' ablation. 

Specifically, in the \textit{5gen\_protoss task-switch} environment, we replaced the hypernetwork with alternative architectures: a Multi-Layer Perceptron (MLP) and a Convolutional Neural Network (CNN). To deeply analyze the underlying excessive optimization inertia, we tracked not only the stagnation ratio but also the Weight Spectra Entropy (measuring the expressivity and rank of the weight matrices) and the Gradient Norm (reflecting the active optimization capacity) across the training process.

\begin{table}[htbp]
\centering
\caption{Analysis of Stagnant Ratio, Weight Spectra Entropy, and Gradient Norm across different mixing network architectures during training.}
\label{tab:dehypernetwork}
\renewcommand{\arraystretch}{1.2}
\resizebox{\linewidth}{!}{
\begin{tabular}{llcccccc}
\toprule
\textbf{Metric} & \textbf{Architecture} & \textbf{0.5M step} & \textbf{1M step} & \textbf{1.5M step} & \textbf{2M step} & \textbf{2.5M step} & \textbf{3M step} \\
\midrule
\multirow{3}{*}{\textbf{Stagnant Ratio (\%)}} 
& \textbf{Hypernetwork (Default)} & 13.2 & 31.4 & 27.1 & 32.8 & 35.1 & 44.9 \\
& \textbf{MLP} & 26.6 & 31.4 & 29.8 & 34.0 & 30.7 & 34.6 \\
& \textbf{CNN} & 53.3 & 62.1 & 61.3 & 61.5 & 58.5 & 59.7 \\
\midrule
\multirow{3}{*}{\textbf{Weight Spectra Entropy}} 
& \textbf{Hypernetwork (Default)} & 4.006 & 3.968 & 3.953 & 3.937 & 3.938 & 3.933 \\
& \textbf{MLP} & 3.941 & 3.911 & 3.882 & 3.871 & 3.859 & 3.850 \\
& \textbf{CNN} & 3.893 & 3.862 & 3.846 & 3.829 & 3.823 & 3.813 \\
\midrule
\multirow{3}{*}{\textbf{Gradient Norm}} 
& \textbf{Hypernetwork (Default)} & 0.0291 & 0.0353 & 0.0307 & 0.0335 & 0.0242 & 0.0285 \\
& \textbf{MLP} & 0.0418 & 0.0310 & 0.0264 & 0.0227 & 0.0197 & 0.0207 \\
& \textbf{CNN} & 0.0157 & 0.0215 & 0.0132 & 0.0141 & 0.0102 & 0.0107 \\
\bottomrule
\end{tabular}
}
\end{table}

As shown in Table \ref{tab:dehypernetwork}, stagnant neurons persist across all architectures. Notably, replacing the hypernetwork with a CNN or MLP does not eliminate the pathology; in fact, the CNN structure exhibits an even higher stagnation ratio. Furthermore, across all architectures, the Weight Spectra Entropy steadily decreases (indicating a systemic loss of feature diversity and representational expressivity), while the Gradient Norm diminishes progressively over time. This synchronous decay in entropy and gradient flow firmly confirms that the accumulation of stagnant neurons is a fundamental, architecture-agnostic optimization inertia issue in continual MARL, rather than a structural artifact unique to hypernetworks.
\subsection{Initialization Strategies: Methodological Comparison and Ablation}
\label{sec:appendix_init}

The effectiveness of plasticity injection depends heavily on how the newly introduced parameters are integrated with the existing frozen network. KNIFE and PI fundamentally differ in their initialization logic.

\textbf{PI: State-Agnostic Kaiming Initialization.} PI typically employs standard Kaiming initialization for both weights $W$ and biases $b$ of the injected layers. The initialization follows:
\begin{equation}
W \sim \mathcal{N}\left(0, \frac{2}{d_{in}}\right), \quad b \sim \mathcal{U}\left(-\frac{6}{\sqrt{d_{in}}}, \frac{6}{\sqrt{d_{in}}}\right)
\end{equation}
where $d_{in}$ is the input dimension. This approach is entirely state-agnostic, ignoring the current optimization state of the pre-existing layers. This carries a significant risk that the newly injected neurons will immediately fall into activation dead zones (e.g., dying ReLUs) or fail to propagate meaningful signals during the early stages of transfer learning.

\textbf{KNIFE: State-Aware Initialization.} In contrast, KNIFE applies a state-aware strategy specifically designed to bridge the gradient flow for stagnant neurons:
\begin{itemize}
    \item \textbf{Positive Input Bias:} We explicitly initialize the input bias $b_{in}^l$ to a positive value: $b_{in}^l = 0.5 \sigma_{in}$, where $\sigma_{in}$ is the standard deviation of its input weights. This increases the likelihood of positive activations, ensuring the injected neurons remain active and can receive backpropagated gradients immediately.
    \item \textbf{Contextual Output Weight:} The output weights $w_{out}^l$ are initialized as $w_{out}^l \sim \mathcal{N}(0, \sigma_{out})$, where $\sigma_{out}$ is tied to the weights of the subsequent layer: $\sigma_{out} = \max(std(w^{l+1}), 10^{-2})$. This ensures that the forward signal is fully compatible with the existing network distribution.
\end{itemize}

\textbf{Ablation Study.} To empirically validate the necessity of this state-aware design, we evaluated an ablated version, \textit{KNIFE (Kaiming init)}, which uses standard Kaiming initialization instead of our proposed approach. We compare its performance against the full KNIFE algorithm, PI, and ReDo across different sequential task transfer environments.

\begin{table}[htbp]
\centering
\caption{Ablation study of initialization strategies (Win Rate \%).}
\label{tab:init_ablation}
\renewcommand{\arraystretch}{1.2}
\resizebox{\linewidth}{!}{
\begin{tabular}{lccccc}
\toprule
\textbf{Environment} & \textbf{KNIFE (Ours)} & \textbf{KNIFE (Kaiming init)} & \textbf{PI @ 25\%} & \textbf{PI @ 50\%} & \textbf{ReDo} \\
\midrule
\textbf{5gen\_protoss task-switch} & $\mathbf{91.17 \pm 1.17}$ & $88.43 \pm 0.97$ & $82.55 \pm 0.33$ & $84.45 \pm 1.38$ & $81.20 \pm 1.36$ \\
\textbf{5gen\_protoss task-drift}   & $\mathbf{94.68 \pm 0.74}$ & $89.49 \pm 1.44$ & $81.48 \pm 0.87$ & $82.27 \pm 1.19$ & $79.19 \pm 0.63$ \\
\textbf{5gen\_protoss}              & $\mathbf{60.89 \pm 1.87}$ & $57.91 \pm 0.04$ & $53.41 \pm 1.07$ & $54.02 \pm 2.85$ & $50.33 \pm 2.49$ \\
\bottomrule
\end{tabular}
}
\end{table}

As shown in Table \ref{tab:init_ablation}, KNIFE achieves the highest overall performance and consistently outperforms the KNIFE (Kaiming init) baseline. This empirically demonstrates that our state-aware initialization actively prevents newly injected neurons from falling into activation dead zones, ensuring they successfully bridge the gradient flow from the earliest stages. 

Furthermore, the advantage of our neuron-level surgery is evident: even when reverting to the standard Kaiming initialization, KNIFE still consistently outperforms PI across all tested environments. This strongly demonstrates that KNIFE's structural innovation—specifically targeted, neuron-level surgery for stagnant neurons—is fundamentally more effective than PI's layer-level injection approach.

\subsection{Ablations on Stagnant/Dormant Neurons}
\label{sec:appendix_causal}

To pinpoint which component truly drives the performance gains and to prove that stagnant neurons are the primary cause of degradation, we conducted causal ablation studies. We developed several variants by altering the intervention targets:
\begin{itemize}
    \item \textbf{KNIFE (non-stagnant):} Applies KNIFE surgery to non-stagnant neurons, keeping stagnation constant.
    \item \textbf{KNIFE (Dormant):} Applies KNIFE specifically to dormant neurons (inactive units).
    \item \textbf{ReDo (Stagnant) \& ReDo (Dormant):} Applies the standard ReDo reset mechanism to stagnant and dormant neurons, respectively.
\end{itemize}
\begin{table}[htbp]
\centering
\caption{Performance comparison of causal ablation variants (Win Rate \%).}
\label{tab:causal_ablation}
\renewcommand{\arraystretch}{1.2}
\resizebox{\linewidth}{!}{
\begin{tabular}{lccccc}
\toprule
\textbf{Environment} & \textbf{KNIFE (Ours)} & \textbf{KNIFE (Dormant)} & \textbf{ReDo (Stagnant)} & \textbf{ReDo (Dormant)} & \textbf{KNIFE (non-stagnant)} \\
\midrule
\textbf{5gen\_protoss} & $\mathbf{60.89 \pm 1.87}$ & $55.37 \pm 3.09$ & $55.20 \pm 2.30$ & $50.33 \pm 2.49$ & $46.09 \pm 2.10$ \\
\textbf{5gen\_protoss task-switch} & $\mathbf{91.17 \pm 1.17}$ & $87.05 \pm 2.30$ & $88.01 \pm 0.63$ & $81.20 \pm 1.36$ & $77.46 \pm 0.41$ \\
\bottomrule
\end{tabular}
}
\end{table}

As demonstrated in Table \ref{tab:causal_ablation}, applying interventions specifically to stagnant neurons (both in KNIFE and ReDo) strictly outperforms targeting dormant or non-stagnant ones. Conversely, intervening on non-stagnant units yields the poorest results, indicating that modifying healthy neurons disrupts learning. This causally strengthens our claim: stagnant neurons are the primary drivers of plasticity loss, and targeting them is the key to effective adaptation.

\subsection{Computational Overhead and Memory Efficiency}
\label{sec:appendix_overhead}

To address potential concerns regarding the computational overhead and memory pressure introduced by our composite units, we comprehensively profiled the peak GPU memory consumption (in MB) and total training time (in hours) of KNIFE against both layer-level (PI) and neuron-level (ReDo, ReBorn) baselines.

As shown in Table \ref{tab:pi_perf}, PI's layer-level expansion nearly doubles the memory footprint (e.g., reaching over 630 MB in the task-switch setting). Conversely, KNIFE's targeted neuron-level surgery introduces only a minimal memory increase compared to the standard QMIX baseline. Furthermore, when compared with other neuron-level methods (Table \ref{tab:overhead}), although KNIFE requires a moderate temporary memory expansion during the injection phase, its continuous decay and tensor reuse mechanisms keep overall utilization highly efficient. Crucially, KNIFE achieves the highest win rates and the shortest total training times. This proves that our method delivers superior adaptability without imposing significant computational burdens.

\begin{table}[htbp]
\centering
\caption{Empirical Performance and Memory Overhead Comparison against Layer-Level Injection}
\label{tab:pi_perf}
\renewcommand{\arraystretch}{1.1}
\resizebox{\linewidth}{!}{
\begin{tabular}{llcccc}
\toprule
\textbf{Environment} & \textbf{Metric} & \textbf{KNIFE (Ours)} & \textbf{PI @ 25\%} & \textbf{PI @ 50\%} & \textbf{QMIX (Baseline)} \\
\midrule
\multirow{2}{*}{\textbf{5gen\_protoss task-switch}} 
& Win Rate (\%) & $91.17 \pm 1.17$ & $82.55 \pm 0.33$ & $84.45 \pm 1.38$ & $82.50 \pm 2.50$ \\
& Memory (MB)   & 330 & 654 & 630 & 293 \\
\midrule
\multirow{2}{*}{\textbf{5gen\_protoss task-drift}}  
& Win Rate (\%) & $94.68 \pm 0.74$ & $81.48 \pm 0.87$ & $82.27 \pm 1.19$ & $82.75 \pm 1.40$ \\
& Memory (MB)   & 431 & 655 & 679 & 416 \\
\midrule
\multirow{2}{*}{\textbf{5gen\_protoss}}             
& Win Rate (\%) & $60.89 \pm 1.87$ & $53.41 \pm 1.07$ & $54.02 \pm 2.85$ & $50.98 \pm 3.55$ \\
& Memory (MB)   & 176 & 239 & 230 & 166 \\
\bottomrule
\end{tabular}
}
\end{table}

\begin{table}[htbp]
\centering
\caption{Performance, Training Time, and Memory Consumption Comparison against Neuron-Level Methods}
\label{tab:overhead}
\renewcommand{\arraystretch}{1.2}
\resizebox{\linewidth}{!}{
\begin{tabular}{llccc}
\toprule
\textbf{Environment} & \textbf{Metric} & \textbf{KNIFE (Ours)} & \textbf{ReDo} & \textbf{ReBorn} \\
\midrule
\multirow{3}{*}{\textbf{5gen\_protoss}}
& Win Rate (\%) & $\mathbf{60.89 \pm 1.87}$ & $50.50 \pm 2.05$ & $53.41 \pm 2.94$ \\
& Time (h)      & $\mathbf{26.01}$ & 27.70 & 30.73 \\
& Memory (MB)   & 176.63 & $166.88$ & $164.45$ \\
\midrule
\multirow{3}{*}{\textbf{5gen\_protoss task-switch}}
& Win Rate (\%) & $\mathbf{91.17 \pm 1.17}$ & $81.20 \pm 1.36$ & $79.65 \pm 1.34$ \\
& Time (h)      & $\mathbf{25.98}$ & 29.34 & 31.10 \\
& Memory (MB)   & 330.58 & $315.36$ & $300.93$ \\
\bottomrule
\end{tabular}
}
\end{table}

\subsection{Scalability to Large-Scale}%可以在大规模上运行

To explicitly evaluate whether KNIFE can effectively transfer to larger-scale MARL settings we designed an extended sequential continual setting:
\begin{itemize}
    \item \textbf{Large-Scale Agent Scaling:} The environment progressively scales up both the number of agents and enemies from 5 to 40 ($5m \rightarrow 10m \rightarrow 20m \rightarrow 30m \rightarrow 40m$).
\end{itemize}

As shown in Table \ref{tab:large_scale}, scaling the environment up to 40 agents significantly challenges the plasticity of the baseline methods. 

\begin{table}[htbp]
\centering
\caption{Performance (Win Rate \%) on extended large-scale.}
\label{tab:large_scale}
\renewcommand{\arraystretch}{1.2}
\resizebox{\linewidth}{!}{
\begin{tabular}{lccccccc}
\toprule
\textbf{Sequential Setting} & \textbf{QMIX}  & \textbf{ReDo} & \textbf{KNIFE (Ours)} \\
\midrule
\textbf{5m $\rightarrow$ 10m $\rightarrow$ 20m $\rightarrow$ 30m $\rightarrow$ 40m} & $66.45 \pm 17.55$ & $39.06 \pm 13.43$ & $\mathbf{83.59 \pm 5.52}$ \\
\bottomrule
\end{tabular}
}
\end{table}

KNIFE successfully navigates these complex transitions. By effectively mitigating the accumulation of stagnant neurons, KNIFE maintains remarkably high win rates ($83.59\%$) and drastically outperforms all other baselines. This compelling empirical evidence confirms that KNIFE is not limited to simple or small-scale scenarios, but serves as a highly robust and scalable plasticity injection module for large-scale MARL settings.

\subsection{Parameter Sensitivity}%参数敏感性消融 还要加一个和finetune的结果比
\label{app:parameter_sensitivity}

We study the weight decay factor $\gamma \in \{0, 0.99, 0.995, 1\}$, where $\gamma$ controls how many frozen neurons are retained after each pruning stage. The result is shown in Figure~\ref{app:fig:parameters} (a). 

With $\gamma=0$, the model discards past knowledge immediately and performs worst, suggesting that abrupt forgetting harms previously learned coordination. With $\gamma=1$, the outdated knowledge in knowledge neurons is kept without adaptation to other tasks, which leads to a slight performance drop. An intermediate $\gamma$ achieves the best trade-off: it avoids sudden forgetting while gradually removing outdated knowledge, thereby allowing the injected neurons to optimally improve adaptation.

\begin{figure*}[!t]
	\centering
    \begin{minipage}[!t]{\linewidth} % 第一行画图结果
		\centering
        \includegraphics[width=0.24\columnwidth]{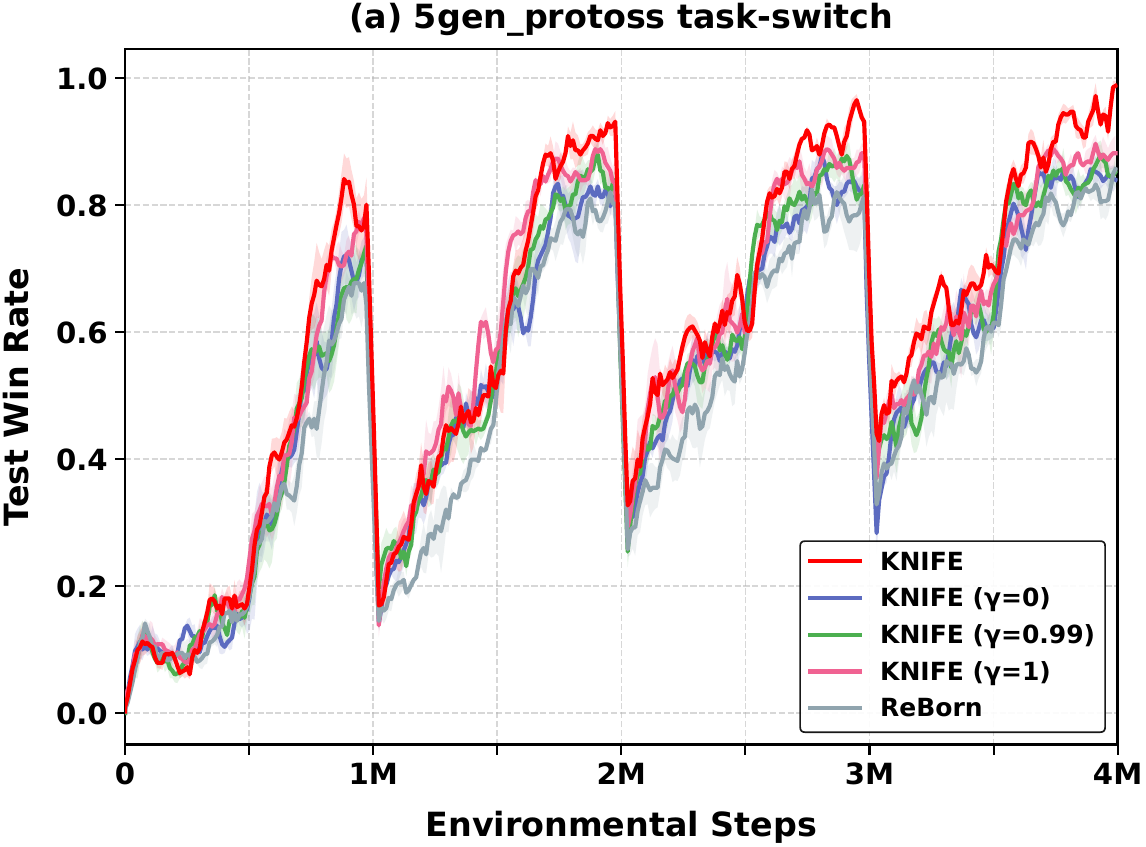} 	
            \includegraphics[width=0.24\columnwidth]{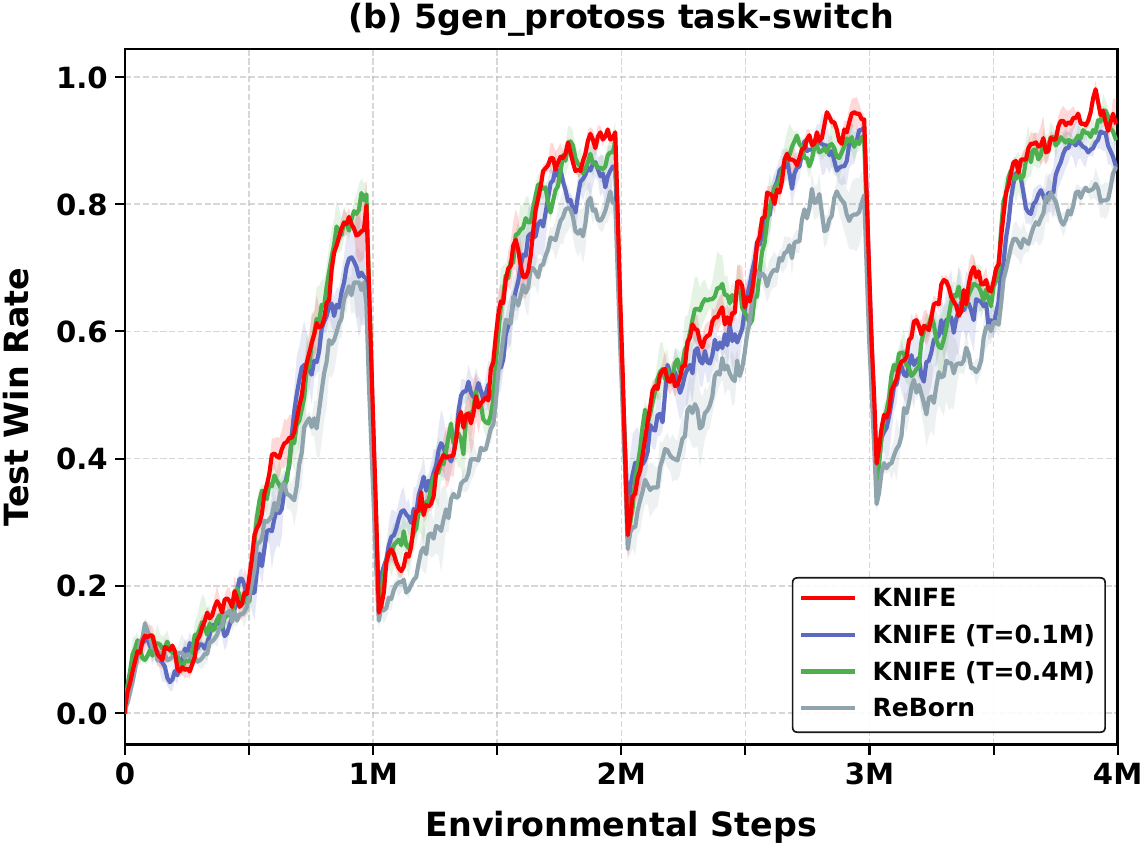} 	
            \includegraphics[width=0.24\columnwidth]{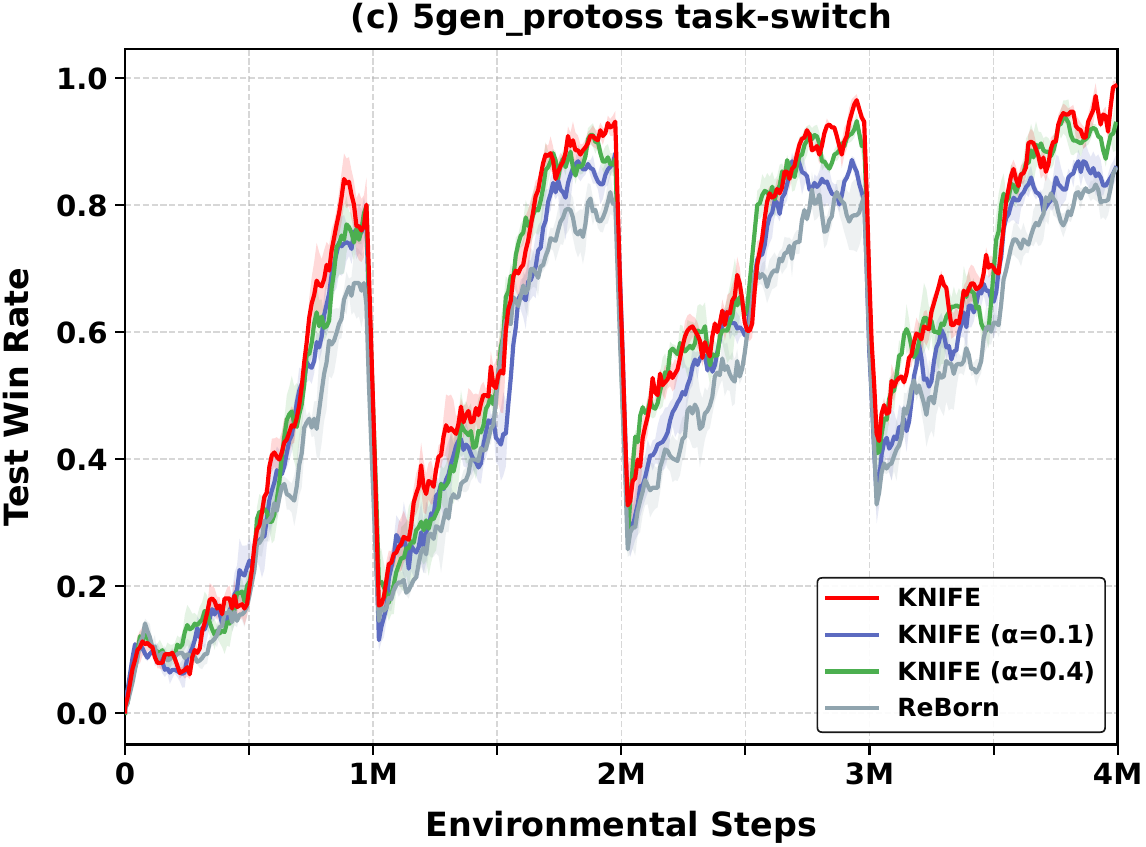} 		
		\includegraphics[width=0.24\columnwidth]{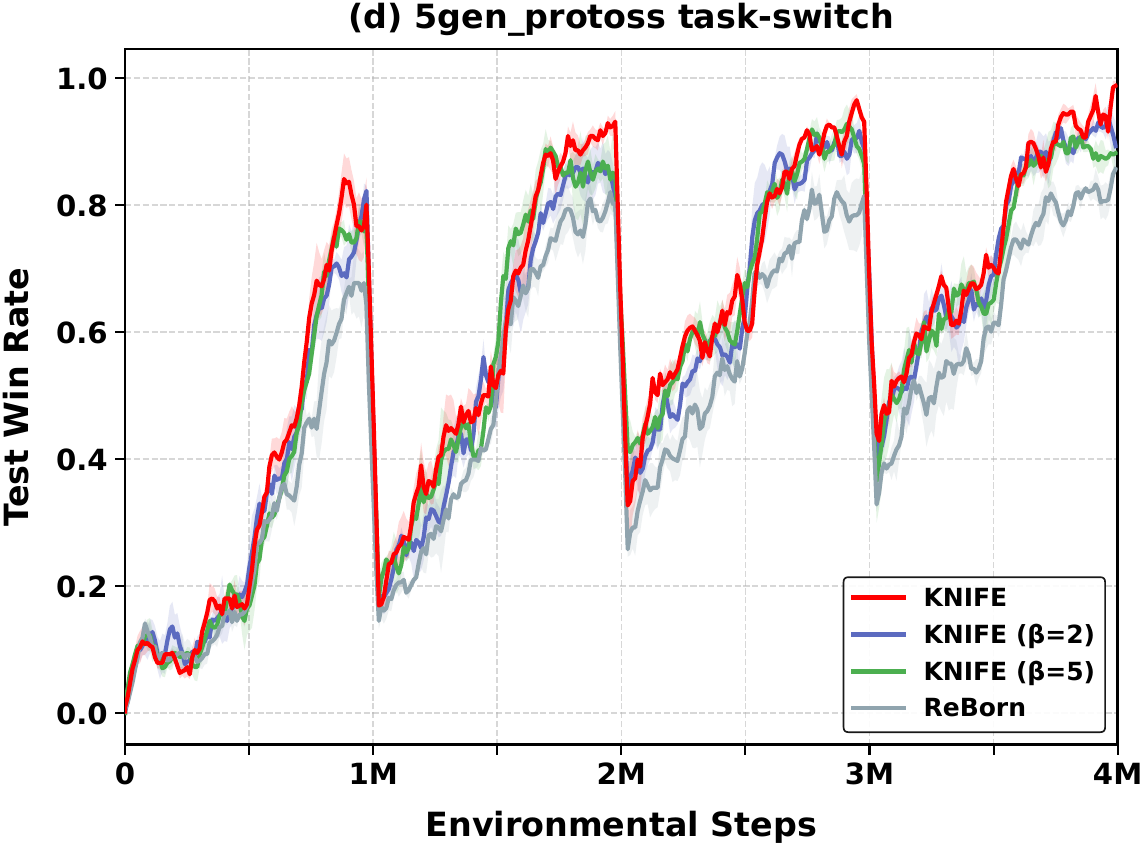} 	
	\end{minipage}
        \caption{Parameter Sensitivity: (a) Weight decaying factor, (b) Execution interval, (c) Stagnant threshold, (d) Volatile threshold}
	\label{app:fig:parameters} 
\end{figure*}

\textbf{Impact of Hyperparameter Tuning.} 
Furthermore, we evaluated the performance gap between a single cross-environment default configuration and a lightly tuned configuration. While KNIFE achieves strong results using identical default hyperparameters across most environments, tuning for a specific scenario can further unlock its potential. As shown in Table \ref{tab:tuned_perf}, lightly tuning the hyperparameters for the specific \textit{5gen\_protoss} environment yields an additional performance improvement, demonstrating the method's flexibility and upper-bound capacity.

\begin{table}[htbp]
\centering
\caption{Performance comparison between default and lightly tuned configurations.}
\label{tab:tuned_perf}
\renewcommand{\arraystretch}{1.2}
\begin{tabular}{lcc}
\toprule
\textbf{Environment} & \textbf{KNIFE (Default)} & \textbf{KNIFE (Tuned)} \\
\midrule
\textbf{5gen\_protoss} & $60.89 \pm 1.87$ & $63.50 \pm 1.45$ \\
\bottomrule
\end{tabular}
\end{table}

% 

% \begin{figure*}[!t]
% 	\begin{minipage}[!t]{\linewidth} % 第一行画图结果
% 		\centering
%     \includegraphics[width=0.32\columnwidth]{ICML26/experiment/smacv2/vd/[vd]5gen_protoss_tasks_v5_500k_test_battle_won_mean.pdf} 
%     \includegraphics[width=0.32\columnwidth]{ICML26/experiment/smacv2/vd/[vd]5gen_protoss_tasks_v6_500k_test_battle_won_mean.pdf}
% 	\includegraphics[width=0.32\columnwidth]{ICML26/experiment/smacv2/vd/[vd]5gen_terran_tasks_v6_500k_test_battle_won_mean.pdf} 	
% 	\end{minipage}
%             \caption{KNIFE can reduce the plasticity loss for QPLEX and QMIX in 5gen\_protoss environment of (a) SMACv2 (task-drift) and (b) SMACv2 (task-drift) (c) 5gen\_terran of SMACv2 (task-switch)}
%     \label{exp:compare:vd}
% \end{figure*}   

\section{Discussion}

\subsection{Societal impact}
This paper studies continual multi-agent reinforcement learning (MARL) and proposes a neuron-level plasticity injection method to improve adaptation under non-stationary reward profiles. The work is evaluated only on standard benchmark environments. Potential positive impacts include improved robustness and reduced retraining when objectives change. As with many RL techniques, improved adaptation could also be misused in unsafe or unmonitored deployments; we do not claim readiness for real-world use.
% \subsection{Limitations and future work}
% We mainly evaluate value-decomposition MARL benchmarks; it remains to be tested how well the diagnosis and method transfer to other MARL architectures and larger-scale settings.Future work will address these broader evaluations.

\subsection{Limitations and future work}
While we have demonstrated the effectiveness of KNIFE across diverse MARL architectures and successfully scaled it to large-scale scenarios (up to 40 agents), our evaluations are currently confined to simulated benchmark environments. The dynamic complexities of real-world physical systems (e.g., severe sim-to-real gaps and unpredictable sensor noise) remain to be fully explored. Additionally, while our pruning mechanism ensures long-term memory efficiency, the initial plasticity injection phase introduces a temporary memory overhead. Future work will investigate memory-constant injection mechanisms and evaluate KNIFE's robustness on physical multi-agent robotic systems.

%%%%%%%%%%%%%%%%%%%%%%%%%%%%%%%%%%%%%%%%%%%%%%%%%%%%%%%%%%%%

%\newpage
%\input{checklist.tex}

\end{document}